\documentclass{article}

\usepackage[english]{babel}
\usepackage[utf8x]{inputenc}
\usepackage[T1]{fontenc}
\usepackage{ dsfont }
\usepackage{placeins}
\usepackage{multirow}
\usepackage{lscape}
\usepackage{algpseudocode}
\usepackage{color,soul}
\usepackage{hyperref} 
\usepackage{csvsimple}
\usepackage{longtable}
\usepackage{arydshln,amssymb,color}
\usepackage{multicol}
\usepackage{bm}
\usepackage{tikz-cd}
\usetikzlibrary{arrows.meta}
\usepackage{collectbox}
\usepackage{chngpage}
\usepackage{natbib}
\usepackage{centernot}
\usepackage{graphicx}
\usepackage[caption=false, font=footnotesize]{subfig}

\usepackage{amsmath}
\usepackage{amsthm}
\usepackage{ mathrsfs }

\usepackage{booktabs} % For formal tables
\usepackage{xtab,booktabs}
\usepackage{bm}
\usepackage{mwe}
\usepackage{bbding}
\usepackage{array}

\usepackage{color, colortbl}
\definecolor{LightCyan}{rgb}{0.82,0.90,1}

\usepackage[linesnumbered,algoruled,boxed,lined]{algorithm2e}
\SetKwInOut{KwParam}{Parameters}
\DeclareMathOperator*{\sign}{sign}
\DeclareMathOperator*{\supp}{supp}

\DeclareMathOperator*{\diff}{diff}
\DeclareMathOperator*{\BhattCoef}{BhattCoef}

\usepackage{booktabs} % For formal tables

\newtheorem{mydef}{Definition}

\newtheorem{myproperty}{Property}
\newtheorem{myproposition}{Proposition}
\newtheorem{myconjecture}{Conjecture}

\newtheorem{mylemma}{Lemma}

\usepackage{mathtools}
\usepackage{url}
\usepackage{bbm}

\begin{document}

\newcommand{\titlecontent}{Comparing Two Samples Through Stochastic Dominance: A Graphical Approach}

%%%%%%%%%%%%%%%%%%%%%%%%%%%%%%%%%%%%%%%%%%%%%%%%%%%%%%%%%%%%%%%%%%%%%%%%%%%%%%

	\title{\bf \titlecontent}
	\author{
		Etor Arza\\
		BCAM - Basque Center for Applied Mathematics\\
		and \\
		Josu Ceberio\\
		University of the Basque Country UPV/EHU\\
		and \\
		Ekhiñe Irurozki\\
		Télécom Paris\\
		and \\
		Aritz Pérez\\
		BCAM - Basque Center for Applied Mathematics}
	\maketitle

\bigskip
\begin{abstract}
	Non-deterministic measurements are common in real-world scenarios: the performance of a stochastic optimization algorithm or the total reward of a reinforcement learning agent in a chaotic environment are just two examples in which unpredictable outcomes are common. 
	These measures can be modeled as random variables and compared among each other via their expected values or more sophisticated tools such as null hypothesis statistical tests.
	In this paper, we propose an alternative framework to visually compare two samples according to their estimated cumulative distribution functions.
	First, we introduce a dominance measure for two random variables that quantifies the proportion in which the cumulative distribution function of one of the random variables stochastically dominates the other one.
	Then, we present a graphical method that decomposes in quantiles i) the proposed dominance measure and ii) the probability that one of the random variables takes lower values than the other.
	With illustrative purposes, we re-evaluate the experimentation of an already published work with the proposed methodology and we show that additional conclusions---missed by the rest of the methods---can be inferred.
	Additionally, the software package \textit{RVCompare} was created as a convenient way of applying and experimenting with the proposed framework.
\end{abstract}

\noindent%
\vfill

\newpage

\setlength{\parskip}{1em}

%\author{Etor Arza,
%		Josu Ceberio,
%		Ekhiñe Irurozki and
%		Aritz Pérez
%
%% <-this % stops a space
%	\thanks{Etor Arza (earza@bcamath.org) and Aritz Pérez (aperez@bcamath.org) are with BCAM - Basque Center for Applied Mathematics.}% <-this % stops a space
%	\thanks{Josu Ceberio (josu.ceberio@ehu.eus) is with the University of the Basque Country UPV/EHU.}% <-this % stops a space
%	\thanks{Ekhiñe Irurozki (irurozki@telecom-paris.fr) is with Télécom Paris.}
%}
%	

%		\author[BCAM]{Etor Arza}
%		\ead{earza@bcamath.org}
%		\author[EHU]{Josu Ceberio}
%		\ead{josu.ceberio@ehu.eus}
%		\author[PARIS]{Ekhiñe Irurozki}
%		\ead{irurozki@telecom-paris.fr}
%		\author[BCAM]{Aritz Pérez}
%		\ead{aperez@bcamath.org}
%		
%		\address[BCAM]{BCAM - Basque Center for Applied Mathematics, Spain}
%		\address[EHU]{University of the Basque Country UPV/EHU, Spain}
%		\address[PARIS]{Télécom Paris, France}
%				
%		

\section{Introduction}

	The objective value obtained by an optimization algorithm may be non-deterministic.
	For example, in stochastic algorithms, the objective value measured depends on the seed used in the random number generator.
	In these kinds of scenarios, we can think that these non-deterministic measurements are observations of random variables with unknown distributions.
	Based on these measurements, we sometimes need to choose the random variable that takes the lowest (or largest) values.
	The expected values of the random variables---usually estimated as an average of several repeated observations---can be used for this purpose.
	However, many statisticians have claimed that summarizing data with simple statistics such as the average or the standard deviation is misleading, as very different data can still have the same statistics~\cite{matejka2017same,chatterjee_generating_2007}.

	\paragraph*{Motivating example 1)}

	A real-world motivation for this work is as follows.
	Suppose we need to choose the best option between two stochastic gradient-based methods for optimizing the parameters of a neural network. 
	A neural network classifier trained with a gradient-based method will produce different error rates~\cite{goodfellowDeepLearning2016} each time it is trained-tested, even if the same train-test dataset is used in each repeated measurement.
	One of the reasons is that the learned classifier depends on the initialization of its weights (before applying a gradient-based optimizer), which are often initialized randomly~\cite{pmlr-v9-glorot10a}.

	To illustrate the previous scenario, we trained and tested a neural network\footnote{We follow an example in the Keras~\cite{chollet2015keras} library, and train the neural network for one epoch.} in the MNIST dataset, and we compared two gradient-based optimizers in this data set: \textit{adam} and \textit{RMSProp}~\cite{goodfellowDeepLearning2016}.
	The error rate in the test set depends on the seed used to train the neural network, and therefore, we can model the error rate of each of the algorithms in this problem as a random variable.
	An observation of each of the two random variables (the error rate of each gradient-based optimizer is modeled as a random variable) involves training the neural network in the training set and measuring its error rate in the test set: the training and test sets are the same for each trained neural network.
	Figure~\ref{fig:dist_X1} shows the kernel density estimations of these random variables using the uniform kernel.
	As we see in the figure, the error rate is not the same in each measurement and ranges between $0.022$ and $0.04$.
	This shows that, in this context, it makes sense to model the error rate as a random variable rather than a constant: a unique value cannot represent the error rate without a significant amount of information loss.

	\begin{figure}
		\centering
		\includegraphics[width=0.60\textwidth]{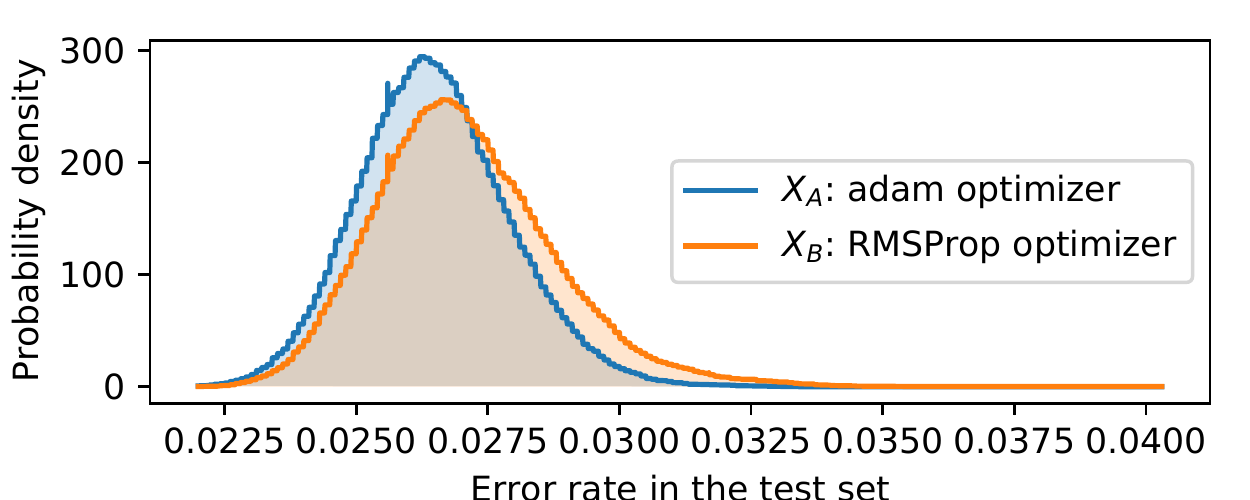}
		\caption{
			Density estimates of the error rates produced by the optimizers \textit{adam} and \textit{RMSProp} in the MNIST dataset.
			The sci-kit learn\cite{scikit-learn} package was used in the estimation.
		}		
		\label{fig:dist_X1}
	\end{figure}

	\paragraph*{Motivating example 2)}

	In the following, we present another example with synthetic data.
	Let us consider the two random variables $X_A$ and $X_B$ shown in Figure~\ref{fig:counterexample_being_better_dist}.
	$X_B$ has a lower expected value than $X_A$, $\mathbb{E}[X_B] < \mathbb{E}[X_A]$.
	If we use the expected value as the only criterion, then $X_B$ takes lower values than $X_A$.
	However, notice that with a low but nonzero probability, $X_B$ will take very large values that are undesirable in the context of minimization.
	Without loss of generality, in this paper, we assume that lower values are preferred.
	
	An error with low variance is very important in an environment where reliability is key, even if it means a slightly worse expected value.
	Some examples include breast cancer detection~\cite{cruz-roa_accurate_2017}, or some reinforcement learning tasks \cite{francois-lavet_introduction_2018,mnih_playing_2013} like self-driving cars \cite{badue_self-driving_2021}.
	
	\begin{figure}
		\centering
		\includegraphics[width=0.50\textwidth]{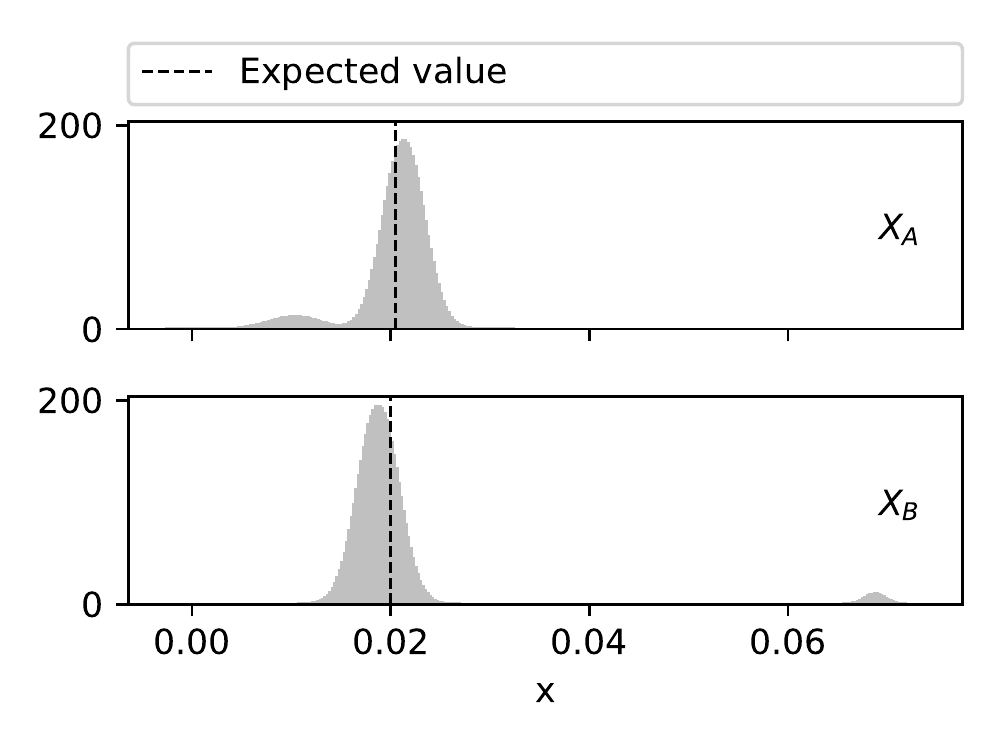}
		\caption{
			The probability density of two random variables $X_A$ and $X_B$, with probability density functions $g_{A} = 0.925 \cdot g_{\mathcal{N}(0.210325, 0.002)} + 0.075 \cdot g_{\mathcal{N}(0.010325, 0.025)}$ and $g_{B} = 0.975 \cdot g_{\mathcal{N}(0.01875, 0.002)} + 0.025 \cdot g_{\mathcal{N}(0.06875, 0.001)}$ where $\mathcal{N}(\rho,\sigma)$ is the normal distribution with mean $\rho$ and standard deviation $\sigma$.
			Their expected values are $\mathbb{E}[X_A] = 0.0205$ and $\mathbb{E}[X_B] = 0.02$ respectively.
		}
		\label{fig:counterexample_being_better_dist}
	\end{figure}
	
	In other circumstances, obtaining the lowest possible error can be more important than reliability.
	One could argue that reliability is less important in sentiment analysis~\cite{zhang_deep_2018}, or in certain real-world optimization problems~\cite{regnier2016truck}, where obtaining the best possible solution is key.
	When obtaining the best possible score is more important than reliability, it may even be worth running an optimization algorithm several times and choosing the best solution out of all the runs.
	In that case, $X_A$ would also be preferred to $X_B$, as $X_A$ has a higher probability of taking a value lower than $0.01$ (see Figure~\ref{fig:counterexample_being_better_dist}).

	\paragraph{Related work}

	In these two examples, we have seen that summarizing and comparing random variables with only the expected value can leave important information out (such as which of the random variables can take lower values), especially when neither random variable clearly takes lower values than the other one.
	Many works in the literature use null hypothesis tests~\cite{10.2307/2236101,conover1980practical,wilcoxonIndividualComparisonsRanking1945} to analyze observed samples and choose one of the random variables accordingly. 
	Nonetheless, as claimed in Benavoli et al.~\cite{benavoli2017time}, null hypothesis tests have their limitations too: when the null hypothesis is not rejected---this will happen often when the random variables being compared take similar values---, we get no information. 
	Not only that but even when the null hypothesis is rejected, it does not quantify the amount of evidence in favor of the alternative hypothesis~\cite{benavoli2017time}.

	\paragraph{Contribution}
	In this paper, we propose a graphical framework that compares two random variables using their associated cumulative distribution functions, in the context of choosing the one that takes lower values.
	The proposed methodology can compare the scores of two stochastic optimization algorithms or the error rates of two classifiers, among other applications.
	To achieve this, the performances of the optimization algorithms (or the error rates of the classifiers) are modeled as random variables, and then, we compare them by measuring the \textit{dominance}.
	
	Specifically, we first propose 8 desirable properties for \textit{dominance measures}: functions that compare two random variables in this context.
	From the measures in the literature, we find that the \textit{probability that one of the random variables takes a lower value than the other random variable} satisfies most of these properties.
	In addition, we propose a new dominance measure, the \textit{dominance rate}, that also satisfies most of the properties and is related to the first-order stochastic dominance~\cite{10.2307/2295819}.
	Then, we propose a graphical method that involves visually comparing the random variables through these two dominance measures.
	The graphical method, named \textit{cumulative difference-plot}, can also be used to compare the quantiles of the random variables, and it models the uncertainty associated with the estimate.
	By re-evaluating the experimentation of a recently published paper with the proposed methodology, we demonstrate that this new plot can be useful to compare two random variables, especially in the case when the random variables take similar values.

	Finally, an \textit{R} package named \textit{RVCompare}, available in CRAN, is distributed alongside this paper.
	With this package, the \textit{cumulative difference-plot} can be conveniently computed.
	The source code of the package and the supplementary material for the paper are available at~\href{https://github.com/EtorArza}{github.com/EtorArza}\footnote{
			The source of the package \textit{RVCompare} can be found at \href{https://github.com/EtorArza/RVCompare}{github.com/EtorArza/RVCompare}. 
			The code to reproduce every figure in the paper is available at~\href{https://github.com/EtorArza/SupplementaryPaperRVCompare}{github.com/EtorArza/SupplementaryPaperRVCompare}.
		}.

	The rest of the paper is organized as follows: in the next section, we propose eight desirable properties for dominance measures.
	Then, in Section~\ref{section:dominance_measures}, we study two dominance measures that satisfy most of these properties.
	Section~\ref{section:limitedata} introduces a graphical method to compare random variables. 
	In Section~\ref{section:related_work}, we discuss related methods in the literature and compare them to the proposed approach. 
	Section~\ref{section:real_world_case_study}, evaluates the proposed graphical method and other alternatives in an already published experimentation.
	In Section~\ref{section:assumptions_and_limitations}, we state the assumptions and limitations of the proposed \textit{cumulative difference-plot}.
	Finally, Section~\ref{section:conclusion} concludes the paper.

	\section{Desirable properties for dominance measures}
	\label{section:the_meaning_of_being_better}
	
	\subsection{Background}
	\label{section:background}

	When we have two random variables and we need to choose the one that takes the lowest values, we usually take i) the random variable with the lowest expected value or ii) the random variable with the lowest median.
	The \textit{median}~\cite{conover1980practical} of a continuous random variable $X_A$, denoted as $m_A$, is the value that satisfies $\mathcal{P}(X_A < m_A)  = \mathcal{P}(X_A > m_A)$.
	In other words, if $m_A$ is the median of $X_A$, a sample of $X_A$ is as likely to be lower than $m_A$ as it is to be higher.

	Interestingly enough, the median and the expected value have their strengths and weaknesses when it comes to choosing the random variable that takes the lowest values.
	In the following, we elaborate on this point with two particular cases of study.
	The first case is shown in Figure~\ref{fig:example_1_mean_median_prob_better}, with two random variables $X_A$ and $X_B$.
	Each of the random variables is a mixture of two Gaussian distributions with the same shape and similar weight in the mixture.
	It is clear that $X_A$ tends to take values lower than $X_B$, as the Gaussian distributions of $X_A$ are centered in $0.05$ and $0.07$, while the Gaussian distributions of $X_B$ are centered in $0.06$ and $0.08$.
	While the expected values of $X_A$ and $X_B$ are aligned with this intuition, the medians are not; as $\mathbb{E}[X_A] < \mathbb{E}[X_B]$ and $m_A > m_B$.
	However, the expected value does a poor job of summarizing the bimodal shape of $X_A$ or $X_B$:  both of these random variables usually take much higher or much lower values than their expected values.
	
	\begin{figure}
	\centering
	\vspace{-0.8em}
	\includegraphics[width=0.5\textwidth]{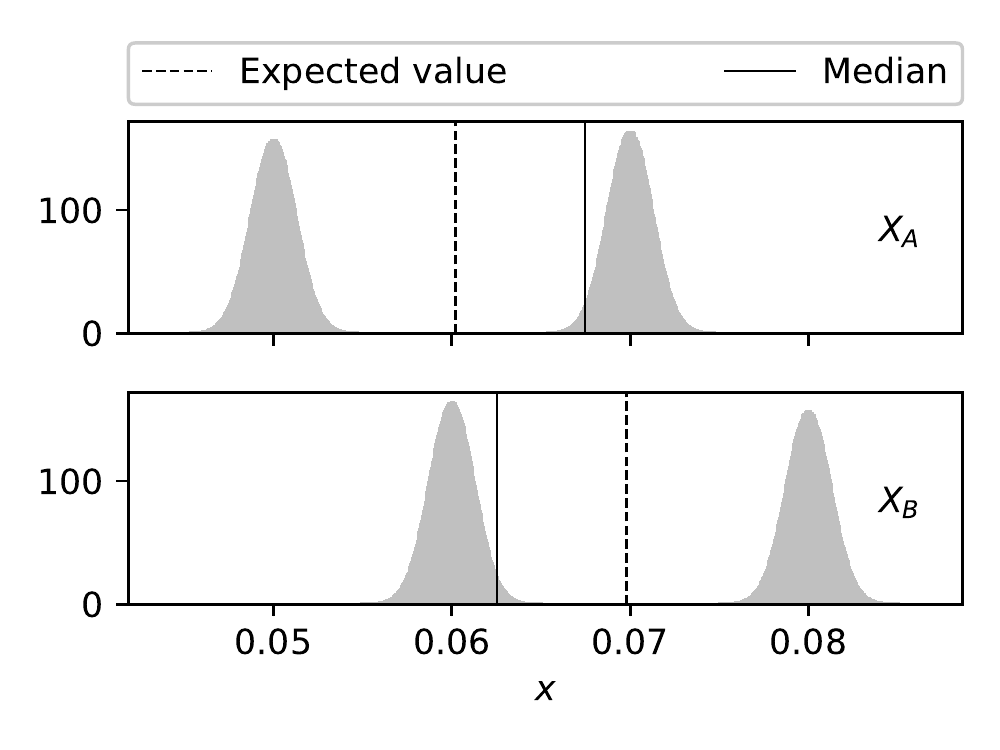}
	\caption{
		Case 1.
		The probability density functions of $X_A$ and $X_B$: $g_{A} = 0.489 \cdot g_{\mathcal{N}(0.05, 0.00125)} + 0.511 \cdot g_{\mathcal{N}(0.07, 0.00125)}$ and $g_{B} = 0.511 \cdot g_{\mathcal{N}(0.06, 0.00125)} + 0.489 \cdot g_{\mathcal{N}(0.08, 0.00125)}$ where $g_{\mathcal{N}(\rho,\sigma)}$ is the density function of the normal distribution with mean $\rho$ and standard deviation $\sigma$.
	}
	\label{fig:example_1_mean_median_prob_better}
	\end{figure}

	The second case is shown in Figure~\ref{fig:example_2_mean_median_prob_better}.
	With a very high probability, $X_A$ takes lower values than $X_B$, even though $X_B$ will rarely take really low values, which might prove useful in some particular applications.
	In this case, $m_A < m_B$ and $\mathbb{E}[X_A] > \mathbb{E}[X_B]$, hence, the comparison of the medians are aligned with the intuition that $X_A$ takes lower values than $X_B$, while the expected values are not.
	In the presence of outliers~\cite{carrenoAnalyzingRareEvent2020}, the median is considered more robust than the expected value~\cite{rousseeuw_robust_2011}.

	\begin{figure}
	\centering
	\vspace{-0.8em}
	\includegraphics[width=0.5\textwidth]{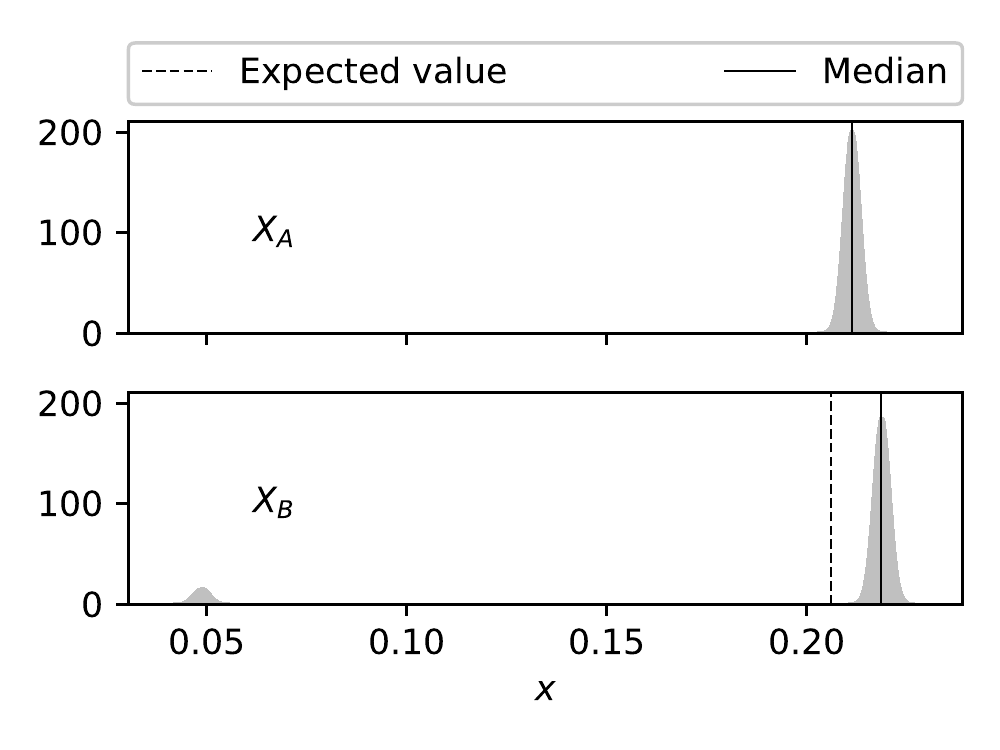}

	\caption{
		Case 2.
		The probability density functions of $X_A$ and $X_B$: $g_{A} = g_{\mathcal{N}(0.211325, 0.002)}$ and $g_{B} = 0.925 \cdot g_{\mathcal{N}(0.21875, 0.002)} + 0.075 \cdot g_{\mathcal{N}(0.04875, 0.002)}$ respectively.
	}
	\label{fig:example_2_mean_median_prob_better}
	\end{figure}

	Notice that, in the second case, it is not trivial to choose between $X_A$ and $X_B$, as $X_B$ can take lower values, but $X_A$ is more likely to be lower than $X_B$. 
	So, when can we claim that one of them clearly takes lower values than the other?
	When the cumulative distribution of $X_A$ is higher than the cumulative distribution of $X_B$ in the entire domain of definition: in that case, $X_A$ has a higher probability than $X_B$ of taking values lower than $x$, for all $x$ in the domain of definition.
	This is known \cite{10.2307/2236101} as $X_A$ being stochastically smaller than $X_B$.
%	A conservative answer is when the probability that $X_A < X_B$ is $1$.
%	However, this is a very strict requirement, as it implies that the two random variables $X_A$ and $X_B$ with piecewise-integrable densities have disjoint supports.
%	The support, denoted as $\supp$, of a continuous random variable $X_A$ is the subset of its domain with strictly positive probability density.
%	A less strict requirement is that the cumulative distribution of $X_A$ is higher than the cumulative distribution of $X_B$ in all the domain of definition~\cite{10.2307/2236101}.
%	This is still enough to state that one of the random variables takes lower values than the other: the cumulative distribution function is a function that maps each value $x$---in the domain of definition of a random variable---to the probability that a sample of that random variable is equal to or lower than this value.
	Depending on the field of study, it can also be referred to~\cite{schmid_testing_1996,https://doi.org/10.1111/iere.12038,10.2307/2295819} as ``$X_A$ stochastically dominates $X_B$''\footnote{Without loss of generality, minimization is assumed in this paper.}.
	The stochastic dominance can be further relaxed, obtaining what is known as \textit{first-order stochastic dominance} in the literature~\cite{schmid_testing_1996,10.2307/2295819}, although, for the sake of brevity, we will call it \textit{stochastic dominance} throughout the paper. 
	
	\begin{mydef}
	\label{def:dominance}
	(Stochastic dominance) Let $X_A$ and $X_B$ be two continuous random variables defined in a connected subset $N \subseteq \mathbb{R}$.
	We say that $X_A$ stochastically dominates $X_B$, denoted as  ${X_A \succ X_B}$, when	
	
	$i) \  G_A(x) \geq G_B(x)  \ \ \text{for all } x \in N$\\
	and \\	
	$ii) \ \text{There exists an } x \in N \ \ \text{such that} \ \ G_A(x) > G_B(x).$
	
	\end{mydef}

	where $G_A$ and $G_B$ are the cumulative distributions of $X_A$ and $X_B$ respectively.

	For $X_A$ not to stochastically dominate $X_B$ (denoted as $X_A \nsucc X_B$\footnote{Note that $X_A \nsucc X_B$ is not equivalent to $X_B \succ X_A$.}), either condition i) or ii) must be violated.
	The special case that ${X_A \nsucc X_B}$ and ${X_B \nsucc X_A}$ at the same time is defined, it is said that $X_A$ and $X_B$ \textit{cross}~\cite{https://doi.org/10.1111/iere.12038}, and we denote it as $X_A \lessgtr X_B$. 
	In the non trivial ($X_A \neq X_B$) case that $X_A \lessgtr X_B$, there exists two points $x_1,x_2 \in N$ such that $G_A(x_1) < G_B(x_1)$ and $G_A(x_2) > G_B(x_2)$: we cannot say, for all $x \in N$, that one of the random variables has a higher probability of taking values lower than $x$.

	Let us now see how the cumulative distributions can be used to compare random variables in an example.
	In Figure~\ref{fig:example_1_pareto}, the cumulative distributions of the random variables described in Figure~\ref{fig:example_1_mean_median_prob_better} are shown.
	We can see that $G_A(x) > G_B(x)$ for almost all $x \in N$.
	But there is at least a point $x \in (0.06, 0.07)$ where $G_A(x) < G_B(x)$, hence, $X_A \lessgtr X_B$. 
	The same happens in the second case (Figure~\ref{fig:example_2_pareto}).
	As in the previous case, $X_A \lessgtr X_B$, because even though $G_A(x) > G_B(x)$ for almost all $x \in N$ (in which $g_A(x) \neq 0$ and $g_B(x) \neq 0 $), for all $x \in (0.05, 0.2), G_A(x) < G_B(x)$.
		
	\begin{figure} 
		\centering
		\subfloat[Case 1\label{fig:example_1_pareto}]{%
		\includegraphics[width=0.5\linewidth]{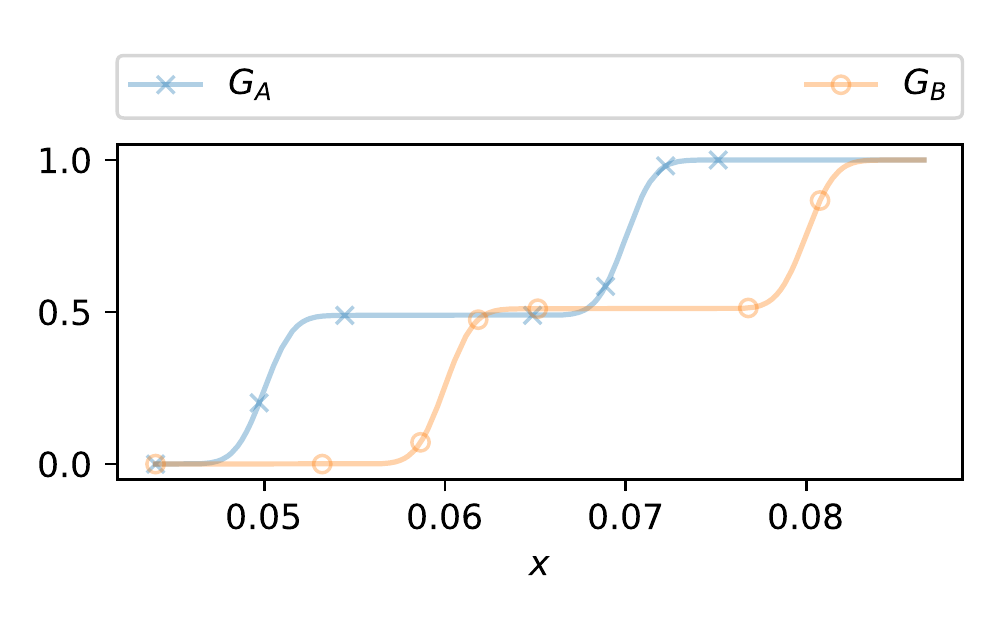}
		}
		\\
		\subfloat[Case 2\label{fig:example_2_pareto}]{%
		\includegraphics[width=0.5\linewidth]{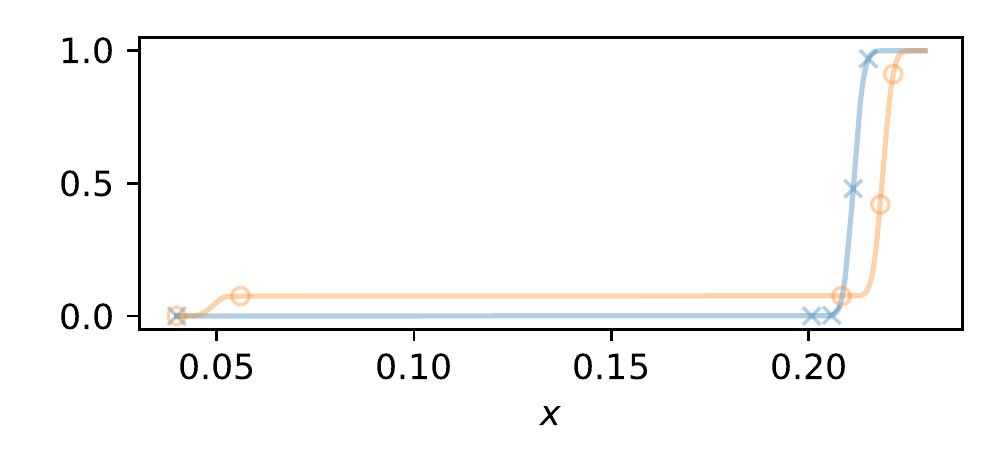}
		}
		\caption{The cumulative distributions of the two cases shown in Figures~\ref{fig:example_1_mean_median_prob_better} and \ref{fig:example_2_mean_median_prob_better}.}
		\label{fig:examples_1_2_multiobjective}
	\end{figure}

%	Trying to give a response to the problem of determining if either random variable dominates the other~\cite{rinne_testing_1995,schmid_testing_1996}, many null hypothesis tests have been proposed, including the Mann-Whitney test~\cite{10.2307/2236101}. 
%	When one of the random variables dominates the other, choosing the random variable that takes the lowest values is straightforward with these tests.
%	However, their applicability is limited when the random variables cross (see Figures~\ref{fig:example_1_mean_median_prob_better} and \ref{fig:example_2_mean_median_prob_better}).
	In the following, we will study how to quantify the difference between two random variables, emphasizing the degree to which one of the random variables stochastically dominates the other.

	\subsection{Desirable properties}
	\label{section:properties_of_comparison_comparison_functions}

	There are many ways to compare two random variables, each with a different point of view: some aim to find how dissimilar two random variables are (disregarding which of them takes lower values), while other methods try to guess if one of the random variables stochastically dominates the other one.
	In the context of this paper, we are interested in measures that, given two random variables, quantify through the stochastic dominance how much one of the random variables tends to take lower values than the other.
	We use the term \textit{dominance measure} to refer to functions that quantify the difference between two random variables following this intuition.
	In this section, we define eight desirable properties for these dominance measures, and we study the suitability of several measures from the literature.

	\begin{mydef}
		Let $X_A$ and $X_B$ be two continuous random variables.
		We define a \emph{dominance measure} between two random variables as a function $\mathcal{C}$ that maps two random variables into a real value $\mathcal{C}(X_A,X_B)$.
	\end{mydef}
		
	It is desirable that $\mathcal{C}(X_A,X_B)$ quantifies the stochastic dominance.
	Hence, we want $\mathcal{C}(X_A,X_B)$ to be proportional to the portion of the support of $X_A$ and $X_B$ in which $G_A(x) < G_B(x)$.
	Formally, this intuitive idea can be represented as:

	\begin{myproperty}
	\label{prop:bounds_with_interpretation}
	$\mathcal{C}$ is defined in the $[0,1]$ interval, where: 
	
	i) $$\mathcal{C}(X_A,X_B) = 1 \iff X_A \succ X_B$$ 
	ii) $$ \mathcal{C}(X_A,X_B) = 0 \iff X_B \succ X_A$$
	iii)		$$\mathcal{C}(X_A,X_B) \in (0,1) \iff X_B \lessgtr X_A  $$
	\end{myproperty}

	\begin{myproposition}
		If a dominance measure $\mathcal{C}$ satisfies Property~\ref{prop:bounds_with_interpretation} i) and ii), then it also satisfies Property~\ref{prop:bounds_with_interpretation} iii).
	\end{myproposition}
	\begin{proof}
		By definition, $X_B \lessgtr X_A$ iff  $X_A \nsucc X_B$ and $X_B \nsucc X_A$. Property~\ref{prop:bounds_with_interpretation} i) and ii) implies that $X_A \nsucc X_B$ and $X_B \nsucc X_A$ iff $\mathcal{C}(X_A,X_B) \neq 1$ and $\mathcal{C}(X_A,X_B) \neq 0$.
		From Property~\ref{prop:bounds_with_interpretation} i) also $\mathcal{C}(X_A,X_B) \in [0,1]$, thus $X_B \lessgtr X_A$ iff $\mathcal{C}(X_A,X_B) \in (0,1) $.
	\end{proof}

	\begin{myproperty}
		\label{prop:antisimmetry}
		(Antisymmetry) $\mathcal{C}(X_A,X_B)$ and $\mathcal{C}(X_B,X_A)$ add up to 1.
		$$\mathcal{C}(X_A,X_B) = 1-\mathcal{C}(X_B,X_A)$$
	\end{myproperty}

		It is noteworthy that Property~\ref{prop:bounds_with_interpretation} ii) can be inferred from Property~\ref{prop:bounds_with_interpretation} i) and Property~\ref{prop:antisimmetry}.

	\begin{myproperty}
	\label{prop:inverse}
	The inversion (under the sum) of the operands of $\mathcal{C}$ equals the inversion of $\mathcal{C}$:
	$$\mathcal{C}(-1 \cdot X_A, -1 \cdot X_B) = 1-\mathcal{C}(X_A,X_B)$$
	\end{myproperty}

	\begin{myproperty}
		\label{prop:equality_value}
		When $X_A$ and $X_B$ are equal, $\mathcal{C}$ is symmetric.
		$$X_A = X_B \implies \mathcal{C}(X_A,X_B) = \mathcal{C}(X_B,X_A)$$ 
	\end{myproperty}

		Assuming Property~\ref{prop:antisimmetry} holds, we can rewrite the previous property as:
		$$X_A = X_B \implies \mathcal{C}(X_A,X_B) = 0.5.$$ 
		Note that the opposite is not true: 
		$$\mathcal{C}(X_A,X_B) = \mathcal{C}(X_B,X_A) \centernot\implies X_A = X_B$$

	\begin{myproperty}
		\label{prop:translation}
		(Invariance to translation) Moving the domain of definition of $X_A$ and $X_B$ by the same amount does not change \nolinebreak $\mathcal{C}$\footnote{We define $X_A + \lambda$ as the random variable that is sampled in two steps: first obtain an observation from $X_A$ and then add $\lambda$ to this observation. We define $\lambda \cdot X_A$ in a similar way.}.
		$$\text{for all } \lambda \in \mathbb{R}, \ \ \	\mathcal{C}(X_A + \lambda ,X_B + \lambda) = \mathcal{C}(X_A,X_B)$$
	\end{myproperty}

	\begin{myproperty}
	\label{prop:scaling}	
	(Invariance to scaling)	Scaling both $X_A$ and $X_B$ by the same positive amount does not change $\mathcal{C}$.
	$$\text{for all } \lambda > 0, \ \ \	\mathcal{C}(\lambda \cdot X_A, \lambda \cdot X_B) = \mathcal{C}(X_A,X_B)$$
	\end{myproperty}

	In the following lines, we give an intuition for Property 7.
	In Case 2, shown in Figure~\ref{fig:example_2_mean_median_prob_better}, we saw that $\text{for all } x \in (0.075, 0.2)$, $G_A(x) < G_B(x)$.
	However, notice that most of the mass of $X_A$ and $X_B$ is in the interval $(0.2,0.23)$, where $G_A(x) > G_B(x)$.
	This means that most of the observed points of $X_A$ and $X_B$ will be in that interval.
	Therefore, it makes sense that $G_A(x) > G_B(x)$ has a higher weight than $G_A(x) < G_B(x)$ in the computation $\mathcal{C}(X_A,X_B)$.
	In other words, the \textit{small} mass of $X_B$ centered in $0.05$ can only account for a \textit{small} part of $\mathcal{C}(X_A,X_B)$.
	In what follows, this is formalized as $X_B$ being a mixture of two distributions, where one of the distributions represents this small mass with a small weight in the mixture.
	Property~\ref{prop:scale_of_portions} states that the change in the computation of $\mathcal{C}$ produced by the distribution of small weight in the mixture can be, at most, its weight in the mixture.
	
	\begin{myproperty}
		\label{prop:scale_of_portions}
		Let $X_B =\mathcal{M}_{[1-\tau,\tau]}(X_{B1},X_{B2})$ be the mixture\footnote{The probability density function of $\mathcal{M}_{[1 - \tau,\tau]}(X_{B1},X_{B2})$ is defined as $(1 - \tau) \cdot g_{B1}(x) + \tau \cdot g_{B2}(x)$. Note that $\tau \in [0,1]$.} distribution of $X_{B1}$ and $X_{B2}$ with weights $1 - \tau$ and $\tau$ respectively and let $X_A$ be another random variable.
		Then, 
		
		$$\left| \mathcal{C}(X_A, X_B) -  \mathcal{C}(X_A, X_{B1})  \right| \leq \tau$$
	\end{myproperty}

	\begin{figure}
		\centering
		\vspace{-0.8em}
		\includegraphics[width=0.5\textwidth]{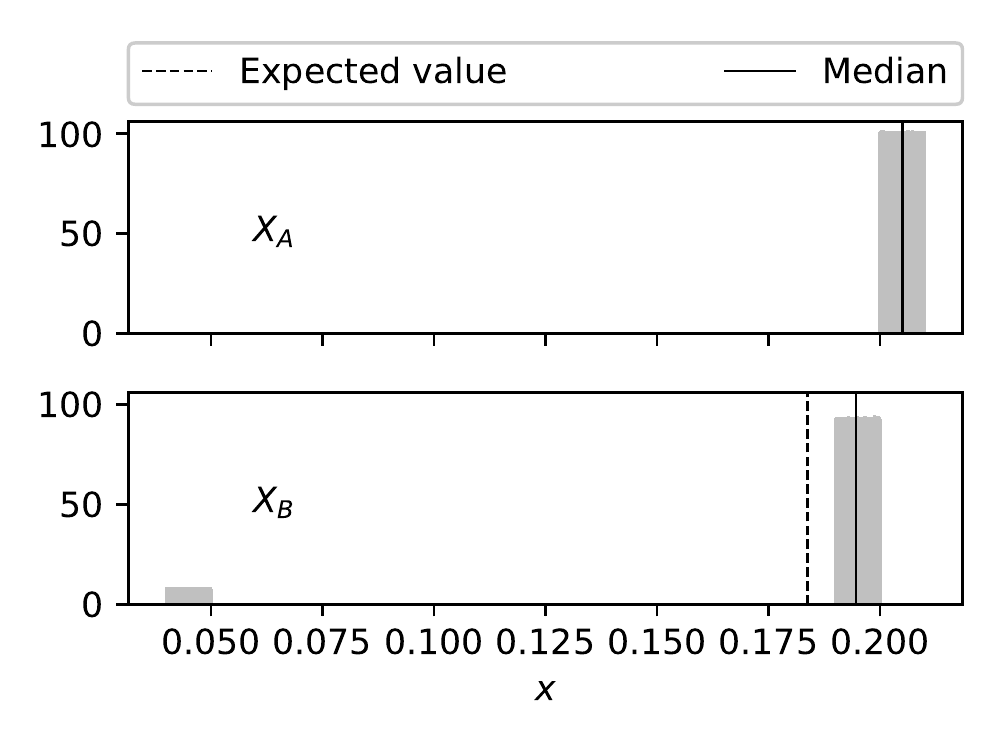}
		
		\caption{
			Case 3.
			The probability density functions of $X_A$ and $X_B$: $g_{A} = g_{\mathcal{U}(0.2, 0.21)}$ and $g_{B} = 0.925 \cdot g_{\mathcal{U}(0.19, 0.2)} + 0.075 \cdot g_{\mathcal{U}(0.04, 0.05)}$ respectively, where $\mathcal{U}(0.2, 0.21)$ is the uniform distribution in the interval $(0.2,0.21)$.
		}
		\label{fig:example_5_mean_median_prob_better}
	\end{figure}

	Property~\ref{prop:partial_translation_within_support} explains that, under certain circumstances, $\mathcal{C}(X_A,X_B)$ is invariant to the translation/dilatation of one of the random variables.
	Specifically, it states that the distribution of one of the random variables ($X_B$) can change without affecting the value of $\mathcal{C}(X_A,X_B)$ as long as the changed part does not overlap with the support of the other random variable ($X_A$).
	Let us assume that the random variable $X_B$ is defined as mixture distribution $\mathcal{M}_{[1 - \rho,\rho]}(X_{B1},X_{B2})$ where the supports of $X_{B2}$ and $X_A$ do not overlap, with $\rho \in (0,1)$.
	Property~\ref{prop:partial_translation_within_support} states that a translation and/or dilatation can be applied to $X_{B1}$, as long as: i) this transformation does not cause an overlap of the supports of $X_A$ and $X_{B2}$, and ii) partial transformations will also not cause an overlap (hence the need for $\xi_1$ and $\xi_2$ in Property~\ref{prop:partial_translation_within_support}).
	In the following, we formalize this property:

	\begin{myproperty}
		\label{prop:partial_translation_within_support}
		Let $X_B = \mathcal{M}_{[1 - \rho,\rho]}(X_{B1},X_{B2})$ be the mixture distribution of $X_{B1}$ and $X_{B2}$ with weights $1 - \rho$, and $\rho$, respectively and let $X_A$ be another random variable with $\rho \in (0,1)$.
		Suppose that $\supp(X_{B2}) \cap \supp(X_{A}) = \varnothing$.
		Let $\lambda_1 \in \mathbb{R}^+, \lambda_2 \in \mathbb{R}$ be two numbers such that $\text{for all } \xi_1, \xi_2 \in [0,1]$, $\supp((1+(\lambda_1 -1)\xi_1) \cdot X_{B2} + \xi_2\lambda_2) \cap \supp(X_{A}) = \varnothing$.
		Then, 
		$$\mathcal{C}(X_A, X_B) = \mathcal{C}(X_A, \mathcal{M}_{[1 - \rho,\rho]}(X_{B1}, \lambda_1 \cdot X_{B2} + \lambda_2)$$
	\end{myproperty}

	This property can be applied to the distributions in Case~3 shown in Figure~\ref{fig:example_5_mean_median_prob_better}.
	For example, the probability mass in the interval $(0.04, 0.05)$ could have been centered in $0.1$ or $0.15$ instead of $0.045$, without any changes to $\mathcal{C}(X_A,X_B)$. 
	In addition to the position, the shape of the mass can also be altered as long as its weight in the mixture stays the same and does not overlap with $X_A$.

	Unfortunately, it is impossible that a dominance measure satisfies Properties~\ref{prop:bounds_with_interpretation} and \ref{prop:scale_of_portions} at the same time.
	Intuitively, the problem is that, given the distributions $X_A$ and $X_B = \mathcal{M}_{[1 - \tau, \tau]} (X_{B1}, X_{B2})$,  it is possible that $X_A \succ X_{B1}$ and at the same time $X_B \succ X_A$ with $\tau < 0.5$\footnote{See \href{https://etorarza.github.io/pages/2021-interactive-comparing-RV.html}{https://etorarza.github.io/pages/2021-interactive-comparing-RV.html} for an interactive example that illustrates the above point.}.
	We formalize and prove this claim in the following proposition:
	
	\begin{myproposition}
	Let $\mathcal{C}$ be a dominance measure.
	
	i)	If $\mathcal{C}$ satisfies Property~\ref{prop:bounds_with_interpretation}, then it fails to satisfy Property~\ref{prop:scale_of_portions}.

	ii)	If $\mathcal{C}$ satisfies Property~\ref{prop:scale_of_portions}, then it fails to satisfy Property~\ref{prop:bounds_with_interpretation}.	
	\end{myproposition}
	\begin{proof}
	A dominance measure only satisfies a property when that property is true for every possible random variable.
	Consequently, to prove this proposition, it is enough to find four random variables $X_A, X_B, X_{B1}$ and $X_{B2}$ where

	i) $X_B =\mathcal{M}_{[0.1,0.9]}(X_{B1},X_{B2}),$\\
	ii) $X_A \succ X_{B1},$ \\
	iii) $X_B \succ X_A.$

	If four random variables can be found that satisfy these three statements, then with Property~\ref{prop:bounds_with_interpretation} we obtain that ${\mathcal{C}(X_A,X_{B1}) = 1}$ and ${\mathcal{C}(X_A,X_{B}) = 0}$. 
	This contradicts Property~\ref{prop:scale_of_portions}, because ${\left| \mathcal{C}(X_A, X_B) -  \mathcal{C}(X_A, X_{B1})  \right| \nleq 0.1}$.
	The same is true the other way around, Property~\ref{prop:scale_of_portions} states that ${\left|  \mathcal{C}(X_A, X_B) -  \mathcal{C}(X_A, X_{B1})  \right| \leq 0.1}$ and this contradicts Property~\ref{prop:bounds_with_interpretation}, with ${\mathcal{C}(X_A,X_{B1}) < 1}$ or ${\mathcal{C}(X_A,X_{B}) > 0}$.

	A simple example in which this happens is for the random variables
	
	${X_A = \mathcal{U}(0,1)},$

	${X_B = \mathcal{M}_{[0.9, 0.1]} (\mathcal{U}(0.1,1),\mathcal{U}(-0.5,0))},$

	${X_{B1} = \mathcal{U}(0.1,1)},$
	 
	${X_{B2} = \mathcal{U}(-0.5,0)}.$
	
	The cumulative distribution functions of $X_A$, $X_B$ and $X_{B1}$ are shown in Figure~\ref{fig:impossibilityporp7and1cumulative}, where it is clear that $X_B \succ X_A$ and $X_A \succ X_{B1}$.
	\end{proof}

\begin{figure}
	\centering
	\includegraphics[width=0.6\linewidth]{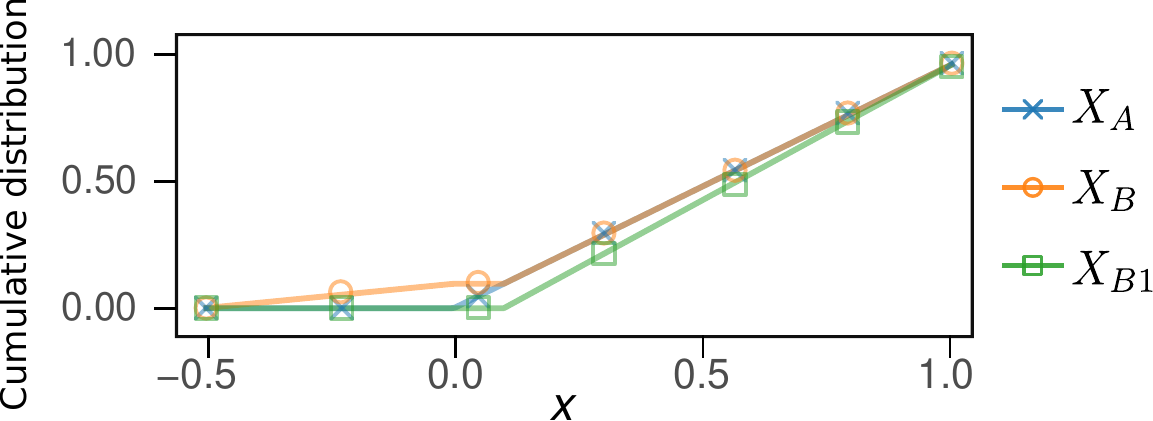}
	\caption{The cumulative distribution functions of $X_A$, $X_B$ and $X_{B1}$.}
	\label{fig:impossibilityporp7and1cumulative}
\end{figure}

		In the following, we will briefly review several measures in the literature and, specifically, which of the proposed properties they satisfy.
		Many measures describe the difference between $X_A$ and $X_B$, disregarding whether the difference in cumulative density is positive or negative.
		Consequently, they cannot satisfy Property~\ref{prop:bounds_with_interpretation} (see Appendix~\ref{appendix:additional_information_on_bad_comparison_functions} for details).
		This is the case for \textit{f-divergences}---including Kullback-Leibler, Jensen-Shannon, the Hellinger distance and the total variation---and for the Wasserstein distance.
		These measures also fail to satisfy several other properties (see a summary in Table~\ref{tab:functions_which_properties_satisfy}).

		\begin{table}
		\setlength\tabcolsep{3.5pt} % default value: 6pt
		\renewcommand{\arraystretch}{1.2} % General space between rows (1 standard)
		\centering
		\caption{
			Which of the properties in Section~\ref{section:properties_of_comparison_comparison_functions} does each measure satisfy?
		}
		\label{tab:functions_which_properties_satisfy}
		\begin{scriptsize}
			\begin{tabular}{|l|l|c|c|c|c|c|c|c|}
				\multicolumn{9}{c}{ }\\
				\hline
				& 1  &  2      &     3      &     4      &     5      &     6      &     7      &     8      \\ \hline
				Kullback-Leibler divergence     &    &    &    & \checkmark & \checkmark & \checkmark &    & \checkmark \\ \hline
				Jensen-Shannon divergence       &    &    &           & \checkmark & \checkmark & \checkmark &    & \checkmark \\ \hline
				Total-Variation                 &    &    &           & \checkmark & \checkmark & \checkmark & \checkmark & \checkmark \\ \hline
				Hellinger distance              &    &    &           & \checkmark & \checkmark & \checkmark & \checkmark & \checkmark \\ \hline
				Wasserstein distance            &    &    &           & \checkmark & \checkmark &    &    &    \\ \hline
				$\mathcal{C}_{\mathcal{P}}$: Probabitlity of $X_A < X_B$        &   & \checkmark & \checkmark & \checkmark & \checkmark & \checkmark & \checkmark & \checkmark \\ \hline
				$\mathcal{C}_{\mathcal{D}}$: Dominance rate of $X_A$ over $X_B$ & \checkmark & \checkmark & \checkmark & \checkmark & \checkmark & \checkmark &    & \checkmark \\ \hline
			\end{tabular}
		\end{scriptsize}
		\begin{scriptsize}		
			\begin{flushleft}
			\hspace{3cm}	A checkmark $\checkmark$ indicates that the measure satisfies the property.
			\end{flushleft}
		\end{scriptsize}
	\end{table}

	\section{Dominance measures}
	\label{section:dominance_measures}
	
	Most of the measures in the literature fail to satisfy the eight properties introduced in Section~\ref{section:properties_of_comparison_comparison_functions}.
	However, there is a dominance measure in the literature that overcomes this limitation: the probability that $X_A < X_B$~\cite{conover1980practical}.
	
	\subsection{$\mathcal{C}_\mathcal{P}$: the probability of $X_A < X_B$}
	\label{section:p_a_lower_b}
	
	We can compare $X_A$ and $X_B$ with the probability that a value sampled from $X_A$ is smaller than a value sampled from $X_B$.
	When the random variables are exactly the same, this probability is $0.5$.
	Formally, given two continuous random variables $X_A$ and $X_B$ defined in a connected set $N \subseteq \mathbb{R}$, the probability that $X_A < X_B$ is defined as:
	
%	\begin{multline*}
%	 \mathcal{P}(X_A < X_B)= \int_N g_B(x) \left( \int_{\inf\{N\}}^{x} g_A(t)dt\right)  dx = \\ \int_N g_B(x) G_A(x)   dx 
%	\end{multline*}	 
	
	\begin{equation}
			\label{equation:definition_of_C_P_from_density}
			\mathcal{P}(X_A < X_B) = \int_N g_B(x) G_A(x)   dx.
	\end{equation}

%	It has many applications, for instance, the sign test~\cite{conover1980practical} is based on this property.

	When we consider $\mathcal{P}(X_A < X_B)$ as a dominance measure, we will denote it as $\mathcal{C}_\mathcal{P}(X_A, X_B)$.

	One of the advantages of $\mathcal{C}_\mathcal{P}$ is its easy interpretation.
	In addition, $\mathcal{C}_\mathcal{P}$ is a well behaved dominance measure, as it satisfies Properties~\ref{prop:antisimmetry}, \ref{prop:inverse}, \ref{prop:equality_value}, \ref{prop:translation}, \ref{prop:scaling}, \ref{prop:scale_of_portions} and \ref{prop:partial_translation_within_support}.
	It also satisfies a weak version of Property~\ref{prop:bounds_with_interpretation}: 
	$$\mathcal{C}_\mathcal{P}(X_A,X_B) = 1 \implies X_A \succ X_B \implies $$
	$$\mathcal{C}_\mathcal{P}(X_A,X_B) \in (0.5,1]$$
	\centerline{and} 
	$$\mathcal{C}_\mathcal{P}(X_A,X_B) = 0 \implies X_B \succ X_A \implies$$
	$$\mathcal{C}_\mathcal{P}(X_A,X_B) \in [0,0.5).$$
	
	Note that, when $X_A \succ X_B$, $\mathcal{C}_\mathcal{P}(X_A,X_B) \neq 1$ is still possible, and this is why it does not satisfy Property~\ref{prop:bounds_with_interpretation} entirely.
	For instance, when the probability densities of $X_A$ and $X_B$ are two Gaussian distributions with the same variance and the mean of $X_A$ is lower, then $X_A \succ X_B$ but $\mathcal{C}_\mathcal{P}(X_A,X_B) < 1$.

	So far, we have seen that $\mathcal{C}_\mathcal{P}$ satisfies most of the properties.
	Unfortunately, since it does not satisfy Property~\ref{prop:bounds_with_interpretation}, not all cases of $X_A \succ X_B$ can be identified by $\mathcal{C}_\mathcal{P}$.
	We now propose a dominance measure that satisfies Property~\ref{prop:bounds_with_interpretation} and, thus, can be used to identify cases in which $X_A \succ X_B$.
	
	\subsection{$\mathcal{C}_\mathcal{D}$: dominance rate}
	\label{sec:dominance_rate}

	Intuitively, the \textit{dominance rate} is a dominance measure that quantifies the extent to which $X_A$ has a lower cumulative distribution function than $X_B$, normalized by the portion of the probability densities with different cumulative distributions.

	\begin{mydef}
		\label{def:dominance_density_function}
		(Dominance density function) 
		Let $X_A$ and $X_B$ be two continuous random variables defined in a connected set $N\subseteq \mathbb{R}$.
		We define the dominance density function as follows:
		
		$$
		\mathcal{D}_{X_A,X_B}(x) = 
			\begin{cases}
		 \hspace{0.8em}	g_A(x) \cdot k_A  &\quad\text{if}\quad G_A(x) > G_B(x) \\
			 -g_B(x) \cdot k_B &\quad\text{if}\quad G_A(x) < G_B(x)\\
		\hspace{0.8em}	0 &\quad\text{otherwise.} \\ 
			\end{cases} 
		$$
		where $k_A = \left( \int_{\{x\in N \ | \ G_A(x) \neq G_B(x)\}} g_A(t) dt\right) ^{-1}$ is the normalization constant and $k_B$ is defined likewise.
	\end{mydef}
	
	Note that the dominance density function is not correctly defined when $\int_{N} |g_A(x) - g_B(x)|dx = 0$.

	\begin{mydef}
		\label{def:dominance_rate}
		(Dominance rate)
		Let $X_A$ and $X_B$ be two continuous random variables defined in a connected set $N\subseteq \mathbb{R}$.
		The dominance rate of $X_A$ over $X_B$ is defined as
		$$
		\mathcal{C}_{\mathcal{D}}(X_A,X_B) = 
		\begin{cases}
		0.5 ,  \quad \text{if} \int_{N} |g_A(x) - g_B(x)|dx = 0 \\
		0.5   \int_{N} \mathcal{D}_{X_A,X_B}(t) dt + 0.5 , \quad \text{otherwise.} \\ 
		\end{cases} 
		$$
	\end{mydef}

	Basically, we are measuring the amount of mass of $X_A$ in which $G_A(x) > G_B(x)$ minus the amount of mass of $X_B$ in which $G_A(x) < G_B(x)$.
	This value is then normalized so that all sections in which $G_A(x) = G_B(x)$ are ignored, i.e.\vspace{0.5em}
	$\int_{N} \mathcal{D}_{X_A,X_B}(t) dt = $
	$$\dfrac{\mathbb{E}_A[\mathcal{I}[G_A(x)>G_B(x)]]}{\mathbb{E}_A[\mathcal{I}[G_A(x) \neq G_B(x)]]} - \frac{\mathbb{E}_B[\mathcal{I}[G_A(x) < G_B(x)]]}{\mathbb{E}_B[\mathcal{I}[G_A(x) \neq G_B(x)]]}$$
	Finally, we apply the linear transformation $l(x) = 0.5x - 0.5$ ensuring the dominance rate is defined in the interval $[0,1]$ (instead of $[-1,1]$), required to comply with Property~\ref{prop:bounds_with_interpretation}.

	From
	 
	i) $\mathcal{C}_{\mathcal{D}}(X_A,X_B) = 1 \iff X_A \succ X_B$ and 
	
	ii) $\mathcal{C}_{\mathcal{D}}(X_A,X_B) = 0 \iff X_B \succ X_A$,
	
	 we deduce that the dominance rate satisfies Property~\ref{prop:bounds_with_interpretation}.
	Note that the previous deduction is only possible when $g_A$ and $g_B$ are bounded, as this implies that $G_A$ and $G_B$ are continuous.
	Specifically, it is enough to find a point in $N$ in which $G_A(x) > G_B(x)$ to satisfy that $\int_{x \in \{t\in N \ | \ G_A(t) > G_B(t)\}} g_A(x) dx > 0$, and this point is guaranteed to exist when $X_A \succ X_B$ because of the definition of the dominance.
	The dominance rate is also a well behaved dominance measure, as it satisfies Properties~\ref{prop:bounds_with_interpretation}, \ref{prop:antisimmetry}, \ref{prop:inverse}, \ref{prop:equality_value}, \ref{prop:translation}, \ref{prop:scaling} and \ref{prop:partial_translation_within_support}.

	We have seen that the dominance measures $\mathcal{C}_\mathcal{P}$ and $\mathcal{C}_\mathcal{D}$ satisfy most of the properties listed in Section~\ref{section:properties_of_comparison_comparison_functions}. 
	As we will see in the next section, their values are related.

	\subsection{The relationship between $\mathcal{C}_\mathcal{P}$ and $\mathcal{C}_\mathcal{D}$}	
	
	In Section~\ref{section:background} we stated that $\mathcal{C}_\mathcal{P} = 1$ is a stronger condition than $\mathcal{C}_\mathcal{D} = 1$, because $\mathcal{C}_\mathcal{P}(X_A,X_B) = 1$ implies that for all $x$ in $N$ that $G_A(x) < 1$, $G_B(x) = 0$.
	On the other hand,  $\mathcal{C}_\mathcal{D} = 1$ implies that $X_A \succ X_B$ (the two conditions in Definition~\ref{def:dominance}), which is weaker.
	In the diagram below, we show the values of $\mathcal{C}_\mathcal{P}$ and $\mathcal{C}_\mathcal{D}$ that imply other values of $\mathcal{C}_\mathcal{P}$ and $\mathcal{C}_\mathcal{D}$.
	Each arrow can be interpreted as an implication.
	The implications are transitive: i.g. $\mathcal{C}_\mathcal{D}(X_A,X_B) = 1$ implies $\mathcal{C}_\mathcal{P}(X_A,X_B) > 0.5$.
	
	\begin{figure}[!h]
		\centering
%		\tikzset{extraright/.style={xshift=2.5cm}}
		\begin{tikzpicture}[node distance=1.25cm and 0.8cm,
		every node/.style={fill=white, font=\sffamily}, align=center]
		% Specification of nodes (position, etc.)
	    \node (Pe1)    []          {$\mathcal{C}_\mathcal{P}(X_A,X_B) = 1$};
		\node (De1)  [below of=Pe1]        {$\mathcal{C}_\mathcal{D}(X_A,X_B) = 1$};
		\node (Pg05) [xshift=1.75cm, below of=De1]                   {$\mathcal{C}_\mathcal{D}(X_A,X_B) > 0.5$};
		\node (Dg05) [xshift=-1.75cm, below of=De1]    {$\mathcal{C}_\mathcal{P}(X_A,X_B) > 0.5$};
		\node (invNode2) [xshift=-6cm, yshift=1.5cm, below of=Pg05] {\color{black!50}\small };%Less strict};
		\node (inv1) [yshift=1.5cm, above of=invNode2] {\color{black!50}\small };%More strict};
		
		% Specification of lines between nodes specified above
		% with aditional nodes for description 
		\draw[->] (Pe1) -- (De1);
		\draw[->] (De1) -- (Dg05);
		\draw[->] (De1) -- (Pg05);
		%\draw[->] (Pg05) -- (Dg05);
		\end{tikzpicture}
		\caption{Implications between the values of $\mathcal{C}_\mathcal{P}$ and $\mathcal{C}_\mathcal{D}$.}
	\end{figure}
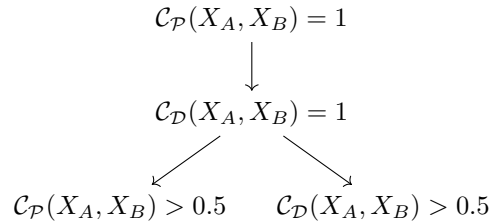

	\subsection{Estimating $\mathcal{C}_\mathcal{P}$ and $\mathcal{C}_\mathcal{D}$}	
	\label{section:simple_estimators_for_cp_and_cd}	
	In the previous sections, we have assumed that the random variables $X_A$ and $X_B$ are known, but usually, we only have a few observed values from each random variable.
	Therefore, it may be interesting to estimate $\mathcal{C}_\mathcal{P}$ and $\mathcal{C}_\mathcal{D}$ from the observed samples.
	With this purpose, we propose the following empirical estimates of $\mathcal{C}_\mathcal{P}$ and $\mathcal{C}_\mathcal{D}$.

	\begin{mydef}(estimation of $\mathcal{C}_\mathcal{P}$) \\
		\label{def:estimation_of_CP_simple}
		Let $X_A$ and $X_B$ be two continuous random variables and $A_n = \{a_1,...,a_n\}$ and $B_n = \{b_1,...,b_n\}$ their $n$ observations respectively.
		We define the estimation of the probability that $X_A$ < $X_B$ as 
		$$\widetilde{\mathcal{C}_\mathcal{P}}(A_n, B_n) =  \sum_{i,k=1...n} \dfrac{\sign(b_k -a_i)}{2n^2} + \frac{1}{2}.$$
	\end{mydef}

	This estimator is well known in the literature because it is the U statistic of the Mann-Whitney test~\cite{10.2307/2236101}.

	\begin{mydef}(estimation of $\mathcal{C}_\mathcal{D}$) \\
		\label{def:estimation_of_CD_simple}
		Let $X_A$ and $X_B$ be two continuous random variables and $A_n = \{a_1,...,a_n\}$ and $B_n = \{b_1,...,b_n\}$ their $n$ observations respectively.
		Let $C_{2n} = \{c_j\}_{j=1}^{2n}$ be the list of all the sorted observations of $A_n$ and $B_n$ where $c_1$ is the smallest observation and $c_{2n}$ the largest.
		Suppose that $a_i \neq b_k$ for all $i,k=1,...,n$. 
		We define the estimation of the dominance rate as 
		$$\widetilde{\mathcal{C}_\mathcal{D}}(A_n, B_n) =  \sum_{j=1}^{2n} \dfrac{    \mathcal{I}(\hat{G}_A(c_j) > \hat{G}_B(c_j) \land c_j \in A_n)    }{2n}  - $$
		$$\sum_{j=1}^{2n} \dfrac{    \mathcal{I}(\hat{G}_A(c_j) < \hat{G}_B(c_j) \land c_j \in B_n)    }{2n} + \frac{1}{2}.$$
		
		where $\mathcal{I}$ is the indicator function and $\hat{G}_A(x)$ and $\hat{G}_B(x)$ are the empirical distributions estimated from $A_n$ and $B_n$ respectively.
	\end{mydef}

	For simplicity, this estimator of the dominance rate assumes there are no repeated samples. 
	However, it can be extended to take into account repeated values (see Appendix~\ref{appendix:proof_difference_graph}).

	\clearpage
	
	\section{Cumulative difference-plot}
	\label{section:limitedata}
	In this section, we propose a graphical method called \textit{cumulative difference-plot} that shows the estimations of $\mathcal{C}_\mathcal{P}$ and $\mathcal{C}_\mathcal{D}$ decomposed by quantiles: $\mathcal{C}_\mathcal{P}$ and $\mathcal{C}_\mathcal{D}$ can be visually estimated from the difference plot.
	In addition, the proposed plot allows a comparison of quantiles of the two random variables.
	The proposed approach also models the uncertainty associated with the estimation of the \textit{cumulative difference-plot} from the data.

	\subsection{Quantile random variables}
	\label{section:quantile_rv}

	From a practical point of view, it is unlikely that the probability densities of the compared random variables $X_A$ and $X_B$ are known.
	Usually, we only have $n$ observations ${A_n = \{a_1,...,a_n\}}$ and ${B_n = \{b_1,...,b_n\}}$ from each random variable.
	The proposed \textit{cumulative difference-plot} is based on two random variables $Y_A$ and $Y_B$ that are defined with these observations.
	Specifically, we define the densities of the two \textit{quantile random variables} $Y_A$ and $Y_B$ as a mixture of several uniform distributions in the interval $[0,1]$.

	\begin{figure}
		\centering
		\includegraphics[width=0.75\linewidth]{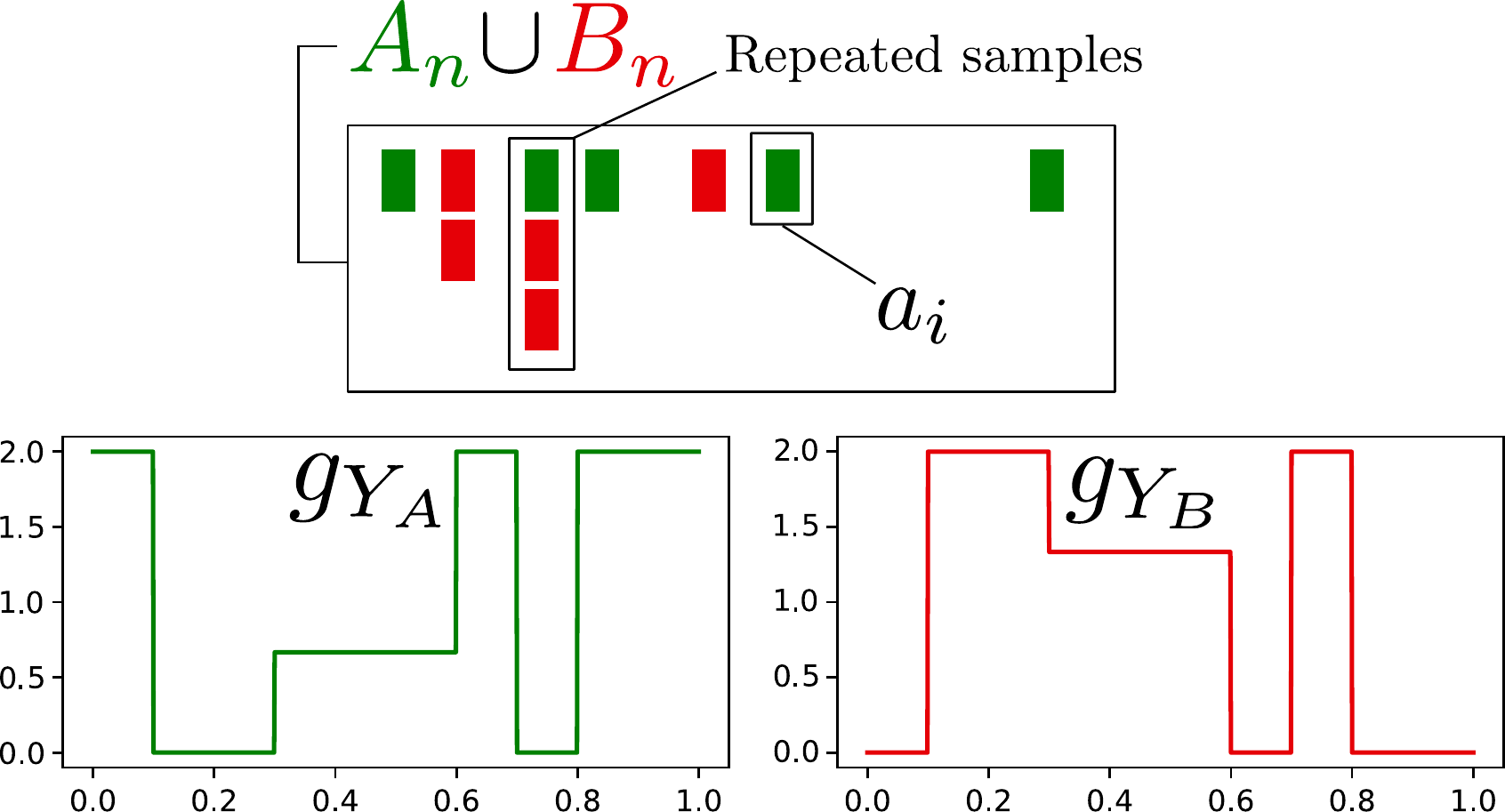}
		\caption{An example of the probability density functions of $Y_A$ and $Y_B$ given the observed samples $A_n \cup B_n$.}
		\label{fig:from_samples_inkscape_Y_A_and_Y_B}
	\end{figure}

	The uniform distributions in the quantile random variables are placed according to their rank in $A_n \cup B_n$.
	Assuming no repetitions, for each value $k$ in $A_n \cup B_n$, its corresponding kernel is centered in $\frac{\text{rank}(k) + 0.5}{2n}$ where $\text{rank}(k)$ is the rank of $k$ in $A_n \cup B_n$.
	The kernels have a bandwidth of $1/4n$, ensuring that the sum of the densities of $Y_A$ and $Y_B$ is constant.
	If there are repeated values in $A_n \cup B_n$, their corresponding kernel is placed at the middle of the previous and the next rank, and the width of the kernel is increased proportionally with respect to the number of repetitions.
	See Figure~\ref{fig:from_samples_inkscape_Y_A_and_Y_B} for an example.
	In Appendix~\ref{appendix:estimation_densityY_A_step_by_step} we show how to compute the probability densities of $Y_A$ and $Y_B$ step by step.

	A more simple approach would be to estimate and define the quantile random variables through the empirical cumulative distribution functions of the observed samples of $X_A$ and $X_B$.
	However, the quantile random variables defined through uniform kernels have some interesting properties: they have the same $\mathcal{C}_\mathcal{P}$ and $\mathcal{C}_\mathcal{D}$ as the kernel density estimations of $X_A$ and $X_B$ (shown in Appendix~\ref{appendix:Y_A_mantains_same_mathcalC}).
	In addition, $g_{Y_A}(x) + g_{Y_B}(x) = 2$ for all ${x \in [0,1]}$. 
	As we will later see, these properties are essential for the interpretation of the \textit{cumulative difference-plot}.

	\subsection{Confidence bands}
	\label{section:comp_cumulative_limited_data}

	The \textit{cumulative difference-plot} is based on the cumulative distribution functions of $Y_A$ and $Y_B$, which are estimated from the observed samples.
	This means that we need to model the uncertainty associated with the estimations.
	Confidence bands are a suitable choice in this scenario: a confidence band is a region in which the cumulative distribution is expected to be with a certain confidence.
	The size of the band is determined by the number of samples and the desired level of confidence: a high number of samples or a low level of confidence are associated with a small band size.
	There is an extensive literature~\cite{cheng_confidence_1983,10.2307/2958458,cheng1988one,wangSmoothSimultaneousConfidence2013, doi:10.1080/01621459.1990.10474983,doi:10.1080/01621459.1989.10478742,hall_simple_2013} on how to estimate the confidence bands of cumulative distributions, and, in this work, we use a simple bootstrap approach\footnote{
		The bootstrapping~\cite{efron_introduction_1993} method involves considering the observed values as a population from which random samples with replacement are drawn.
		These samples are then used to estimate the upper and lower pointwise confidence intervals of the cumulative distribution of $Y_A$ and $Y_B$.
		Since a pointwise estimation of the confidence interval is used, we can expect that a portion proportional to $\alpha$ will fall outside the confidence band.\\
		Note that we are interested in having an overall confidence of $1 - \alpha$, thus, we want that the cumulative distributions of $Y_A$ and $Y_B$ are inside their confidence bands at the same time with this level of confidence~\cite{goemanMultipleHypothesisTesting2014,bauerMultipleTestingClinical1991}.
		This means that we have to use a higher confidence level for each band: $\sqrt{(1 - \alpha)}$.
	}.

	\begin{figure}
		\centering
		\includegraphics[width=0.45\textwidth]{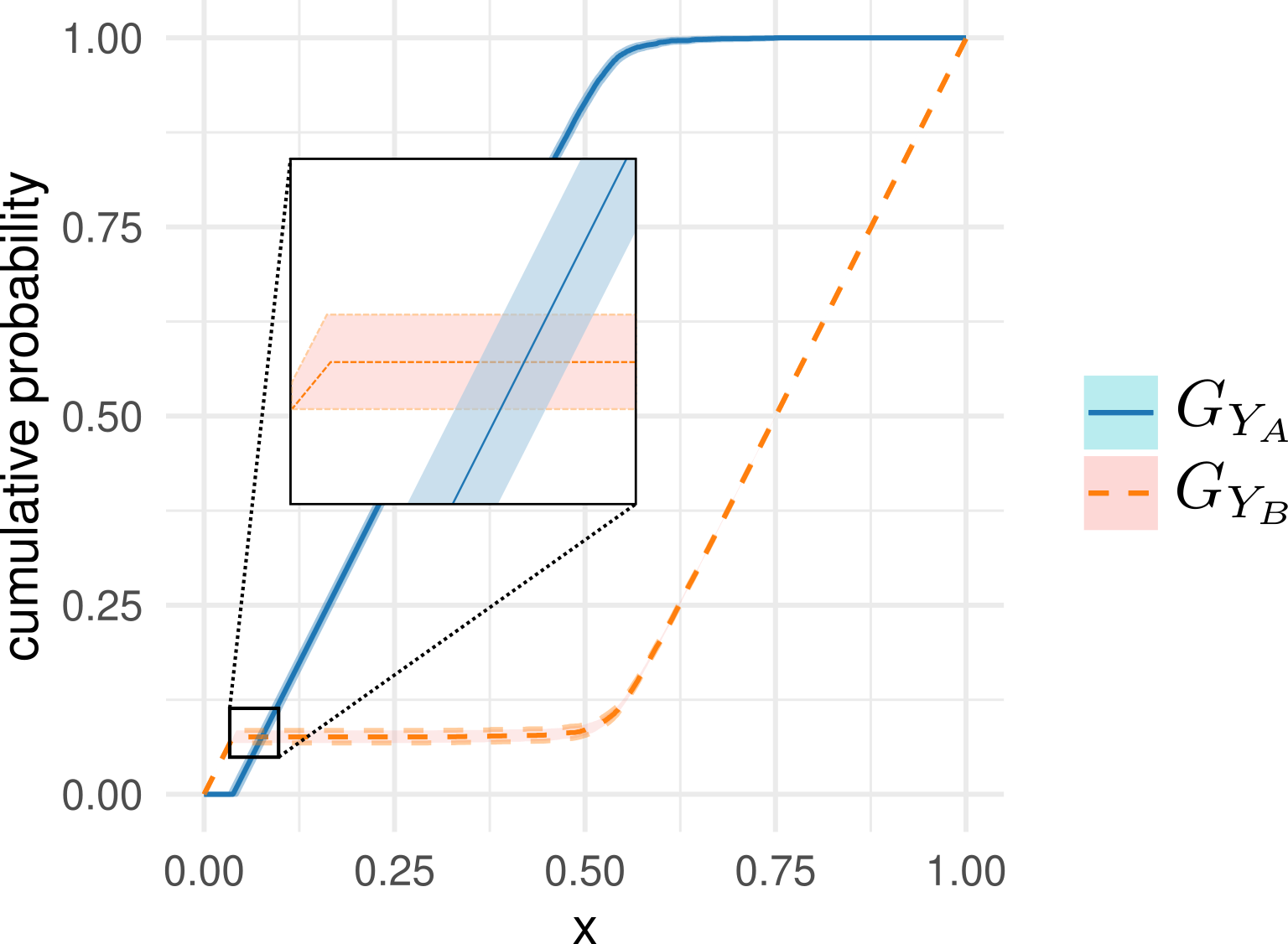}
		\caption
		{
			The confidence bands of the cumulative distributions of the quantile random variables $Y_A$ and $Y_B$ corresponding to the distributions $X_A$ and $X_B$ in Case 2 (shown in Figure~\ref{fig:example_2_mean_median_prob_better}).
		}
		\label{fig:example2_bands_x_primaAB} 
	\end{figure}

	To illustrate how to interpret the confidence bands of the cumulative distributions of $Y_A$ and $Y_B$, we will assume that we have observed $n = 400$ samples from each random variable $X_A$ and $X_B$ from Case 2 (see Figure~\ref{fig:example_2_mean_median_prob_better} in Section~\ref{section:background}).
	We show the 95\% confidence bands of the cumulative distribution functions of $Y_A$ and $Y_B$ in Figure~\ref{fig:example2_bands_x_primaAB}.
	The estimated cumulative distribution functions of $Y_A$ and $Y_B$ resemble the cumulative distribution functions of $X_A$ and $X_B$ from Figure~\ref{fig:example_2_pareto}.
	However, there are several relevant differences.
	In Figure~\ref{fig:example2_bands_x_primaAB}, we observe that $Y_A$ and $Y_B$ are defined in the interval $[0,1]$, while the cumulative distribution functions of $X_A$ and $X_B$ are defined in the sample space.
	Each of the values in this interval can be used to deduce the distribution with the lowest quantile: at $x=0.5$, the cumulative distribution function of $Y_A$ is larger than the cumulative distribution function of $Y_B$, hence, the median of $X_A$ is lower than the median of $X_B$.
%	Note that since the confidence bands overlap at $x = 0.5$, there is uncertainty associated with this statement.
%	A higher number of samples could reduce this uncertainty. 
	In addition, the sum of the density function of $Y_A$ and $Y_B$ is constant.
	As a result, unlike $X_A$ and $X_B$, the probability density functions of $Y_A$ and $Y_B$ do not have large areas where the probability density is zero.

	\subsection{The cumulative difference-plot}
	\label{section:the_difference_plot}

	In this section, we introduce a new graphical method designed to visually analyze the dominance of $X_A$ and $X_B$.
	Without loss of generality, a minimization\footnote{
		Note that if the random variables being compared take values in a maximization setting (higher values are preferred), then the random variables need to be redefined as the inverse with respect to the sum (this simply means the sampled values are multiplied by $-1$) before generating the cumulative difference-plot.
		With this change, the interpretation of the cumulative difference-plot is consistent and intuitive: for either minimization or maximization, on the left side of the cumulative difference-plot, the most desirable values that the random variables take are compared.
		If the difference is positive on the left side of the cumulative difference-plot, then the best values that $X_A$ takes are better than the best values that $X_B$ takes.
		Similarly, the worst values are compared on the right side of the cumulative difference-plot: if the difference is positive on this side, then the worst values of $X_A$ are better than the worst values of $X_B$.	
	} setting is assumed: lower values in $X_A$ and $X_B$ are preferred to higher values.
	It builds upon the difference function defined as 
	
	\begin{equation}
	\label{eq:difference_function_for_cum_plot}
	\begin{aligned}
	\diff \colon [0,1] & \longrightarrow \left[ -1,1\right]  \\
	x\ \ 				 & \longmapsto G_{Y_A}(x) -  G_{Y_B}(x),
	\end{aligned}
	\end{equation}

	where $G_{Y_A}(x)$ and $G_{Y_B}(x)$ are the cumulative distribution functions of $Y_A$ and $Y_B$, respectively.

	The \textit{cumulative difference-plot} is the plot of the difference function (the difference between the cumulative distributions of $Y_A$ and $Y_B$), including a confidence band.
	A positive value in the \textit{cumulative difference-plot} can be interpreted as a quantile in which the cumulative distribution function of $X_A$ is larger than the cumulative distribution function of $X_B$.
	Hence, if the difference is positive at $0.5$, the median of $X_A$ is lower than the median of $X_B$ (assuming minimization).
	In this sense, the best values obtained by both random variables are compared on the left side, and the worst values are compared on the right side.
	
	\subsubsection{$\mathcal{C}_\mathcal{P}$ and $\mathcal{C}_\mathcal{D}$ in the cumulative difference-plot}
	 
	$\mathcal{C}_\mathcal{P}$ and $\mathcal{C}_\mathcal{D}$ can be directly obtained from the proposed plot.
	The integral of the difference between $Y_A$ and $Y_B$ is $\mathcal{C}_\mathcal{P} - 0.5$ (we prove this in Appendix~\ref{appendix:proof_difference_graph}).
	Formally, $\mathcal{C}_\mathcal{P} = 0.5 + \int_{0}^{1}  \diff(x) dx$.
	However, in practice, $\mathcal{C}_\mathcal{P}$ can be visually estimated by adding 0.5 to the difference in the areas over and under 0.
	For the example shown in Figure~\ref{fig:plotexamplecdandcpindifferencegraph}, $\mathcal{C}_\mathcal{P} = 0.5 - \text{Area1} + \text{Area2}$.
	The difference can only be in the area highlighted in blue in the cumulative difference-plot.
	When the probability that $X_A < X_B$ is $1$, the difference is at its maximum:	in the cumulative difference-plot we see a line from $(x, f(x))=(0,0)$ to $(0.5,1)$ and from $(0.5,1)$ to $(1,0)$.
	Similarly, when the probability that $X_A < X_B$ is $0$, the difference between $Y_A$ and $Y_B$ is equal to the lowest possible values inside the light blue area.

	By contrast, $\mathcal{C}_\mathcal{D}$ is represented in the plot as the total length in which the difference is positive minus the total length in which the difference is negative.
	Specifically, 
	\begin{equation}
		\label{equation:C_D_from_cumulative_diff_plot_visually}
		\mathcal{C}_\mathcal{D} = \frac{  \dfrac{\int_{0}^{1}\mathcal{I}[\diff(x) > 0] - \mathcal{I}[\diff(x) < 0] dx}{2} + \frac{1}{2}}{\int_{0}^{1}\mathcal{I}[\diff(x) \neq 0]dx},
	\end{equation}
	where $\mathcal{I}$ is the indicator function (we prove this in \nobreak{Appendix}~\ref{appendix:proof_difference_graph}).
	As an example, $\mathcal{C}_\mathcal{D}$ is proportional to ${\text{Length2} - \text{Length1}}$ in Figure~\ref{fig:plotexamplecdandcpindifferencegraph}: it is higher than $0.5$, because ${\text{Length2} > \text{Length1}}$.
	In this example, there is no need to divide by the total length in which the difference is nonzero because the difference is zero in only a limited number of points.
	In such cases, $\mathcal{C}_\mathcal{D}$ can  also be estimated as the total length in which ${\diff(x) > 0}$.
	In the example in Figure~\ref{fig:plotexamplecdandcpindifferencegraph}, the estimation is ${\mathcal{C}_\mathcal{D} = \text{Length2} \approx 0.75}$.
	Note that Equation~\eqref{equation:C_D_from_cumulative_diff_plot_visually} is not correctly defined when $Y_A$ and $Y_B$ are equal, but this is an easy case to identify, as the difference is constantly $0$.

	\begin{figure}
		\centering
		\includegraphics[width=0.45\linewidth]{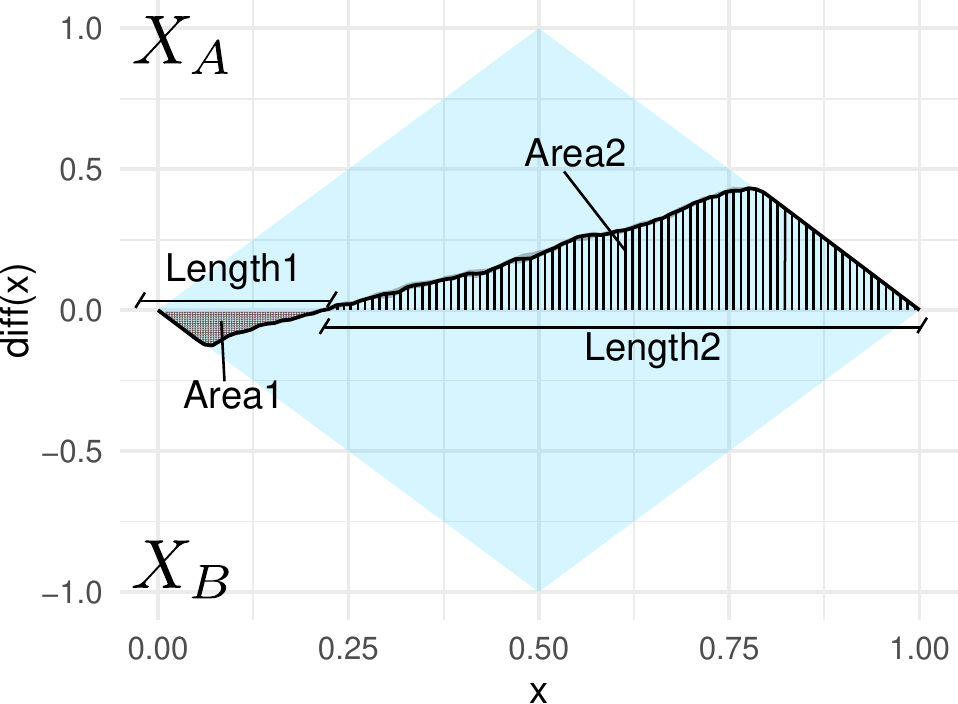}
		\caption{The areas and lengths in the cumulative difference-plot that can be used to deduce $\mathcal{C}_\mathcal{P}$ and $\mathcal{C}_\mathcal{D}$.}
		\label{fig:plotexamplecdandcpindifferencegraph}
	\end{figure}

	\begin{figure}
		\centering
		\includegraphics[width=0.45\linewidth]{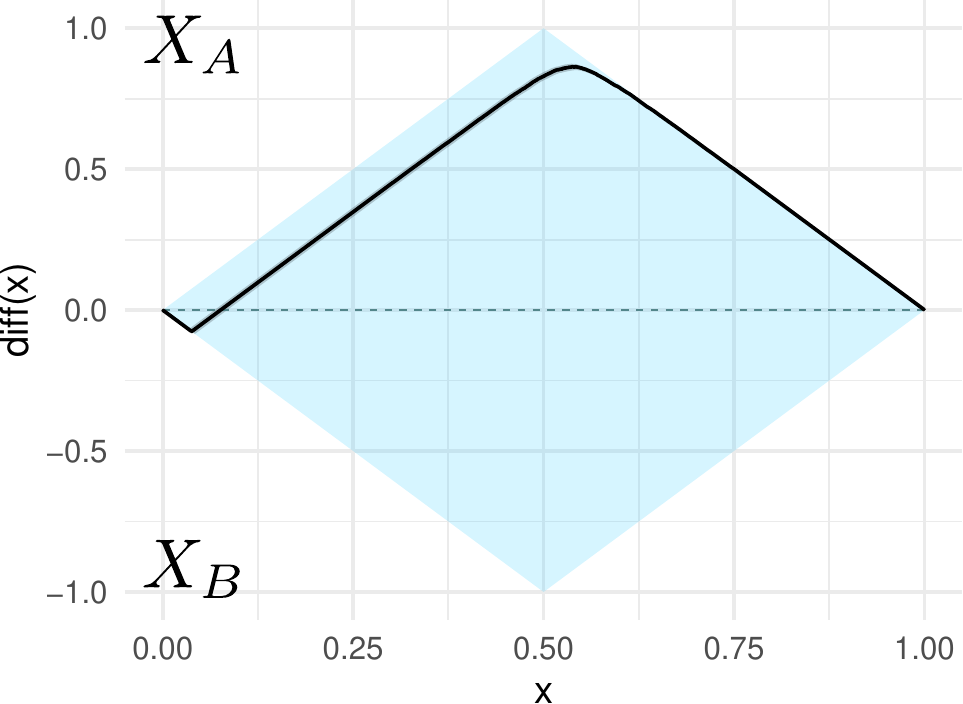}
		% Creo que la razon por la que puede salir la confidence band del triangulo superior (o inferior) es pq no es una confidence band de la resta, sino que es el la differencia entre los confidence bands mas extremos. 
		\caption
		{
			The cumulative difference-plot for Case 2: the difference between the cumulative distribution functions of $Y_A$ minus $Y_B$ corresponding to the distributions $X_A$ and $X_B$ in Case 2 (shown in Figure~\ref{fig:example_2_mean_median_prob_better}).
		}
		\label{fig:Example2_xprimaABDiff} 
	\end{figure}

	\subsubsection{Illustrative example}
	
	Figure~\ref{fig:Example2_xprimaABDiff} shows the cumulative difference-plot for the random variables $X_A$ and $X_B$ from Case 2 (their densities were shown in Figure~\ref{fig:example_2_mean_median_prob_better}).
	First, we see that the difference is both negative and positive, hence, neither random variable dominates the other.
	The difference is negative when $x = 0.05$ or lower. 
	This can be interpreted as $X_B$ having a smaller $5\%$ quantile than $X_A$.
	The difference is positive otherwise, thus we deduce from the cumulative difference-plot that the $25\%$, the $50\%$ (the median), the $75\%$ and $95\%$ quantiles are smaller in $X_A$ than in $X_B$.
	In other words, the random variable $X_B$ can take really low values with a small probability, but apart from these really low values, $X_A$ takes lower values than $X_B$.
	This is also reflected by $\mathcal{C}_\mathcal{D}(X_A,X_B) > 0.75$, as deduced from $\diff(x) > 0$ for all $x \in (0.25,1)$.
	The difference is also near its maximum value, implying that its integral is high and thus $\mathcal{C}_\mathcal{P}$ (the probability that $X_A < X_B$) is also near 1.

	\section{Related work}
	\label{section:related_work}

	Statistical assessment of experimental results is a very studied research topic.
	In this section, we locate our proposal in the field and focus on similarities and differences with respect to other random variable comparison methods.

	\subsubsection{Visualizing densities}
	As mentioned in the introduction, it makes sense to model the performance of stochastic optimization algorithms as random variables.
	Therefore, statistical tools that compare random variables have become an increasingly important part of the analysis of experimental data.
	Among these tools, visualization techniques such as histograms or box-plots are usually applied before the rest of the methods.
	The advantage of these methods is their simplicity.
	If one of the random variables clearly takes lower values than the other, then these two methods effectively convey this message simply and naturally.
	Unfortunately, when both random variables have similar probability densities, these two methods might fail to represent the random variables in a way that makes it easy to compare them (example shown in Section~\ref{sec:case_study_histogram_and_boxplot}).

	The simplicity of these methods is also a drawback: for example, they give no information about the uncertainty associated with the estimates.
	The histogram suffers from the bin positioning bias~\cite{thas2010comparing,scikit-learndevelopersDensityEstimation2021}.
	A kernel density estimation with the uniform kernel---considered to be the moving window equivalent of the histogram---overcomes this limitation~\cite{thas2010comparing}, at the cost of using a more complex model.
	Similarly, the box-plot has a ``non\nobreakdash-injectivity'' problem: very different data can still have the same box-plot~\cite{matejka2017same,chatterjee_generating_2007}.
	The violin-plot is an extension of the box-plot that overcomes the above limitation by combining the kernel density estimate of the random variables with the traditional box-plot~\cite{hintze1998violin}.
	The proposed cumulative difference-plot improves on these methods because it represents the data clearly, even when the two random variables being compared are similar.

	\subsubsection{Null hypothesis statistical testing}
	\label{sect:null_hypothesis_testing_explaination}
	
	Null hypothesis tests can be used to compare random variables without having to visually represent them.
	In a very general way, carrying out a null hypothesis test involves the following: first, a null hypothesis is proposed.
	Under certain assumptions, the null hypothesis implies that a given statistic obtained from the data follows a known distribution.
	Then, assuming the null hypothesis is true, the probability of obtaining data with a more extreme statistic value\footnote{The definition of what \textit{data with a more extreme statistic value} is not the same for every null hypothesis test, and it depends on the test being used.} than the observed~\cite{conover1980practical} is computed.
	When the probability under the null hypothesis of the observed statistic is lower than a predefined threshold, the null hypothesis is rejected and the alternative hypothesis is accepted~\cite{greenland_statistical_2016}.
	Usually, this threshold is set at an arbitrary but well established~\cite{wasserstein_asa_2016} $p = 0.05$, although recently, further reducing the threshold to $p = 0.005$ has been proposed~\cite{benjamin_redefine_2018,ioannidis2018proposal}.

	In the context of comparing two random variables $X_A$ and $X_B$, in general, we cannot assume that a statistic obtained from the data follows a known distribution under the null hypothesis.
	In this case, a non-parametric test~\cite{conover1980practical} is a suitable choice.
	Specifically, the Mann-Whitney~\cite{10.2307/2236101} test is a good choice, as the samples observed from the random variables are i.i.d for each random variable~\footnote{
		For paired data, the Wilcoxon signed-rank test~\cite{wilcoxonIndividualComparisonsRanking1945} or the sign test~\cite{conover1980practical} should be used.
		However, in the context of this paper, the samples observed from the random variables are not paired.
		In this paper, we consider the Mann-Withney test as it is probably the most well known non-parametric test for unpaired data, although take into account that more modern alternatives have been proposed~\cite{ledwinaNonparametricTestsStochastic2012,baumgartnerNonparametricTestGeneral1998,biswasNonparametricTwosampleTest2014}.
	}.
	With this test, the null hypothesis is that $\mathcal{P}(X_A >X_B) = \mathcal{P}(X_B > X_A)$, and a possible alternative hypothesis is that $X_A \succ X_B$~\cite{10.2307/2236101}.

	Null hypothesis tests have some limitations: for example, the \textit{p}-value does not separate between the effect size and the sample size~\cite{benavoli2017time,calvoBayesianPerformanceAnalysis2019}.
	In addition, rejecting the null hypothesis does not always mean that there is evidence in favor of the alternative hypothesis: it just means that the observed statistic (or a more extreme statistic) is very unlikely when the null hypothesis is true.

	To show this, we generate $400$ samples of the distributions $X_A$ and $X_B$ from Case 2 (density functions shown in Figure~\ref{fig:example_2_mean_median_prob_better}) and we apply the Mann-Whitney test, rejecting the null hypothesis when $p < 0.005$.
	If we repeat this experiment $10^4$ times (with different samples each time), the null hypothesis is rejected every time\footnote{The source code to replicate this experiment is available in the file \href{https://github.com/EtorArza/SupplementaryPaperRVCompare/blob/main/mann_whitney_counter_example.R}{mann\_whitney\_counter\_example.R} in our \href{https://github.com/EtorArza/SupplementaryPaperRVCompare}{Github repository}.}.
	However, $X_A \nsucc X_B$ and $X_B \nsucc X_A$, implying that the alternative hypothesis is not true.
	Note that the proposed cumulative difference-plot (shown in Figure~\ref{fig:Example2_xprimaABDiff}) avoids this problem because it correctly points out that neither random variable dominates the other one, for the same case and with the same number of samples.

	\subsubsection{Bayesian analysis}
	\label{sec:bayesian_simplex_explained}
	
	As an alternative~\cite{benavoli2017time, calvoBayesianPerformanceAnalysis2019} to the limitations of null hypothesis test, Bayesian analysis has been proposed.
	Bayesian analysis \cite{gelmanBayesianDataAnalysis2014,bernardo2009bayesian} estimates the probability that a hypothesis is true, conditioned to the observed data.
	This estimation requires the prior probabilities of the hypotheses and the data, but usually, they are assumed to follow a distribution that gives equal probability to all hypotheses and data.
	Recently a Bayesian version of the Wilcoxon signed-rank test~\cite{benavoli2014bayesian,benavoli2017time} has been proposed.
	In this paper, we will consider the simplex-plot of its posterior distribution.
	For convenience, in the rest of the paper, we will call it \textit{simplex-plot}.
%	Section~4.2.1 of the work by Benavoli \textit{et al.}~\cite{benavoli2017time} introduces the Bayesian Wilcoxon signed-rank test in detail and Section~4.3 of the same paper offers a detailed explanation on the choice of prior and the effect it has on the posterior distribution.

	Once the posterior distribution is known, the probability that the difference between a sample from $X_A$ and a sample from $X_B$ is in the interval $(-\infty,-r)$, $[-r,r]$ or $(r, + \infty)$ can be computed.
	These probabilities can be interpreted as the probability that $X_A > X_B$, $X_A = X_B$ and $X_B > X_A$, where two samples $x_a$ and $x_b$ are considered equal when $|x_a - x_b| \leq r$.
	Note that the simplex-plot is just a convenient representation of the posterior distribution, where `rope' or \textit{range of practical equivalence} denotes hypothesis $X_A = X_B$ (when the difference is in the interval $[-r,r]$).

	We computed the simplex-plot (Figure~\ref{fig:simplexcase2}) with the $400$ samples of $X_A$ and $X_B$ from Case 2 obtained in Section~\ref{section:comp_cumulative_limited_data}.	
	Two samples were considered equal when their difference is lower than $r = 10^{-3}$, and we used the prior proposed by \cite{benavoli2017time}.
	We can deduce from this figure that the hypothesis $X_B > X_A$ is much more likely than $X_A = X_B$ or $X_B > X_A$.

\begin{figure}
	\centering
	
	\includegraphics[width=0.35\linewidth]{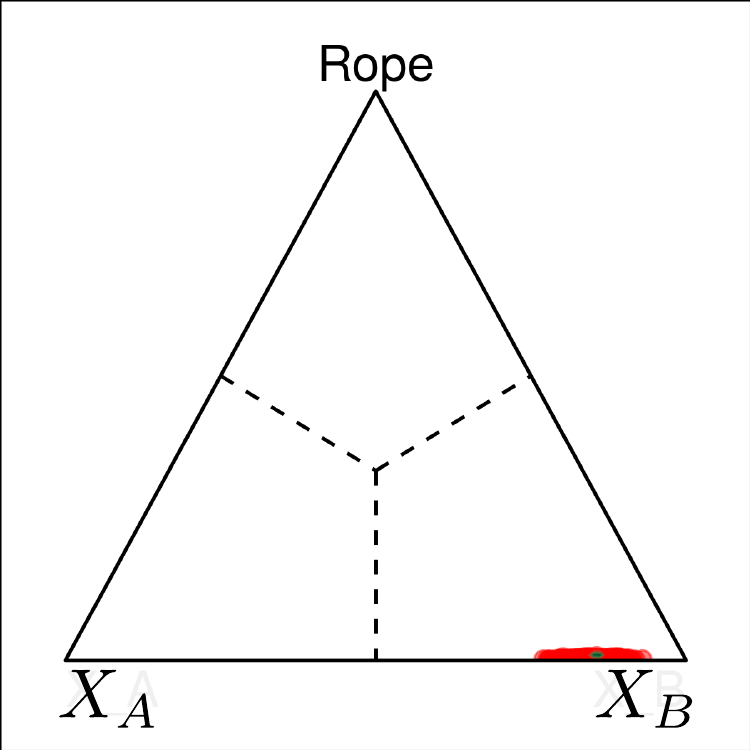}
	\caption{
		The simplex-plot computed with the package \textit{scmamp}~\cite{calvoScmampStatisticalComparison2016} of the posterior distribution for Case 2.
	}
	\label{fig:simplexcase2}
\end{figure}

	The simplex plot summarizes the data through the probabilities of $X_B > X_A$, $X_A = X_B$ or $X_B > X_A$, but does not offer any additional information: we cannot deduce from these probabilities in which intervals the values of a random variable are lower than the other.
	In this sense, the cumulative difference-plot is a more detailed comparative visualization. 
	Specifically, the observation that the $1\%$ lowest values of $X_A$ are lower than the $1\%$ lowest values of $X_B$ cannot be deduced from the simplex-plot, while it is easy to see in the cumulative difference-plot.
	Also, the cumulative difference-plot shows a comparison of the cumulative distributions through the dominance rate, while the simplex-plot does not.

	\subsubsection{Other plots in the interval $[0,1]$}

	The probability-probability plot is defined as 
	$$PP:[0,1]\rightarrow [0,1]^2 : p \rightarrow(p, G_A(G_B^{-1}(p))).$$
	As proposed by Schmid et al.~\cite{schmid_testing_1996}, it can be interpreted via the integral of the non-negative part, which represents the amount of violation against the hypothesis that $X_A$ dominates $X_B$.

	The quantile-quantile plot~\cite{thas2010comparing,wilkProbabilityPlottingMethods1968} is defined as
	$$QQ:[0,1]\rightarrow N^2 : p \rightarrow(G_A^{-1}(p), G_B^{-1}(p)),$$
	and it is a natural way to visualize the differences in quantiles of $X_A$ and $X_B$ in $N$ (the domain of definition of the random variables).

	The quantile-quantile plot also allows a comparison between quantiles, just like the \textit{cumulative difference-plot}.
	However, the cumulative difference-plot proposed in this paper is distinct from the two plots above in three aspects: 
	i) the proposed cumulative difference-plot is defined directly from the observed samples.
	Because of its definition, it has a confidence band built-in, which allows the uncertainty associated with the estimation to be directly interpreted within the plot.
	ii) The proposed cumulative difference-plot contains several statistics simultaneously. 
	Specifically, the estimated $\mathcal{C}_\mathcal{D}$, $\mathcal{C}_\mathcal{P}$ and the comparison of the quantiles can be visually interpreted.
	iii) The proposed plot is just the difference of two cumulative distributions ($G_{Y_A}$ and $G_{Y_B}$), and thus, unlike in the pp-plot and qq-plot mentioned above, it can be defined without the need of the inverse function.
	The random variables $Y_A$ and $Y_B$ have the same $\mathcal{C}_\mathcal{D}, \mathcal{C}_\mathcal{P}$ as the kernel density estimations of the original distributions, and therefore, we can think of the cumulative difference-plot as the difference between the cumulative distribution function of two simpler versions of the original random variables.

%[1] "--- Example 1 ---"
%Computing samples......done!
%Computing Cd...Cp...done!
%
%Cd =  0.9791985  ~ proportion where positive =  0.9518556 in ( 0.6509529 ,  1 ) 
%Cp =  0.7390123  ~ integral + 0.5  =  0.727978 in ( 0.5790443 ,  0.8406489 ) 
%[1] "--- Example 2 ---"
%Computing samples......done!
%Computing Cd...Cp...done!
%
%Cd =  1  ~ proportion where positive =  1 in ( 0.8194584 ,  1 ) 
%Cp =  0.9496949  ~ integral + 0.5  =  0.9542142 in ( 0.8055978 ,  1.053194 ) 
%[1] "--- Example 3 ---"
%Computing samples......done!
%Computing Cd...Cp...done!
%
%Cd =  1  ~ proportion where positive =  0.9779338 in ( 0.09729188 ,  1 ) 
%Cp =  0.6113513  ~ integral + 0.5  =  0.5978178 in ( 0.4507054 ,  0.7417874 ) 
%[1] "--- Example 4 ---"
%Computing samples......done!
%Computing Cd...Cp...done!
%
%Cd =  0.9999986  ~ proportion where positive =  0.3149448 in ( 0 ,  1 ) 
%Cp =  0.50125  ~ integral + 0.5  =  0.4921922 in ( 0.3454341 ,  0.6372334 ) 

\section{Experimentation with the cumulative difference-plot}
\label{section:real_world_case_study}
	
To illustrate the applicability of the proposed methodology, in the following, we re-evaluate the experimentation of a recently published work.
In a recent paper, \cite{santucci_gradient_2020} introduced a gradient-based optimizer for solving problems defined in the space of permutations (from now on \textit{PL-GS}). 
In their experimentation, they compared it with an estimation of distribution algorithm~\cite{larranaga2001estimation} (from now on \textit{PL-EDA}).
These two algorithms were tested in a set of $50$ problem instances of the linear ordering problem~\cite{schiavinotto_linear_2004}.
The performance of each algorithm in each instance was estimated with the median relative deviation from the best-known objective value, with $n = 20$ repetitions.
From now on, we call \textit{score} to the relative deviation from the best-known objective value and note that a low score is better than a high score, as it means that the objective value found is closer to the best-known.

% MEDIAN deviation 1000 reps -> PL-GS=0,0040699,    EDA = 0,004326047 -> Difference of medians -> 0,000256147 > 1e-4

In the work by \cite{santucci_gradient_2020}, when the score of one of the algorithms was at least $10^{-4}$ higher than the other, it was considered that one of the algorithms performed better than the other in that instance.
\cite{santucci_gradient_2020} concluded that both algorithms performed equally in the instance \textit{N\nobreakdash-t70n11xx}, as the median scores were exactly the same for both algorithms in this instance.

In the following, we take a closer look at the performance of \textit{PL-EDA} and \textit{PL-GS} in this problem instance by comparing $n=10^3$ measurements of the score from each algorithm.
We increase the sample size from $n = 20$ to $n = 10^3$ because the difference between the performance of the algorithms is small.
With a sample size of $n=20$, the uncertainty is too high to come to any meaningful conclusion (regardless of the statistical methodology considered).
With this increased sample size, we obtained more accurate estimates of the median scores---$\textit{PL-GS}=0.00407,    \textit{ EDA} = 0.00433$, lower is better---and \textit{PL-GS} obtains a better value by a difference higher than $10^{-4}$.

\subsection{Step 1: Visualization}
\label{sec:case_study_histogram_and_boxplot}

Figure~\ref{fig:histogram_Nt70n11xx} shows the histogram of the scores.
It can be deduced from the figure that neither algorithm clearly produces better scores.
In particular, neither algorithm dominates the other: \textit{PL-EDA} has a longer tail both to the right and to the left.
Also, notice that the score of the algorithms is not normally distributed: \textit{PL-GS} has a bimodal shape, and \textit{PL-EDA} has a very long tail to the right (while the tail to the left is shorter).

\begin{figure}
	\centering
	\includegraphics[width=0.65\linewidth]{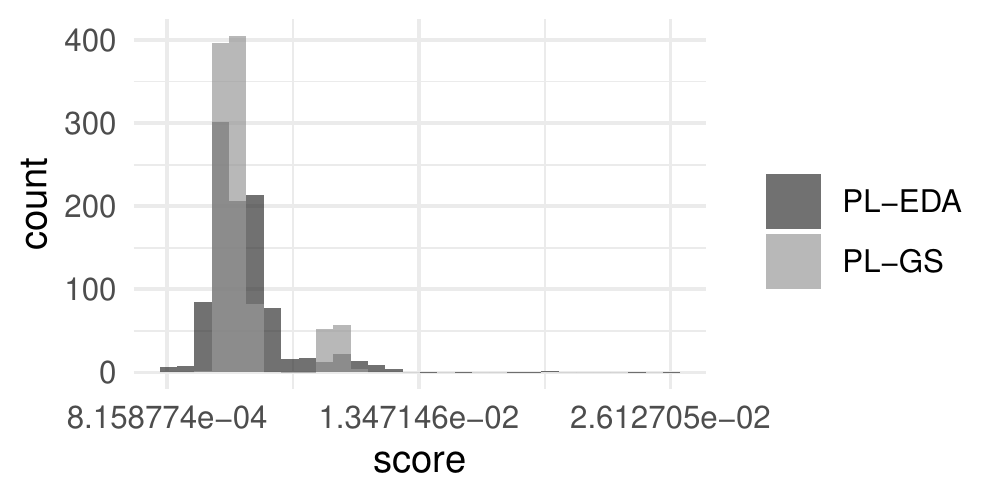}
	\caption{Histogram of the scores obtained in the instance \textit{N-t70n11xx}. Lower is better.}
	\label{fig:histogram_Nt70n11xx}
\end{figure}

Figure~\ref{fig:violin_Nt70n11xx} shows the box-plot and the violin-plot of the data.
Both algorithms have a similar median, but due to the high number of outliers~\cite{carrenoAnalyzingRareEvent2020}, it is difficult to compare the scores of the algorithms with the box-plot.
The same happens with the violin-plot.

\begin{figure}
	\centering
	\includegraphics[width=0.55\linewidth]{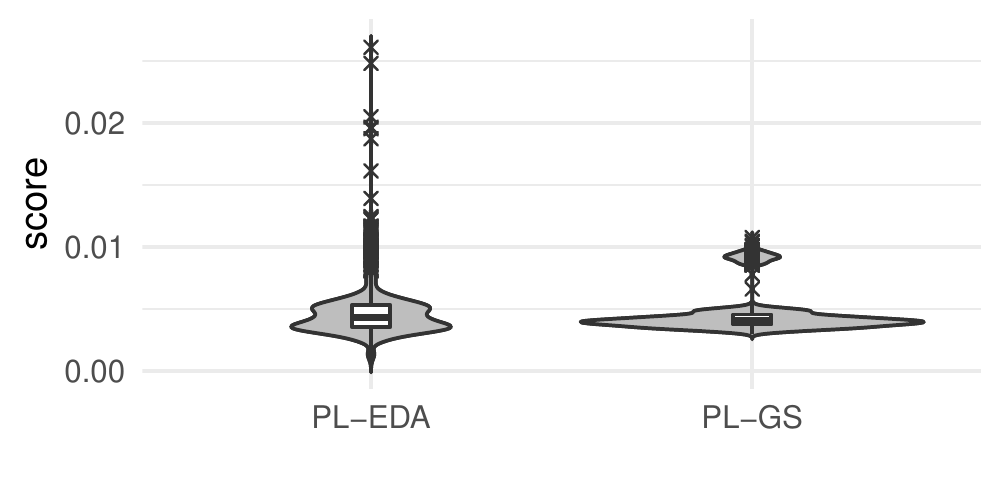}
	\caption{Box-plot and violin-plot of the scores obtained in the instance \textit{N-t70n11xx}. Lower is better.}
	\label{fig:violin_Nt70n11xx}
\end{figure}

\subsection{Step 2: Comparing PL-GS with PL-EDA}

Sometimes, visualization is enough to compare the performance of two algorithms: if one of the algorithms always performs better than the other, there is no need for further analysis.
However, in this case, the three visualization methods considered (histogram, box-plot, and violin-plot) have not been able to summarize the scores obtained with the algorithms in a way that enables an easy comparison.
In the following, we further study the scores of the algorithms with statistical tests, the simplex-plot, and the cumulative difference-plot.

\subsubsection{Mann-Whitney test}

Applying the Mann-Whitney test we obtain a \textit{p}-value of $p = 0.035$, lower than the usually used $0.05$ threshold.
With $\textit{p} < 0.05$, we reject the null hypothesis and accept the alternate hypothesis: the random variable associated with the score of \textit{PL-GS} dominates \textit{PL-EDA}.
Note that neither rejecting the null hypothesis nor a small \textit{p}-value reflect the magnitude of the difference in score of the algorithms.
In addition, as stated when we studied the histogram, we known that it is unlikely that \textit{PL-GS} dominates \textit{PL-EDA}.

\subsubsection{Simplex-plot}

We show the simplex-plot~\cite{benavoli2017time} of the scores in Figure~\ref{fig:simplex_Nt70n11xx}.
Following the criterion by \cite{santucci_gradient_2020}, we considered that two scores are equal when they differ by less than $r = 10^{-4}$.
Unlike in the statistical test, one can deduce the probability that one of the algorithms has a better score than the other from simplex-plot: it is more likely that \textit{PL-GS} takes a lower value than \textit{PL-EDA}.
A closer position in the plot to \textit{PL-EDA} indicates a higher probability of measuring a higher score in \textit{PL-EDA} than in \textit{PL-GS}.
Specifically, from the simplex-plot shown in Figure~\ref{fig:simplex_Nt70n11xx}, we can deduce that given two samples $x_{gs}$ and $x_{eda}$ of the scores of \textit{PL-GS} and \textit{PL-EDA} respectively, 
$$\mathcal{P}(x_{eda} < x_{gs}) < \mathcal{P}(x_{gs} < x_{eda}).$$

However, the difference in these probabilities is small. 
Also, the probability that $\mathcal{P}(x_{gs} = x_{eda})$ is low (no data points near `rope').

\begin{figure}
	\centering
	\includegraphics[width=0.35\linewidth]{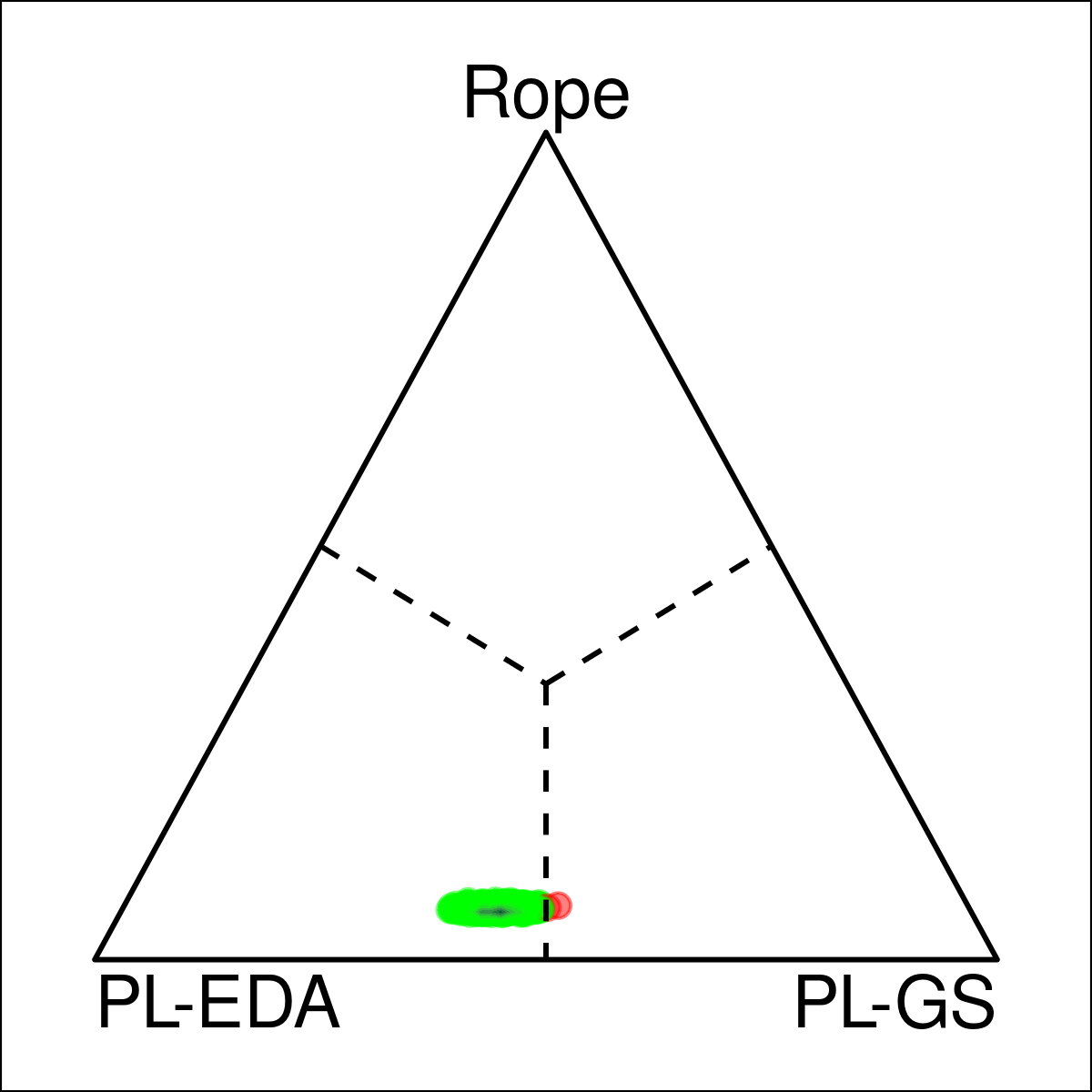}
	\caption{
		Simplex-plot of \textit{PL-GS} and \textit{PL-EDA} in the instance \textit{N-t70n11xx}.
		A closer position in the plot to \textit{PL-EDA} indicates a higher probability of measuring a higher score in \textit{PL-EDA} than in \textit{PL-GS}.
		A low score is preferred to a high score.
	}
	\label{fig:simplex_Nt70n11xx}
\end{figure}

\subsubsection{Cumulative difference-plot}

We show the $95\%$ confidence cumulative difference-plot in Figure~\ref{fig:difference_Nt70n11xx}.
From this plot, we can deduce the following:

\begin{enumerate}
	\item $\mathcal{P}(x_{eda} < x_{gs})$ and  $\mathcal{P}(x_{gs} < x_{eda})$ have similar probabilities, as $\mathcal{C}_\mathcal{P}(\textit{PL-EDA},\textit{PL-GS}) \approx 0.5$. However, The area under $\diff(x)=0$ is a little larger than the area over $\diff(x)=0$, hence $\mathcal{P}(x_{eda} < x_{gs})$ is a little smaller than $\mathcal{P}(x_{gs} < x_{eda})$.
	
	\item Neither algorithm dominates the other one, and what is more, $\mathcal{C}_\mathcal{D}(\textit{PL-EDA},\textit{PL-GS}) \approx 0.5$.
	\item The difference is positive when $x < 0.3$, and therefore, if we only consider the best $30\%$ values of both algorithms, \textit{PL-EDA} dominates \textit{PL-GS}.
	\item The difference is negative when $ x > 0.98$. In this case, we conclude that if we only consider the worst $2\%$ values of \textit{PL-EDA} and \textit{PL-GS}, then \textit{PL-GS} dominates \textit{PL-EDA}.
	\item These ``worst'' $2\%$ values are much less likely than the ``best'' $30\%$ values mentioned in 3), as the estimated probability of these ``best'' and ``worst'' values is $0.3$ and $0.02$ respectively.
	\item The difference is negative at $x = 0.5$ and at $x = 0.75$. This can be interpreted as PL-GS having a better median and a better $75\%$ quantile.
\end{enumerate}

Summarizing the above points, we conclude that the performance of the algorithms is quite similar, and \textit{PL-EDA} takes both better and worse scores than \textit{PL-GS}.
The probability that \textit{PL-EDA} takes these better values is much higher than the probability that it takes worse values.
Therefore, if we are in a setting in which repeating the execution of the algorithms is reasonable, \textit{PL-EDA} is a much better algorithm.
On the other hand, if it is critical to avoid really bad values, then \textit{PL-GS} would be preferred.
With an increased number of samples, it might be possible to better compare the algorithms (it would reduce the uncertainty associated with the size of the confidence band).

\begin{figure}
	\centering
	\includegraphics[width=0.55\linewidth]{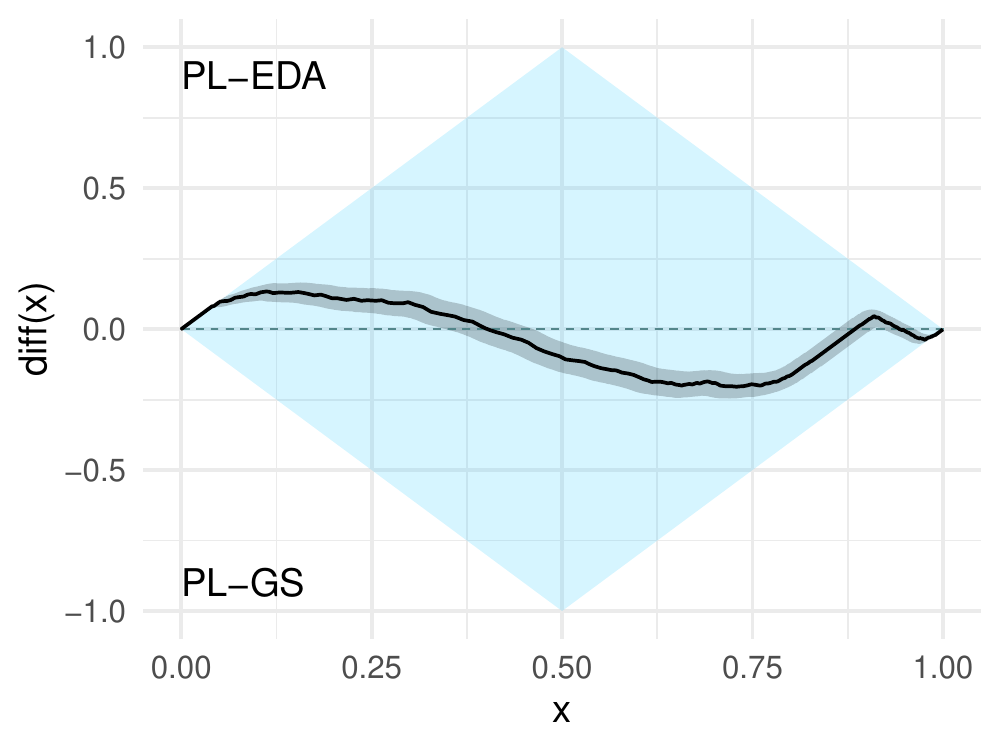}
	\caption{
		The cumulative difference-plot of $95\%$ confidence of the objective values obtained by \textit{PL-EDA} and \textit{PL-GS} in the instance \textit{N-t70n11xx}.
	}
	\label{fig:difference_Nt70n11xx}
\end{figure}

\section{Assumptions and limitations}
\label{section:assumptions_and_limitations}
In the following, we briefly summarize the assumptions that the cumulative difference-plot requires and comment on a few caveats.

\subsection{Assumptions}
Correctly using the proposed cumulative difference-plot requires that the following three assumptions are satisfied.
The first assumption is that all samples of both $X_A$ and $X_B$ are i.i.d, consequently, it should not be used with paired data. 
This is also an assumption made by the Mann-Whitney test.

The second assumption is that the values of the random variables represent a minimization setting: lower values are preferred to higher values.
To apply the proposed method in a maximization setting, it is enough to redefine the objective function by multiplying it by $-1$.

The third assumption is that $X_A$ and $X_B$ are continuous random variables defined in a connected subset of $\mathbb{R}$.
This also implies that the cumulative distribution functions of $X_A$ and $X_B$ are continuous and that their probability density functions are bounded.
Although having a bounded density means that there should never be two identical samples---the probability of observing two independent equal samples is $0$ with a bounded density---, in reality, the proposed cumulative difference-plot can deal with repeated samples. 
To do so, when defining the kernel density estimations of $Y_A$ and $Y_B$ in Section~\ref{section:limitedata}, repeated samples were assigned the same rank.
Then, the size of the uniform distributions was adjusted (with the $\gamma$ function) ensuring that the sum of the estimated densities of $Y_A$ and $Y_B$ remains constant even in the case of repeated observations.

\subsection{Limitations and future work}

Just like with other methods, the number of samples determines in part the stability of the results.
With a small sample size, the confidence band of the cumulative difference-plot will be larger.
There are three reasons why a larger sample size increases the stability of the plot: i) we are doing a kernel density estimation, and a higher sample size~\cite{danicaWhatMinimumNumber2009} implies that the estimation is closer to the real distribution, ii) the bootstrap method also requires several samples to be meaningful~\cite{chernick2011bootstrap,hall2013bootstrap} and iii) the sample size needs to be reasonable with respect to the quantiles being estimated.
For example, it would not make sense to use 10 samples to estimate a $1\%$ quantile.
In all of these cases, however, determining what is a \textit{too small} sample size is a highly debated question, and is beyond the scope of this paper.
To be on the safe side, we recommend using a sample size of at least $n = 100$.
It is worth noting that this was arbitrarily chosen, and a suitable sample size should be chosen depending on the desired conclusions (for example, comparing small and big quantiles requires more data).
With $n=100$ we ensure that the comparison of $1\%$ quantiles in the cumulative difference-plot is meaningful.

The most obvious limitation of the proposed approach is in its applicability: it should only be used in case of doubt between two random variables, and when none of the random variables dominates the other one. 
Otherwise, there are more suitable alternatives such as Bayesian analysis~\cite{benavoli2014bayesian,calvoBayesianPerformanceAnalysis2019}, or directly comparing box-plots.
For instance, if we take $10^3$ samples of $X_A$ and $X_B$ and all samples of $X_A$ are lower than all samples of $X_B$, then there is no need for further statistical comparison, as the results speak for themselves.

The proposed approach assumes $X_A$ and $X_B$ are continuous random variables and that all samples of both $X_A$ and $X_B$ are i.i.d, and consequently, it cannot be used with paired data.
As future work, the proposed methodology could be extended for paired data and ordinal random variables.
Also, the bootstrap method is the slowest part of the \textit{cumulative difference-plot}, especially as the number of samples increases.
To increase the computation speed, this slow part was written in C++ (the rest of the package was written in \textit{R}~\cite{rcoreteamLanguageEnvironmentStatistical2020}).
However, its speed can probably be further improved with a better implementation.

\section{Conclusion}
\label{section:conclusion}

In this paper, we approached the problem of choosing between two random variables in terms of which of them takes lower values.
We proposed eight desirable properties for dominance measures: functions that compare random variables in the context of quantifying the dominance.
Among the measures in the literature, we found out that \textit{the probability that one of the random variables takes lower values than the other} was the one that satisfies the most properties.
However, it fails to satisfy Property~\ref{prop:bounds_with_interpretation}, hence it cannot be used to determine when one of the random variables stochastically dominates the other.
To overcome this limitation, we introduced a new dominance measure: the dominance rate, which quantifies how much higher one of the cumulative distribution function is than the other.

Based on the above, we proposed a cumulative difference-plot that allows two random variables to be compared in terms of which of them takes lower values.
This cumulative difference-plot contains a comparison of the quantiles, in addition to allowing a graphical estimation of the dominance rate and the probability that one of the random variables takes lower values than the others.
It also models the uncertainty associated with the estimate through a confidence band.
Finally, in Section~\ref{section:real_world_case_study} we showed that the proposed methodology is suitable to compare two random variables, especially when they take similar values and other methods fail to give detailed and clear answers.

\bigskip
\begin{center}
	\section*{Supplementary Material}
\end{center}

\subsection*{}

\begin{description}
	\item[RVCompare: ] With this R package, users can compute the $\mathcal{C}_\mathcal{P}$ and $\mathcal{C}_\mathcal{D}$ of two distributions, given their probability density functions.
	Furthermore, it can be used to produce the proposed cumulative difference-plot, given the observed data. (The package can be directly installed from CRAN and is also available in the GitHub repo \url{https://github.com/EtorArza/RVCompare})
	
\end{description}

\begin{description}
	\item[Reproducibility: ] Alongside the paper, we provide the code to generate the figures in this paper and replicate the experimentation. For instructions on how to install the dependencies and replicate the results, refer to the README.md file in the repository. (GitHub repo \url{https://github.com/EtorArza/SupplementaryPaperRVCompare})
\end{description}

\begin{description}
\item[Appendices: ] To keep the length of the paper at a reasonable size, the appendices have been moved to another document. (The appendices are available for download at \url{https://doi.org/10.5281/zenodo.6528669}).

\end{description}

	\begin{center}
		\section*{Acknowledgments}
	\end{center}

	This work was funded in part by the Spanish Ministry of Science, Innovation and Universities through PID2019-106453GA-I00/AEI/10.13039/501100011033 and the BCAM Severo Ochoa excellence accreditation SEV-2017-0718; by the Basque Government through the Research Groups 2019-2021 IT1244-19, ELKARTEK Program (project code KK-2020/00049) and BERC 2018-2021 program.

	\textbf{Disclosure statement}
	
	The authors report that there are no competing interests to declare.

	\FloatBarrier

	\bibliographystyle{chicago}
	\bibliography{main}

	\FloatBarrier

	% All the paper together
	\clearpage
	\section*{Appendices}

\section{A literature review of measures}
\label{appendix:additional_information_on_bad_comparison_functions}

\subsection{$f$-divergences} % https://en.wikipedia.org/wiki/F-divergence
\label{sec:f_divergences_explained}
The \textit{$f$-divergence} is a family of functions that can be used to measure the difference between two random variables.
Given a strictly convex\footnote{A function $f: (0,+\infty) \rightarrow \mathbb{R}$ is strictly convex if $\text{for all } t 	\in [0,1], \text{for all } x_1,x_2 \in (0,+\infty)$, $f(tx_1+ (1-t)x_1) < tf(x_1)+ (1-t)f(x_2)$} function $f:(0,+\infty) \rightarrow \mathbb{R}$ with $f(1) = 0$, and two continuous random variables $X_A$ and $X_B$, the $f$-divergence~\cite{liese_divergences_2006,renyi1961measures} is defined as 

\begin{equation}
D_f(X_A, X_B) = \int_{\mathbb{R}} g_B(x)  f\left(\dfrac{g_A(x)}{g_B(x)}\right)  dx
\end{equation}

where $g_A$ and $g_B$ are the probability density functions of the random variables $X_A$ and $X_B$ respectively.
Since $g_B(x)$ can be $0$, we assume~\cite{polyanskiy2014lecture} that $0 \cdot f(0 / 0) = 0$ and $0\cdot f(a /  0) = \lim_{x \rightarrow 0+} x \cdot f(a / x)$.
Notice that if $g_A$ and $g_B$ are the same probability density functions, then $D_{f}(X_A, X_B) = 0$.

\emph{Kullback–Leibler divergence:}	The Kullback–Leibler divergence~\cite{kullback_information_1951} is a particular case of the $f$-divergence, for $f(x) = x \cdot \ln(x)$.
Given two random variables $X_A$ and $X_B$, $D_{KL}(X_A, X_B)$ can be interpreted~\cite{111523} as the amount of entropy increased by using $g_B$ to model data that follows the probability density function $g_A$.

The Kullback–Leibler divergence is non-negative, and non symmetric $D_{KL}(X_A, X_B) \neq D_{KL}(X_B, X_A)$, and therefore, it is not actually a distance~\cite{goodfellowDeepLearning2016}.
It will not satisfy Property~\eqref{prop:antisimmetry}, as it is not antisymmetric either.
This also makes the interpretation less intuitive.
The Kullback–Leibler divergence is often used to measure the difference between two random variables~\cite{goodfellowDeepLearning2016}, but since $D_{KL}(X_A, X_B) \neq D_{KL}(X_B, X_A)$, it may be better to interpret the Kullback–Leibler divergence as stated above~\cite{111523}.

In Figure~\ref{fig:mathematica_distances}, we show the probability density functions and cumulative distribution functions of four random variables $X_A,X_B,X_C$ and $X_D$.
Looking at their cumulative distributions (Figure~\ref{fig:mathematica_distances_cumulative}), one can clearly see that $X_A \succ X_B$, $X_B \lessgtr X_C$ and $X_B \succ X_D$.
However, as shown in Table~\ref{tab:mathematica_distances}, $D_{KL}(X_B, X_A) = D_{KL}(X_B, X_C) = D_{KL}(X_B, X_D) = 15.4$ and $D_{KL}(X_A, X_B) = D_{KL}(X_C, X_B) = D_{KL}(X_C, X_D) = 6.2$.
This means that, given any two random variables $X_A$ and $X_B$, the Kullback-Leibler is not able to distinguish if $X_A\succ X_B$, $X_B \succ X_A$ or $X_A \lessgtr X_B$.
We can interpret this as the Kullback–Leibler divergence only caring about the difference between two random variables, and not if this difference is related to one of the random variables taking lower values than the other.
Hence, it cannot satisfy Property~\ref{prop:bounds_with_interpretation}, even if we try to transform it to be defined in the $[0,1]$ interval.
We conclude that the Kullback–Leibler divergence is not suitable to gain information regarding which of the random variables takes lower values.

\begin{figure} 
	\centering
	\subfloat[Probability density.\label{fig:mathematica_distances_density}]{%
		\includegraphics[width=0.45\textwidth]{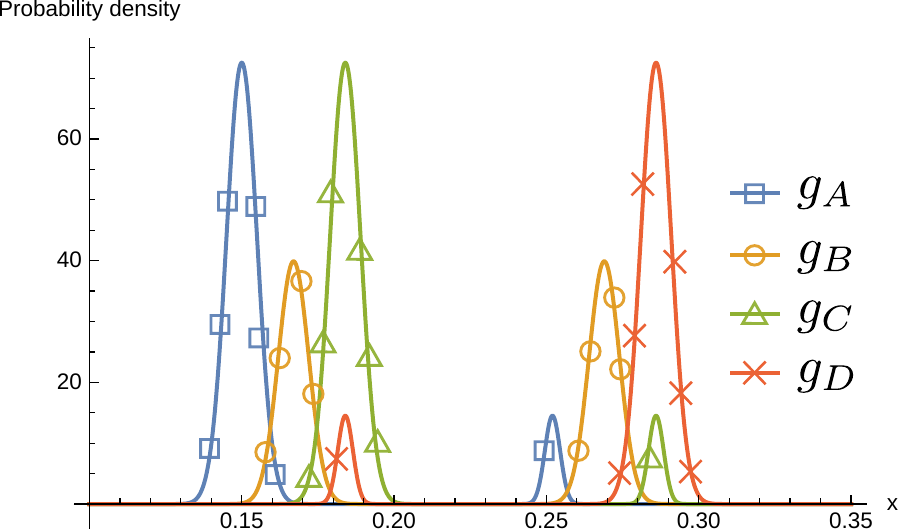}
	}
	\\
	\subfloat[Cumulative distribution function.\label{fig:mathematica_distances_cumulative}]{%
		\includegraphics[width=0.45\textwidth]{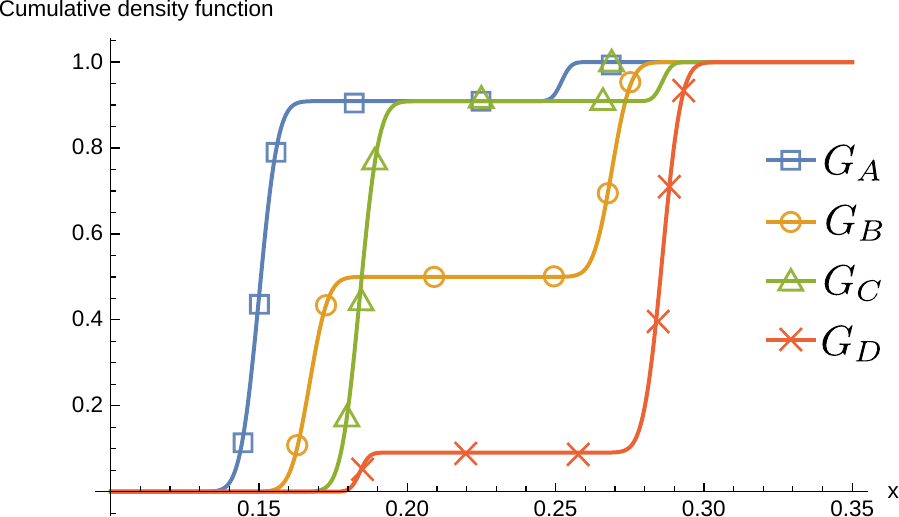}
	}
	\caption
	{
		The probability density function and cumulative distribution of the four random variables.
		The distances between these random variables are listed in Table~\ref{tab:mathematica_distances}.
	}
	\label{fig:mathematica_distances} 
\end{figure}

\begin{figure} 
	\centering
	\subfloat[Probability density.\label{fig:mathematica_distances_density_wasserstein}]{%
		\includegraphics[width=0.45\textwidth]{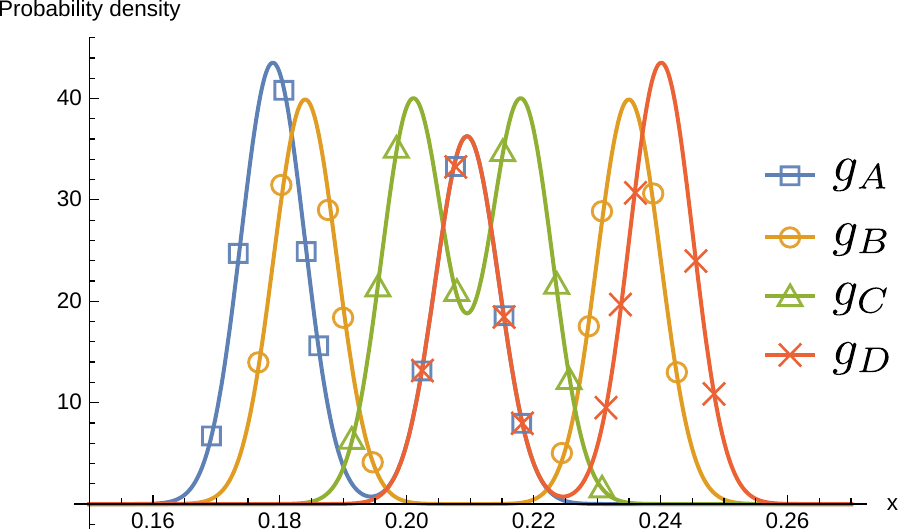}
	}
	\\
	\subfloat[Cumulative distribution function.\label{fig:mathematica_distances_cumulative_wasserstein}]{%
		\includegraphics[width=0.45\textwidth]{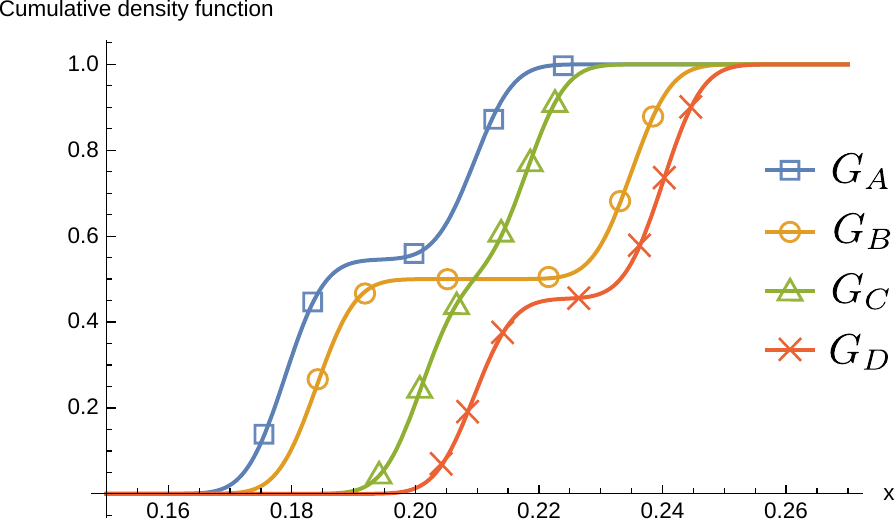}
	}
	\caption
	{
		The probability density function and cumulative distribution of four other random variables.
		The Wasserstein distance between $X_B$ and each of the other random variables is $0.017$.
	}
	\label{fig:mathematica_distances_wasserstein} 
\end{figure}

\begin{table}[]
	\begin{adjustwidth}{-1.2in}{}
		
		\setlength\tabcolsep{3.1pt} % default value: 6pt
		\renewcommand{\arraystretch}{1.2} % General space between rows (1 standard)

		\centering

		\label{tab:mathematica_distances}
		\begin{scriptsize}
			\begin{center}
				
				\begin{tabular}{ll|c|c|c|c|cll|c|c|c|c|c|}
					& \multicolumn{1}{c}{} & \multicolumn{4}{c}{Kullback–Leibler} & & & \multicolumn{1}{c}{} & \multicolumn{4}{c}{Jensen-Shannon}       \\ \cline{3-6} \cline{10-13}
					\multirow{2}{*}{}  & \multirow{2}{*}{} &    \multicolumn{4}{c|}{ RV$_2$}                                                                                 & $\ \ \ \ \ \ $ & \multirow{2}{*}{}  & \multirow{2}{*}{} &    \multicolumn{4}{c|}{ RV$_2$}\\ \cline{3-6}\cline{10-13}
					&  &  $X_A$  &   $X_B$  &  $X_C$  &$X_D$&                                                                                                     & &                                                                          & $X_A$    & $X_B$  & $X_C$ & $X_D$\\ \cline{1-6}\cline{8-13}
					\multicolumn{1}{|c|}{\multirow{4}{*}{\rotatebox{90}{ RV$_1 \ \ $}}} &  $X_A$     &   0.0  &    6.2  &   28.6   &   88.8   & & \multicolumn{1}{|c|}{\multirow{4}{*}{\rotatebox{90}{ RV$_1 \ \ $}}}  &  $X_A$  & 0.0 & 1.2    & 1.4  & 1.4 \\ \cline{2-6}\cline{9-13}
					\multicolumn{1}{|l|}{}                                                    &  $X_B$     &   15.4 &    0.0  &   15.4   &   15.4    & & \multicolumn{1}{|l|}{}                                                    &  $X_B$  & 1.2 & 0.0    & 1.2  & 1.2 \\ \cline{2-6}\cline{9-13}
					\multicolumn{1}{|l|}{}                                                    &  $X_C$     &   29.4 &    6.2  &    0.0   &    2.6      & & \multicolumn{1}{|l|}{}                                                  &  $X_C$  & 1.4 & 1.2    & 0.0  & 0.8 \\ \cline{2-6}\cline{9-13}
					\multicolumn{1}{|l|}{}                                                    &  $X_D$     &   88.8 &    6.2  &    2.6   &    0.0      & & \multicolumn{1}{|l|}{}                                                  &  $X_D$  & 1.4 & 1.2    & 0.8  & 0.0 \\ \cline{1-6}\cline{8-13}
					\\
					& \multicolumn{1}{c}{} & \multicolumn{4}{c}{Total variation} & & & \multicolumn{1}{c}{} & \multicolumn{4}{c}{Hellinger}       \\ \cline{3-6} \cline{10-13}
					\multirow{2}{*}{}  & \multirow{2}{*}{} &    \multicolumn{4}{c|}{ RV$_2$}                                                                                 & $\ \ \ \ \ \ $ & \multirow{2}{*}{}  & \multirow{2}{*}{} &    \multicolumn{4}{c|}{ RV$_2$}\\ \cline{3-6}\cline{10-13}
					&  &  $X_A$  &   $X_B$  &  $X_C$  &$X_D$&                                                                                                     & &                                                                               &$X_A$  & $X_B$    & $X_C$  & $X_D$\\ \cline{1-6}\cline{8-13}
					\multicolumn{1}{|c|}{\multirow{4}{*}{\rotatebox{90}{ RV$_1 \ \ $}}} &  $X_A$     &   0.000&    0.934  &   0.999  &    1.000   & & \multicolumn{1}{|c|}{\multirow{4}{*}{\rotatebox{90}{ RV$_1 \ \ $}}} &  $X_A$  & 0.00   & 1.28    & 1.41  & 1.41 \\ \cline{2-6}\cline{9-13}
					\multicolumn{1}{|l|}{}                                                    &  $X_B$     &   0.934&    0.000  &   0.934  &    0.934    & & \multicolumn{1}{|l|}{}                                                   &  $X_B$  & 1.28   & 0.00    & 1.28  & 1.28 \\ \cline{2-6}\cline{9-13}
					\multicolumn{1}{|l|}{}                                                    &  $X_C$     &   0.999&    0.934  &   0.000  &    0.818     & & \multicolumn{1}{|l|}{}                                                  &  $X_C$  & 1.41   & 1.28    & 0.00  & 0.99 \\ \cline{2-6}\cline{9-13}
					\multicolumn{1}{|l|}{}                                                    &  $X_D$     &   1.000&    0.934  &   0.818  &    0.000     & & \multicolumn{1}{|l|}{}                                                  &  $X_D$  & 1.41   & 1.28    & 0.99  & 0.00 \\ \cline{1-6}\cline{8-13}
					\\
					& \multicolumn{1}{c}{} & \multicolumn{4}{c}{Wasserstein} & & &  \multicolumn{1}{c}{} & \multicolumn{4}{c}{$\mathcal{C}_{\mathcal{P}}$}       \\ \cline{3-6} \cline{10-13}
					\multirow{2}{*}{}  & \multirow{2}{*}{} &    \multicolumn{4}{c|}{ RV$_2$}                                                                                 & $\ \ \ \ \ \ $ & \multirow{2}{*}{}  & \multirow{2}{*}{} &    \multicolumn{4}{c|}{ RV$_2$}\\ \cline{3-6}\cline{10-13}
					&  &  $X_A$  &   $X_B$  &  $X_C$  &$X_D$&                                                                                                     & &                                                                          & $X_A$    & $X_B$  & $X_C$ & $X_D$\\ \cline{1-6}\cline{8-13}
					\multicolumn{1}{|c|}{\multirow{4}{*}{\rotatebox{90}{ RV$_1 \ \ $}}} &  $X_A$     &   0.000 &    0.06  &   0.03   &    0.12   & & \multicolumn{1}{|c|}{\multirow{4}{*}{\rotatebox{90}{ RV$_1 \ \ $}}} &  $X_A$  & 	0.50& 0.95 &0.92& 0.99 \\ \cline{2-6}\cline{9-13}
					\multicolumn{1}{|l|}{}                                                    &  $X_B$     &   0.06 &    0.00  &   0.04   &    0.06   & & \multicolumn{1}{|l|}{}                                                    &  $X_B$  & 0.05& 0.50& 0.54& 0.95 \\ \cline{2-6}\cline{9-13}
					\multicolumn{1}{|l|}{}                                                    &  $X_C$     &   0.03 &    0.04  &   0.000   &    0.083   & & \multicolumn{1}{|l|}{}                                                  &  $X_C$  & 0.08& 0.46& 0.50& 0.91 \\ \cline{2-6}\cline{9-13}
					\multicolumn{1}{|l|}{}                                                    &  $X_D$     &   0.12 &    0.06  &   0.083   &    0.000   & & \multicolumn{1}{|l|}{}                                                  &  $X_D$  & 0.01& 0.05& 0.09& 0.50 \\ \cline{1-6}\cline{8-13}
					\\
					& \multicolumn{6}{c}{$\mathcal{C}_{\mathcal{D}}$}  & &\multicolumn{6}{c}{}       \\ \cline{3-6} 
					\multirow{2}{*}{}  & \multirow{2}{*}{} &    \multicolumn{4}{c|}{ RV$_2$}                                                                                 & $\ \ \ \ \ \ $ & \multirow{2}{*}{}  & \multicolumn{6}{c}{} \\ \cline{3-6}
					&  &  $X_A$  &   $X_B$  &  $X_C$  &$X_D$&                                                                                                     & \multicolumn{6}{c}{} \\ \cline{1-6}
					\multicolumn{1}{|c|}{\multirow{4}{*}{\rotatebox{90}{ RV$_1 \ \ $}}} &  $X_A$     &   	0.50 & 1.00 & 1.00 & 1.00 & & \multicolumn{6}{c}{} \\ \cline{2-6}
					\multicolumn{1}{|l|}{}                                                    &  $X_B$     &  0.00 & 0.50 & 0.59 & 1.00   & &  \multicolumn{6}{c}{} \\ \cline{2-6}
					\multicolumn{1}{|l|}{}                                                    &  $X_C$     &  0.00 & 0.41 & 0.50 & 1.00   & & \multicolumn{6}{c}{} \\ \cline{2-6}
					\multicolumn{1}{|l|}{}                                                    &  $X_D$     &   0.00 & 0.00 & 0.00 & 0.50   & & \multicolumn{6}{c}{} \\ \cline{1-6}
					\\ 
				\end{tabular}
			\end{center}
		\end{scriptsize}

	\end{adjustwidth}
			\caption{
	$\mathcal{C}(\text{RV}_1, \text{RV}_2)$ for the random variables $X_A,X_B,X_C$ and $X_D$ shown in Figure~\ref{fig:mathematica_distances}.
}
\end{table}

\vspace{1em}
\emph{Jensen-Shannon divergence:} The Jensen-Shannon divergence~\cite{polyanskiy2014lecture} is very similar to the Kullback-Leibler divergence, and is another the particular case of the $f$-divergence for $f(x) = x \cdot \ln(\frac{2x}{x+1}) + \ln(\frac{2}{x+1})$.
It is also known as the symmetrized version of the Kullback–Leibler divergence~\cite{polyanskiy2014lecture}, because 

$$ D_{JS}(X_A, X_B) = D_{KL}\left( X_A, X_\mathcal{M} \right)  + D_{KL}\left( X_B, X_\mathcal{M}\right) $$

where the probability density function of $X_\mathcal{M}$ is $g_\mathcal{M}(x) = 0.5 ( g_A(x) +  g_B(x))$. 
Thus, we can interpret this divergence as the sum of the Kullback–Leibler divergences of $g_A$ and $g_B$ with respect to the average probability density function $g_\mathcal{M}$.
The Jensen-Shannon divergence also fails to identify (see Table~\ref{tab:mathematica_distances}) the dominance relationships between $X_B$ and the rest of the random variables in Figure~\ref{fig:mathematica_distances}, thus, it cannot satisfy Property~\ref{prop:bounds_with_interpretation}.
In addition, the Jensen-Shannon divergence also fails to satisfy Properties~\ref{prop:antisimmetry} and \ref{prop:inverse}.
See Table~\ref{tab:functions_which_properties_satisfy} for a detailed list of the properties that each measure satisfies.

\vspace{1em}
\emph{Total variation:}	The total variation~\cite{polyanskiy2014lecture} is also a particular $f$-divergence, for $f(x) = \frac{1}{2}  |x - 1|$.
Unlike the Kullback–Leibler divergence, the total variation is symmetric. 
In fact, it is a properly defined distance~\cite{tsybakov_introduction_2009,polyanskiy2014lecture}.
In addition, it is defined between $0$ and $1$.

Given two random variables $X_A,X_B$, the total variation can also be defined as:

$TV(X_A,X_B) = \sup_{C \subseteq \mathbb{R}} |\mathcal{P}_A(C)  - \mathcal{P}_B(C)|$,

where $\mathcal{P}_A$ and $\mathcal{P}_B$ are the probability distributions\footnote{Given the random variable $X_A$ defined in $\mathbb{R}$, its probability distribution, noted as $\mathcal{P}_A$, is a mapping that,  $\text{for all } U \subseteq \mathbb{R}$ that is measurable, $A(U) = \mathcal{P}(X_A \in U)$~\cite{vapnik_statistical_1998}.} of $X_A$ and $X_B$ respectively.
Since the subset $C$ that takes the supremum is $C=\{x \in \mathbb{R} \ | \ g_A(x) > g_B(x)\}$\cite{devroye_total_2020}, we can interpret the total variation as the ``size'' of the difference in the density functions in all points where $g_A$ is more likely than $g_B$.
Following this intuition, when $TV(X_A,X_B) = 1$, $g_A$ and $g_B$ have disjoint supports~\cite{polyanskiy2014lecture}, and thus $X_A$ and $X_B$ are at their maximum difference with respect to this metric.
On the other hand, when $TV(X_A,X_B) = 0$ the random variables are identical.

The Total-Variance also fails to identify (see Table~\ref{tab:mathematica_distances}) the dominance relationships between $X_B$ and the rest of the random variables in Figure~\ref{fig:mathematica_distances}.

\vspace{1em}
\emph{Hellinger distance and the Bhattacharyya distance}: The Hellinger distance is the square root of the $f$-divergence for $f(x) = (1 - \sqrt{x})^2$~\cite{polyanskiy2014lecture}.
It is related to the Bhattacharya coefficient, since $D_{H}(X_A, X_B) = 2(1 - \BhattCoef(X_A,X_B))$~\cite{130455,polyanskiy2014lecture},
where $\BhattCoef(X_A,X_B)$ is the Bhattacharyya coefficient~\cite{kailath_divergence_1967,bhattacharyya1943measure}.
This coefficient is defined as $\BhattCoef(X_A,X_B) = \int_{\mathbb{R}} \sqrt{g_A(x)g_B(x) }dx$, and has proven useful on signal processing~\cite{kailath_divergence_1967}.
Given two probability density functions $g_A$ and $g_B$, the Bhattacharyya coefficient can be interpreted as the integral of the geometric mean of the probability density functions.
The Bhattacharyya coefficient is also related to the Bhattacharyya distance, as $D_{Bhatt}(X_A,X_B) = -\ln(\BhattCoef(X_A,X_B))$.

The Hellinger distance and the Bhattacharyya distance also fail to identify (see Table~\ref{tab:mathematica_distances}) the dominance relationships between $X_B$ and the rest of the random variables in Figure~\ref{fig:mathematica_distances}.

\subsection{Wasserstein distance} % https://en.wikipedia.org/wiki/F-divergence
The Wasserstein distance is another type of distance between probability random variables.
Given two continuous random variables $X_A,X_B$, the Wasserstein distance (of order 1) is defined as~\cite{schuhmacherComputeWassersteinDistance2021,doi:10.1146/annurev-statistics-030718-104938}

$$D_{W}(X_A,X_B) = \int_{\mathbb{R}} |G_A(x) - G_B(x)| dx$$

In Figure~\ref{fig:mathematica_distances_wasserstein}, we show a different set of four random variables $X_A,X_B,X_C$ and $X_D$.
In this case, it is also clear that $X_A \succ X_B$, $X_B \lessgtr X_C$ and $ X_B \succ X_D$ (Figure~\ref{fig:mathematica_distances_cumulative_wasserstein}), but $D_{W}(X_B,X_A) = D_{W}(X_B,X_C) = D_{W}(X_B,X_D) = 0.017$.
Therefore, in this case, the Wasserstein distance does not give any insights about the dominance between $X_B$ and the rest of the random variables, thus, it cannot satisfy Property~\ref{prop:bounds_with_interpretation} even with a transformation.
It also does not satisfy Properties~\ref{prop:antisimmetry}, \ref{prop:inverse}, \ref{prop:scaling}, \ref{prop:scale_of_portions}, \ref{prop:partial_translation_within_support}.

However, with a small change, the Wasserstein distance can comply with Properties~\ref{prop:antisimmetry} and \ref{prop:inverse}.
This change also improves its correlation with the dominance, even though it still does not comply with Property~\ref{prop:bounds_with_interpretation}.
We remove the absolute value, such that the \textit{signed Wasserstein} distance is defined as

$$D_{SW}(X_A,X_B) = \int_{\mathbb{R}} G_A(x) - G_B(x) dx.$$

For the random variables in Figure~\ref{fig:mathematica_distances_wasserstein}, the signed Wasserstein distance has different values: $D_{SW}(X_B,X_A) = 0.17$, $D_{SW}(X_B,X_C) = 0$ and $D_{SW}(X_B,X_D) = -0.017$. 
Notice that 

${X_A \succ X_B \implies D_{SW}(X_B,X_A) > 0}$ and ${X_B \succ X_A \implies D_{SW}(X_B,X_A) < 0}$, 

but unfortunately, when $X_A \lessgtr X_B$, $D_{SW}(X_B,X_A)$ could be 
positive or negative.
This implies that $D_{SW}(X_B,X_A)$ still can not determine if $X_A\succ X_B$, $X_B \succ X_A$, or $X_A \lessgtr X_B$.

\subsection{Heuristic derivation of the first-order stochastic dominance} 

A measure similar to the Wasserstein distance has been proposed in the literature~\cite{schmid_testing_1996} in the context of comparing random variables.
Specifically, this measure is part of the heuristic derivation of a distribution-free statistical test for first-order stochastic dominance~\cite{schmid_testing_1996}.
Given two random variables $X_A,X_B$, this measure is defined as

$$\mathcal{C}_I(X_A,X_B) = \int_{\mathbb{R}} max(0,G_A(x) - G_B(x)) dG_B(x).$$

Note that the values of I range between $0$ and $0.5$.
When $\mathcal{C}_I(X_A,X_B) = 0.5$, we know that $X_A \succ X_B$.
Unfortunately, when ${\mathcal{C}_I(X_A,X_B) \in (0,0.5)}$, it could be that ${X_A \succ X_B}$ or ${X_A \nsucc X_B}$.
Consequently, $\mathcal{C}_I(X_A,X_B)$ cannot satisfy Property~\ref{prop:bounds_with_interpretation}.

\section{Quantile random variables}

\subsection{Computing the probability density functions of $Y_A$ and $Y_B$}
\label{appendix:estimation_densityY_A_step_by_step}

In Section~\ref{section:quantile_rv} we introduced the quantile random variables $Y_A$ and $Y_B$.
We now describe how to compute the probability density functions of $g_{Y_A}$ and $g_{Y_B}$ step by step, with the pseudocode shown in Algorithm~\ref{algo:y_a_and_y_b}.
We define a function $r$ that returns the position of an observation according to its rank in the sorted list of the observation $A_n \cup B_n$ (lines~1--4).
The ranks go from $0$ (for the smallest observation) to $r_{max}$ (for the largest), where $r_{max}$ is the number of unique observation in $A_n \cup B_n$ minus $1$.
Repeated observations are assigned the same rank, and no ranks are skipped: there is at least a value in $\boldsymbol a \cup \boldsymbol b$ corresponding to each rank from $0$ to $r_{max}$.
For each observation in $\{a_1,...,a_n\}$, a uniform distribution defined in the interval $(\frac{r(a_i) + \gamma(r(a_i) - 1) }{2n}, \frac{r(a_i) + \gamma(r(a_i))}{2n})$ is added to the mixture (lines~10--19), where $\gamma(k)$ (lines~7--9) counts the number of ranks in $A_n \cup B_n$ that are lower than or equal to $k$ (since the lowest rank is $0$, $\gamma(-1) = 0$) .
The kernel density estimation for $Y_B$ is defined similarly, but with the observations $\{b_1,...,b_n\}$ instead.

\begin{algorithm}

	\DontPrintSemicolon % Some LaTeX compilers require you to use
	\caption{Kernel density estimation of $Y_A$ and $Y_B$}
	\label{algo:y_a_and_y_b}
	\KwIn{ \ \\ 
		\noindent \underline{$A_n = \{a_1,...,a_n\}$}: The $n$ observed samples of $X_A$. \\
		\noindent \underline{$B_n =\{b_1,...,b_n\}$}: The $n$ observed samples of $X_B$. \\ \ \\
	}
	\KwOut{ \ \\ 
		\noindent \underline{$g_{Y_A}$}: The probability density of $Y_A$. \\
		\noindent \underline{$g_{Y_B}$}: The probability density of $Y_B$. \\ \ \\
	}
	
	\tcc{\footnotesize	 
		Compute the ranks of $A_n \cup B_n$.
		The lowest value has rank 0. 
		Assign the same rank to ties without skipping any rank.
	}
	\For{i = 1,...,n}{
		$r(a_i) \gets $  rank of $a_i$ in  $A_n \cup B_n$\; 			
		$r(b_i) \gets $  rank of $b_i$ in  $A_n \cup B_n$\; 				
	}
	$R$ $\gets \{r(a_1),...,r(a_n),r(b_1),...,r(b_n)\}$ \;
	$r_{max} \gets \max(R)$\;
	\For{k = -1,0,1,...,$r_{max}$}{
		$\gamma(k) \gets$ the number of items in $R$  lower than or equal to $k$   \;	
	}
	
	\tcc{\footnotesize	 
		The probability density function of $g_{Y_A}$ is represented as a mixture of $n$ uniform distributions.
		$g_{Y_A}[s]$ is the probability density of $Y_A$ in the interval $[\frac{s}{2n}, \frac{s+1}{2n})$.
	}
	$g_{Y_A} \gets$ array of zeros of length $2n$\;
	$g_{Y_B} \gets$ array of zeros of length $2n$\;
	\For{$x_i$ = $a_1,...,a_n,b_1,...,b_n$}{
		$A_{mult} \gets $ number of times that $x_i$ is in $A_n$\;
		$B_{mult} \gets $ number of times that $x_i$ is in $B_n$\;
		\For{mult = 1,...,$(A_{mult} + B_{mult})$}{	
			$g_{Y_A}[\gamma(r(a_i) - 1) + mult -1] \gets   (n \cdot A_{mult})^{-1}$\;
			$g_{Y_B}[\gamma(r(b_i) - 1) + mult -1] \gets  (n \cdot B_{mult})^{-1}$\;
		}
		
	}
	\textbf{return} $g_{Y_A}, g_{Y_B}$ \;

\end{algorithm}

\subsection{The quantile random variables have the same $\mathcal{C}_\mathcal{P}$ and $\mathcal{C}_\mathcal{D}$ as the kernel density estimates of $X_A$ and $X_B$.}
\label{appendix:Y_A_mantains_same_mathcalC}

In Section~\ref{section:quantile_rv}, we claimed that when a ``small enough'' uniform~\cite{scikit-learndevelopersDensityEstimation2021} kernel is used in the kernel density estimations of $X_A$ and $X_B$, these estimations will have the same $\mathcal{C}_\mathcal{P}$ and $\mathcal{C}_\mathcal{D}$ as the quantile random variables $Y_A$ and $Y_B$. 
Specifically, the size of the uniform kernels needs to be smaller than ${\min\limits_{ i,j\in\{1...n\} |  a_i \neq b_j}2|a_i - b_j|}$, where $A_n = \{a_1,...,a_n\}$ and ${B_n = \{b_1,...,b_n\}}$ are the $n$ observed samples of $X_A$ and $X_B$ respectively.
As a result, the $\mathcal{C}_\mathcal{P}$ and $\mathcal{C}_\mathcal{D}$ of the kernel density estimations will not change when the size of the kernels is reduced below its initial size.
This can be deduced from Property~\ref{prop:partial_translation_within_support} in Section~\ref{section:properties_of_comparison_comparison_functions}, which both $\mathcal{C}_\mathcal{P}$ and $\mathcal{C}_\mathcal{D}$ satisfy.
%	the empirical distribution~\cite{10.2307/2958458} can be considered as the ``smallest possible uniform kernel''.

The quantile random variables $Y_A$ and $Y_B$ can also be obtained by applying a sequence of transformations to the kernel density estimations (with small uniform kernels) of $X_A$ and $X_B$.
Three consecutive transformations are required, none of which modify the $\mathcal{C}_\mathcal{D}$ and $\mathcal{C}_\mathcal{P}$ due to Property~\ref{prop:partial_translation_within_support}.
The first transformation involves further \textit{reducing} the size of the kernels to $1 /(4n)$.
Secondly, each kernel $k$ is moved into the position $r(k)/(2n) + (4n)^{-1}$, where $r(k)$ is the rank of the sample in $k$ in $A_n \cup B_n$.
In the case of ties, $r$ assigns the same rank to all kernels and this same rank is the average of the previous and the next rank.
Since each of the possible positions are at distance $1/(2n)$ from each other, this transformation will not change the $\mathcal{C}_\mathcal{D}$ and $\mathcal{C}_\mathcal{P}$.
Finally, the length of the kernels is increased to ${\textit{mult} / (4n)}$, where \textit{mult} is the number of times that the sample defining the kernel is repeated in $A_n \cup B_n$.
Note that this increase in the length will in no case cause an overlap of kernels.

\section{ $\mathcal{C}_\mathcal{P}$ and $\mathcal{C}_\mathcal{D}$ in the cumulative difference-plot}
\label{appendix:proof_difference_graph}

In this section, we mathematically prove and experimentally verify that the cumulative difference-plot can be used to deduce $\mathcal{C}_\mathcal{D}$ and $\mathcal{C}_\mathcal{P}$. 
First, we describe which estimators are used when these dominance measures are visually estimated from the cumulative difference-plot.
Then, we show that these estimators converge to $\mathcal{C}_\mathcal{P}$ and $\mathcal{C}_\mathcal{D}$ as the number of samples increases.

\subsection{Estimating $\mathcal{C}_\mathcal{P}$ and $\mathcal{C}_\mathcal{D}$ from the cumulative difference-plot}

\begin{mydef} (observations of random variables)\\
	Let $X_A$ be a continuous random variable.
	We define \textit{$n$ observations of $X_A$} as the realizations of the i.i.d random variables  $\{X_A^i\}_{i=1}^{n}$ that are distributed as $X_A$, denoted as $A_n =  \{a_i\}_{i=1}^{n}$.
\end{mydef}

%$\mathcal{C}_\mathcal{P}	
\begin{mydef}(estimation of $\mathcal{C}_\mathcal{P}$) \\
	\label{def:estimation_of_CP}
	Let $X_A$ and $X_B$ be two continuous random variables and $A_n$ and $B_n$ their $n$ observations respectively.
	We define the estimation of the probability that $X_A$ < $X_B$ as 
	$$\widetilde{\mathcal{C}_\mathcal{P}}(A_n, B_n) =  \sum_{i,k=1...n} \dfrac{\sign(b_k -a_i)}{2n^2} + \frac{1}{2}.$$
\end{mydef}

\begin{mydef}(estimation of $\mathcal{C}_\mathcal{D}$) \\
	\label{def:estimation_of_CD}
	Let $X_A$ and $X_B$ be two continuous random variables and $A_n$ and $B_n$ their $n$ observations respectively.
	Let $\{c_j\}_{j=1}^{2n}$ the sorted list of all the observations of $A_n$ and $B_n$ where $c_1$ is the smallest observation and $c_{2n}$ the largest.
	Let $\{c_d\}_{d=1}^{d_{max}}$ be the sorted list of unique values in $\{c_j\}_{j=1}^{2n}$.
	We define the estimation of the dominance rate as

	$$\widetilde{\mathcal{C}_\mathcal{D}}(A_n, B_n) = \dfrac{\sum_{j=1}^{2n} \dfrac{\psi(c_j)}{2n} + 1}{2}\cdot k_c^{-1}$$
	
	$k_c = \frac{\sum_{j=1}^{2n} \mathcal{I}[\psi(c_j) \neq 0]}{2n}$ is the normalization constant and $\psi_j$ is defined as

	\begin{small}
		$$ \psi(c_d) =
		\begin{cases}
		\multirow{2}{1em}{$0$} & \text{ if } \hat{G}_A(c_{d-1}) = \hat{G}_B(c_{d-1})                                                             \\ \vspace{0.75em} &  \text{ and } \hat{G}_A(c_{d}) = \hat{G}_B(c_{d})\\ 
		\multirow{2}{1em}{$1$} & \text{ if }  \hat{G}_A(c_{d-1}) \geq \hat{G}_B(c_{d-1})                                                         \\ \vspace{0.75em} &    \text{ and } \hat{G}_A(c_{d})   >  \hat{G}_B(c_{d})\\ 
		\multirow{2}{1em}{$1$}  & \text{ if }  \hat{G}_A(c_{d-1}) >    \hat{G}_B(c_{d-1})                                                        \\ \vspace{0.75em} &    \text{ and } \hat{G}_A(c_{d}) \geq \hat{G}_B(c_{d})\\ 
		\multirow{2}{1em}{$-1$}  & \text{ if }  \hat{G}_B(c_{d-1}) \geq \hat{G}_A(c_{d-1})                                                       \\ \vspace{0.75em} &    \text{ and } \hat{G}_B(c_{d})   >  \hat{G}_A(c_{d})\\ 
		\multirow{2}{1em}{$-1$}  & \text{ if }  \hat{G}_B(c_{d-1}) >    \hat{G}_A(c_{d-1})                                                       \\ \vspace{0.75em} &    \text{ and } \hat{G}_B(c_{d}) \geq \hat{G}_A(c_{d})\\
		\multirow{2}{7em}{$1 - 2\gamma(c_d)$} & \text{ if }  \hat{G}_B(c_{d-1}) >    \hat{G}_A(c_{d-1})      \\ \vspace{0.75em} &    \text{ and } \hat{G}_A(c_{d})   >  \hat{G}_B(c_{d})\\
		\multirow{2}{7em}{$2\gamma(c_d) - 1$} & \text{ if }  \hat{G}_A(c_{d-1}) >    \hat{G}_B(c_{d-1}) \\ \vspace{0.75em} &    \text{ and } \hat{G}_B(c_{d})   >  \hat{G}_A(c_{d})\\
		\end{cases} 
		$$
	\end{small}
	
	with $\gamma(c_d) = \dfrac{\hat{G}_B(c_{d-1}) - \hat{G}_A(c_{d-1})}{[B_n = c_d] - [A_n = c_d]}$.
	Note that $[A_n = c_d]$ counts the number of items in $A_n$ equal to $c_d$ and $\hat{G}_A$ is the empirical distribution~\cite{10.2307/2958458} estimated from $A_n$.
	To improve the readability, we abuse the notation and assume that ${\hat{G}_A(c_{0}) = 0}$.

\end{mydef}

We now show that these estimates can be directly computed from the cumulative difference plot.
First, we show that the estimation of $\mathcal{C}_\mathcal{P}$ from the cumulative difference-plot is equivalent to the estimation in Definition~\ref{def:estimation_of_CP}.
As mentioned in Section~\ref{section:the_difference_plot}, the $\mathcal{C}_\mathcal{P}$ estimated from the cumulative difference-plot is $0.5 + \int_{0}^{1}  \diff(x) dx$ where $\diff$ is the difference function introduced in Equation~\eqref{eq:difference_function_for_cum_plot}.
Specifically, the difference function was defined as ${\diff(x) = G_{Y_A}(x) - G_{Y_B}(x)}$.

\begin{mylemma}
	\label{lemma:lema_for_C_P}
	Let $X_A$ and $X_B$ be two continuous random variables and $A_n$ and $B_n$ their $n$ observations respectively.
	Then, 
	
	$$\int_{0}^{1}  \diff(x) dx  = \sum_{j=1}^{2n}  \dfrac{G_{Y_A}(\frac{j}{2n}) - G_{Y_B}(\frac{j}{2n})}{2n}$$
\end{mylemma}
\begin{proof}
	
	Considering that the density functions of $Y_A$ and $Y_B$ are constant in each interval $[\frac{j}{2n}, \frac{j+1}{2n})$ for $j=0,...,(2n-1)$, we get that

	$\int_{\frac{j}{2n}}^{\frac{j+1}{2n}}  \diff(x) dx = \dfrac{\diff(\frac{j}{2n}) + \diff(\frac{j+1}{2n})}{4n} = $
	
	$ \dfrac{G_{Y_A}(\frac{j}{2n}) - G_{Y_B}(\frac{j}{2n}) + G_{Y_A}(\frac{j+1}{2n}) - G_{Y_B}(\frac{j+1}{2n})}{4n}$
	
	Taking into account that $ G_{Y_A}(0) = G_{Y_B}(0) = 0$ and $ G_{Y_A}(1) = G_{Y_B}(1) = 1$,
	
	$ \int_{0}^{1}  \diff(x) dx = \sum_{j=0}^{2n-1} \int_{\frac{j}{2n}}^{\frac{j+1}{2n}}  \diff(x) dx = $
	
	$ \dfrac{G_{Y_A}(\frac{0}{2n}) - G_{Y_B}(\frac{0}{2n}) +  G_{Y_A}(\frac{2n}{2n}) - G_{Y_B}(\frac{2n}{2n})}{4n} + $
	
	$\sum_{j=1}^{2n-1}  \dfrac{2 \cdot G_{Y_A}(\frac{j}{2n}) - 2 \cdot G_{Y_B}(\frac{j}{2n})}{4n} =$
	
	$\sum_{j=1}^{2n-1}  \dfrac{G_{Y_A}(\frac{j}{2n}) - G_{Y_B}(\frac{j}{2n})}{2n}$
	
	Finally, since  $G_{Y_A}(1) =  G_{Y_B}(1) = 1$, we have that 
	
	$\sum_{j=1}^{2n-1}  \dfrac{G_{Y_A}(\frac{j}{2n}) - G_{Y_B}(\frac{j}{2n})}{2n} = $
	
	$\sum_{j=1}^{2n}  \dfrac{G_{Y_A}(\frac{j}{2n}) - G_{Y_B}(\frac{j}{2n})}{2n}$

\end{proof}

\onecolumn
\begin{myproposition} ($\mathcal{C}_\mathcal{P}$ estimated from the cumulative difference-plot) \\s
	Let $X_A$ and $X_B$ be two random variables and $A_n$ and $B_n$ their $n$ observations respectively.
	Let $\diff$ be the difference function obtained from the samples $A_n$ and $B_n$ as defined in Equation~\eqref{eq:difference_function_for_cum_plot}.
	Then, 
	
	$$ \widetilde{\mathcal{C}_\mathcal{D}}(A_n,B_n) = \int_{0}^{1}  \diff(x) dx + \frac{1}{2}$$
\end{myproposition}
\begin{proof}
	
	Given the observations $A_n$ and $B_n$, we need to prove that 
	
	$ \sum_{i,k=1...n} \dfrac{\sign(b_k -a_i)}{2n^2} + \frac{1}{2} = \int_{0}^{1}  \diff(x) dx + \frac{1}{2}$ 
	
	With Lemma~\ref{lemma:lema_for_C_P}, it is enough to prove that
	
	$\sum_{i,k=1...n} \dfrac{\sign(b_k -a_i)}{2n^2} = \sum_{j=1}^{2n}  \dfrac{G_{Y_A}(\frac{j}{2n}) - G_{Y_B}(\frac{j}{2n})}{2n}$

	Let $C_{2n} = \{c_j\}_{j=1}^{2n}$ be the list of all the sorted observations of $A_n$ and $B_n$ where $c_1$ is the smallest observation and $c_{2n}$ the largest.
	Then, we have that 
	
	$$G_{Y_A}(\frac{j}{2n}) = \dfrac{[A_n < c_j] + \dfrac{[A_n = c_j][k \leq j | c_k = c_j]}{[C_{2n} = c_j]}  }{n} \text{ and } $$  		 
	$$G_{Y_B}(\frac{j}{2n}) = \dfrac{[B_n < c_j] + \dfrac{[B_n = c_j][k \leq j | c_k = c_j]}{[C_{2n} = c_j]}  }{n}  $$

	where $[A_n < c_j]$ counts the number of items in $A_n$ lower than $c_j$, and $[k \leq j | c_k = c_j]$ counts the number of items in $C_{2n}$ equal to $c_j$ but with a lower or equal position in $C_{2n}$.
	Therefore, we have that 		
	
	$$\sum_{j=1}^{2n}  \dfrac{G_{Y_A}(\frac{j}{2n}) - G_{Y_B}(\frac{j}{2n})}{2n} = $$
	
	$$\sum_{j=1}^{2n} \dfrac{ [A_n < c_j] + \dfrac{[A_n = c_j][k \leq j | c_k = c_j]}{[C_{2n} = c_j]} -  [B_n < c_j] - \dfrac{[B_n = c_j][k \leq j | c_k = c_j]}{[C_{2n} = c_j]}}{2n^2}$$
	
	\begin{equation}
	\label{eq:replace_G_A_with_counting}
	\sum_{j=1}^{2n} \dfrac{ [A_n < c_j] - [B_n < c_j]  + \dfrac{([A_n = c_j] - [B_n = c_j])[k \leq j | c_k = c_j]}{[C_{2n} = c_j]} }{2n^2}
	\end{equation}
	
	Now we group the terms in Equation~\eqref{eq:replace_G_A_with_counting} into $d_{max}$ groups such that each group contains all the terms with the same $c_j$, and each group $d$ contains $[C_{2n} = c_d]$ terms, with $c_j = c_d$.

	$$\sum_{j=1}^{2n} \dfrac{ [A_n < c_j] - [B_n < c_j]}{2n^2} + \sum_{d=1}^{d_{max}} \sum_{c_j} \dfrac{\dfrac{([A_n = c_j] - [B_n = c_j])[k \leq j | c_k = c_j]}{[C_{2n} = c_j]} }{2n^2} = $$
	$$\sum_{j=1}^{2n} \dfrac{ [A_n < c_j] - [B_n < c_j]}{2n^2} + \sum_{d=1}^{d_{max}}  \dfrac{  \dfrac{([A_n = c_d] - [B_n = c_d])(([C_{2n} = c_d]+ 1) \cdot [C_{2n} = c_d] / 2 )}{[C_{2n} = c_d]} }{2n^2} = $$
	$$\sum_{j=1}^{2n} \dfrac{ [A_n < c_j] - [B_n < c_j]}{2n^2} + \sum_{d=1}^{d_{max}}  \dfrac{  ([A_n = c_d] - [B_n = c_d])([C_{2n} = c_d]+ 1) / 2  }{2n^2} = $$
	$$\sum_{j=1}^{2n} \dfrac{ [A_n < c_j] - [B_n < c_j]}{2n^2} + \sum_{d=1}^{d_{max}}  \dfrac{  ([A_n = c_d] - [B_n = c_d])([C_{2n} = c_d]) / 2 +  ([A_n = c_d] - [B_n = c_d]) / 2  }{2n^2} = $$
	$$\sum_{j=1}^{2n} \dfrac{ [A_n < c_j] - [B_n < c_j]}{2n^2} +   \sum_{j=1}^{2n} \dfrac{([A_n = c_j] - [B_n = c_j]) / 2}{2n^2} + \underbrace{\sum_{d=1}^{d_{max}} \dfrac{([A_n = c_d] - [B_n = c_d]) / 2  }{2n^2}}_{ = 0} =$$

	$$\underbrace{\sum_{j=1}^{2n} \dfrac{ [A_n < c_j] - [B_n < c_j]}{2n^2}}_{\text{first sum}}  + \underbrace{\sum_{j=1}^{2n} \dfrac{([A_n = c_j] - [B_n = c_j]) / 2}{2n^2}  }_{\text{second sum}} $$

	Focusing on the first sum, we have that

	$$\sum_{j=1}^{2n}  \dfrac{[A_n < c_j] - [B_n < c_j]  }{2n^2} = $$
	$$  \dfrac{ \sum_{j=1}^{2n} [A_n < c_j] - \sum_{j=1}^{2n} [B_n < c_j]  }{2n^2} =  $$
	$$ \dfrac{ \sum_{j=1}^{2n}  \sum_{i=1}^{n} [\{a_i\} < c_j] - \sum_{j=1}^{2n} \sum_{i=1}^{n} [\{b_i\} < c_j]  }{2n^2} = $$
	$$ \dfrac{ \sum_{k=1}^{n}  \sum_{i=1}^{n} [\{a_i\} < a_k] + \sum_{k=1}^{n} \sum_{i=1}^{n} [\{a_i\} < b_k]  }{2n^2} - $$
	$$ \dfrac{ \sum_{k=1}^{n}  \sum_{i=1}^{n} [\{b_i\} < a_k] + \sum_{k=1}^{n} \sum_{i=1}^{n} [\{b_i\} < b_k]  }{2n^2} = $$
	$$ \dfrac{ \sum_{k=1}^{n}  \sum_{i=1}^{n} [\{a_i\} < a_k] +  [\{a_i\} < b_k] - [\{b_i\} < a_k] - [\{b_i\} < b_k]}{2n^2} =  $$
	$$ \dfrac{ \sum_{k=1}^{n}  \sum_{i=1}^{n} [\{a_i\} < b_k] - [\{b_i\} < a_k] + [\{a_i\} < a_k] - [\{b_i\} < b_k]}{2n^2} =  $$
	$$ \dfrac{ \sum_{k=1}^{n}  \sum_{i=1}^{n}  \sign(b_k - a_i) + [\{a_i\} < a_k] - [\{b_i\} < b_k] }{2n^2} =  $$
	$$ \dfrac{\sum_{k=1}^{n}  \sum_{i=1}^{n} \sign(b_k - a_i)}{2n^2} + \dfrac{  \sum_{k=1}^{n}  [A_n < a_k] - [B_n < b_k]}{2n^2}    $$

	From the second sum, we obtain		
	
	$$\sum_{j=1}^{2n} \dfrac{([A_n = c_j] - [B_n = c_j]) / 2  }{2n^2} =  \sum_{k=1}^{n} \dfrac{([A_n = {a_k}] - [B_n = {a_k}] + [A_n = {b_k}] - [B_n = {b_k}]) / 2  }{2n^2}  $$

	Combining these summations,  
	
	\begin{small}
	
	$$\sum_{j=1}^{2n} \dfrac{ [A_n < c_j] - [B_n < c_j]}{2n^2}  + \sum_{j=1}^{2n} \dfrac{([A_n = c_j] - [B_n = c_j]) / 2  }{2n^2} =$$

	$$ \dfrac{\sum_{k=1}^{n}  \sum_{i=1}^{n} \sign(b_k - a_i)}{2n^2} + $$
	$$ \dfrac{ \sum_{k=1}^{n}  [A_n < a_k] - [B_n < b_k]}{2n^2}  + \dfrac{\sum_{k=1}^{n}  ([A_n = {a_k}] - [B_n = {a_k}] + [A_n = {b_k}] - [B_n = {b_k}]) / 2  }{2n^2}  = $$

	$$ \dfrac{\sum_{k=1}^{n}  \sum_{i=1}^{n} \sign(b_k - a_i)}{2n^2} + $$
	$$ \dfrac{ \sum_{k=1}^{n}  [A_n \leq a_k] - [B_n \leq b_k]}{2n^2} +   \dfrac{\sum_{k=1}^{n}  (-[A_n = {a_k}] - [B_n = {a_k}] + [A_n = {b_k}] + [B_n = {b_k}]) / 2  }{2n^2}  = $$

	$$  \dfrac{\sum_{k=1}^{n}  \sum_{i=1}^{n} \sign(b_k - a_i)}{2n^2} + \dfrac{ \sum_{k=1}^{n}  [A_n \leq a_k] - [B_n \leq b_k]}{2n^2} +   \dfrac{\sum_{k=1}^{n}  (-[C_{2n} = {a_k}]  + [C_{2n} = {b_k}])  }{4n^2}  = $$

	$$  \dfrac{\sum_{k=1}^{n}  \sum_{i=1}^{n} \sign(b_k - a_i)}{2n^2} + \dfrac{   n(n+1)/2 + \sum_{d=1}^{d_{max}} \frac{[A_{n} = c_d]^2 - [A_{n} = c_d]}{2}}{2n^2} - $$ 
	$$\dfrac{n(n+1)/2 + \sum_{d=1}^{d_{max}} \frac{[B_{n} = c_d]^2 - [B_{n} = c_d]}{2}}{2n^2} + \dfrac{\sum_{k=1}^{n}  (-[C_{2n} = {a_k}]  + [C_{2n} = {b_k}])}{4n^2}  =   $$
	
	$$  \dfrac{\sum_{k=1}^{n}  \sum_{i=1}^{n} \sign(b_k - a_i)}{2n^2} + \dfrac{\sum_{d=1}^{d_{max}} \frac{[A_{n} = c_d]^2 - [A_{n} = c_d]}{2} - \sum_{d=1}^{d_{max}} \frac{[B_{n} = c_d]^2 - [B_{n} = c_d]}{2}}{2n^2} +$$
	$$ \dfrac{\sum_{k=1}^{n}  (-[C_{2n} = {a_k}]  + [C_{2n} = {b_k}])}{4n^2}  =   $$
	
	considering that $ \sum_{d=1}^{d_{max}} \frac{[B_{n} = c_d] - [A_{n} = c_d]}{2} = 0,$ we simplify the previous equation to

	$$  \dfrac{\sum_{k=1}^{n}  \sum_{i=1}^{n} \sign(b_k - a_i)}{2n^2} + \dfrac{\sum_{d=1}^{d_{max}} \frac{[A_{n} = c_d]^2 - [B_{n} = c_d]^2}{2}}{2n^2} +  \dfrac{\sum_{k=1}^{n}  (-[C_{2n} = {a_k}]  + [C_{2n} = {b_k}])  }{4n^2}  =   $$

	$$  \dfrac{\sum_{k=1}^{n}  \sum_{i=1}^{n} \sign(b_k - a_i)}{2n^2}  + \dfrac {\sum_{d=1}^{d_{max}} [A_{n} = c_d]^2 -  [B_{n} = c_d]^2  }{4n^2} + \dfrac{\sum_{k=1}^{n}  (-[C_{2n} = {a_k}]  + [C_{2n} = {b_k}])}{4n^2} =   $$
	$$  \dfrac{\sum_{k=1}^{n}  \sum_{i=1}^{n} \sign(b_k - a_i)}{2n^2}  + \dfrac {\sum_{d=1}^{d_{max}} [A_{n} = c_d]^2 -  [B_{n} = c_d]^2  }{4n^2} +$$
	$$ \dfrac{\sum_{d=1}^{d_{max}}  (-[C_{2n} = c_d][A_{n} = c_d]  + [C_{2n} = c_d][B_{n} = c_d])}{4n^2} =   $$
	$$  \dfrac{\sum_{k=1}^{n}  \sum_{i=1}^{n} \sign(b_k - a_i)}{2n^2}  + \dfrac {\sum_{d=1}^{d_{max}} [A_{n} = c_d]^2 -   [B_{n} = c_d]^2  }{4n^2}  + \underbrace{\dfrac{\sum_{d=1}^{d_{max}}  [C_{2n} = c_d](  [B_{n} = c_d]  -[A_{n} = c_d]) }{4n^2}}_{\text{third sum}} =   $$

	We expand the third sum,
	
	$$\dfrac{\sum_{d=1}^{d_{max}}  [C_{2n} = c_d](  [B_{n} = c_d]  -[A_{n} = c_d])}{4n^2} = $$
	
	$$\dfrac{\sum_{d=1}^{d_{max}}  ([B_{n} = c_d] + [A_{n} = c_d])(  [B_{n} = c_d]  -[A_{n} = c_d])}{4n^2} = \dfrac{\sum_{d=1}^{d_{max}}  ([B_{n} = c_d]^2 -[A_{n} = c_d]^2) }{4n^2} $$
	\end{small}

	Finally,
	
	$$  \dfrac{\sum_{k=1}^{n}  \sum_{i=1}^{n} \sign(b_k - a_i)}{2n^2}  + \dfrac {\sum_{d=1}^{d_{max}} [A_{n} = c_d]^2 -   [B_{n} = c_d]^2  }{4n^2}  + \dfrac{\sum_{d=1}^{d_{max}}  ([B_{n} = c_d]^2 -[A_{n} = c_d]^2)}{4n^2} =   $$

	$$  \dfrac{\sum_{k=1}^{n}  \sum_{i=1}^{n} \sign(b_k - a_i)}{2n^2}  $$

	%		Let $\{c_j\}_{j=1}^{2n}$ the list of all the sorted observations of $A_n$ and $B_n$ where $c_1$ is the smallest observation and $c_{2n}$ the largest.		
\end{proof}

\begin{myproposition}
	Let $X_A$ and $X_B$ be two random variables and $A_n$ and $B_n$ their $n$ observations respectively.
	The $\mathcal{C}_\mathcal{D}$ estimated from the cumulative difference-plot is $\widetilde{\mathcal{C}_\mathcal{D}}$.
\end{myproposition}
\begin{proof}
	In Section~\ref{section:the_difference_plot}, we defined the $\mathcal{C}_\mathcal{D}$ estimated from the cumulative difference-plot as 
	$$\mathcal{C}_\mathcal{D} = \frac{  \dfrac{\int_{0}^{1}\mathcal{I}[\diff(x) > 0] - \mathcal{I}[\diff(x) < 0] dx}{2} + \frac{1}{2}}{\int_{0}^{1}\mathcal{I}[\diff(x) \neq 0]dx},$$
	where $\mathcal{I}$ is the indicator function.
	This proposition claims that 
	
	$$ \frac{  \dfrac{\int_{0}^{1}\mathcal{I}[\diff(x) > 0] - \mathcal{I}[\diff(x) < 0] dx}{2} + \frac{1}{2}}{\int_{0}^{1}\mathcal{I}[\diff(x) \neq 0]dx} = $$
	$$ \dfrac{\sum_{j=1}^{2n} \dfrac{\psi(c_j)}{2n} + 1}{2}\cdot k_c^{-1}.$$

	To prove it, we show that 
	
	$i) \ \int_{0}^{1}\mathcal{I}[\diff(x) > 0] - \mathcal{I}[\diff(x) < 0] dx = \sum_{j=1}^{2n} \dfrac{\psi(c_j)}{2n}$
	
	and
	
	$ii) \ \int_{0}^{1}\mathcal{I}[\diff(x) \neq 0]dx = k_c. $

	Let us focus our attention in $i)$.
	We split the integral into $2n$ parts:

	$$\int_{0}^{1}\mathcal{I}[\diff(x) > 0] - \mathcal{I}[\diff(x) < 0] dx = $$ 
	\begin{equation}
	\label{eq:desglosar_suma_de_integrales_CD}
	\sum_{j=1}^{2n} \int_{\frac{j-1}{2n}}^{\frac{j}{2n}} \mathcal{I}[\diff(x) > 0] - \mathcal{I}[\diff(x) < 0] dx
	\end{equation}

	Let $C_{2n} = \{c_j\}_{j=1}^{2n}$ be the list of all the sorted observations of $A_n$ and $B_n$ where $c_1$ is the smallest observation and $c_{2n}$ the largest and let $\{c_d\}_{d=1}^{d_{max}}$ be the sorted list of unique values in $C_{2n}$.
	We group the terms in the sum of Equation~\eqref{eq:desglosar_suma_de_integrales_CD} into $d_{max}$ groups such that for every $j$ in a group, $c_j = c_d$ .
	
	$$\sum_{d=1}^{d_{max}} \sum_{j} \int_{\frac{j-1}{2n}}^{\frac{j}{2n}}\mathcal{I}[\diff(x) > 0] - \mathcal{I}[\diff(x) < 0] dx$$

	Now we join the integrals for every $j$ in each group, such that the $j$ of the integral goes from $j_{d\downarrow} - 1$ to $j_{d\uparrow}$ (if the sample $c_d$ is unique in $C_{2n}$, then $j_{d\downarrow} = j_{d\uparrow} = j$).

	\begin{equation}
	\label{equation:integral_cd_desglosado_por_grupos_con_el_mismo_cj}
	\sum_{d=1}^{d_{max}} \int_{\frac{j_{d\downarrow} - 1}{2n}}^{\frac{j_{d\uparrow}}{2n}}\mathcal{I}[\diff(x) > 0] - \mathcal{I}[\diff(x) < 0] dx
	\end{equation}

	In the interval $(\frac{j_{d\downarrow} - 1}{2n},\frac{j_{d\uparrow}}{2n})$, $\diff$ evaluates to one of these four possibilities:
	
	\begin{enumerate}
		\item $\diff(x) = 0$  for all $x\in(\frac{j_{d\downarrow} - 1}{2n},\frac{j_{d\uparrow}}{2n})$
		\item $\diff(x) > 0$  for all $x\in(\frac{j_{d\downarrow} - 1}{2n},\frac{j_{d\uparrow}}{2n})$
		\item $\diff(x) < 0$  for all $x\in(\frac{j_{d\downarrow} - 1}{2n},\frac{j_{d\uparrow}}{2n})$
		\item $\diff(x) = 0$ in one point in the interval $(\frac{j_{d\downarrow} - 1}{2n},\frac{j_{d\uparrow}}{2n})$ and $\diff(x) > 0$  or $\diff(x) < 0$ for every other $x$ in the interval. However, we can safely ignore this point as the value of the integral is invariant to the value of the function in sets of zero measure.
	\end{enumerate}
	
	By looking at the empirical distributions $\hat{G}_A(x)$ and $\hat{G}_B(x)$ estimated from $A_n$ and $B_n$ respectively, we can guess which of these possibilities corresponds to each interval.

	\begin{small}
		$$
		\begin{cases}
		\multirow{2}{1em}{$1)$} & \text{ if } \hat{G}_A(c_{d-1}) = \hat{G}_B(c_{d-1})                                                             \\ \vspace{0.75em} &  \text{ and } \hat{G}_A(c_{d}) = \hat{G}_B(c_{d})\\ 
		\multirow{2}{1em}{$2)$} & \text{ if }  \hat{G}_A(c_{d-1}) \geq \hat{G}_B(c_{d-1})                                                         \\ \vspace{0.75em} &    \text{ and } \hat{G}_A(c_{d})   >  \hat{G}_B(c_{d})\\ 
		\multirow{2}{1em}{$2)$}  & \text{ if }  \hat{G}_A(c_{d-1}) >    \hat{G}_B(c_{d-1})                                                        \\ \vspace{0.75em} &    \text{ and } \hat{G}_A(c_{d}) \geq \hat{G}_B(c_{d})\\ 
		\multirow{2}{1em}{$3)$}  & \text{ if }  \hat{G}_B(c_{d-1}) \geq \hat{G}_A(c_{d-1})                                                       \\ \vspace{0.75em} &    \text{ and } \hat{G}_B(c_{d})   >  \hat{G}_A(c_{d})\\ 
		\multirow{2}{1em}{$3)$}  & \text{ if }  \hat{G}_B(c_{d-1}) >    \hat{G}_A(c_{d-1})                                                       \\ \vspace{0.75em} &    \text{ and } \hat{G}_B(c_{d}) \geq \hat{G}_A(c_{d})\\
		\multirow{2}{7em}{$4)$} & \text{ if }  \hat{G}_B(c_{d-1}) >    \hat{G}_A(c_{d-1})      \\ \vspace{0.75em} &    \text{ and } \hat{G}_A(c_{d})   >  \hat{G}_B(c_{d})\\
		\multirow{2}{7em}{$4)$} & \text{ if }  \hat{G}_A(c_{d-1}) >    \hat{G}_B(c_{d-1}) \\ \vspace{0.75em} &    \text{ and } \hat{G}_B(c_{d})   >  \hat{G}_A(c_{d})\\
		\end{cases} 
		$$
	\end{small}

	The value of the integral in Equation~\eqref{equation:integral_cd_desglosado_por_grupos_con_el_mismo_cj} corresponding to these possibilities are the following:
	
	\begin{enumerate}
		\item $0$
		\item $[C_{2n} = c_d] \cdot \frac{1}{2n}$
		\item $- [C_{2n} = c_d] \cdot \frac{1}{2n}$
		\item $[C_{2n} = c_d] \cdot (2 \cdot l_d - 1) \cdot \frac{1}{2n}$ 
	\end{enumerate}
	
	where $[C_{2n} = c_d]$ counts the number of items in $C_{2n}$ equal to $c_d$ and $l_d$ is the proportion in which $\diff(x) > 0$ in the interval $(\frac{j_{d\downarrow}- 1}{2n},\frac{j_{d\uparrow}}{2n})$.
	For example, $l_d = 0.75$ would represent that $\diff(x) > 0$ in $75\%$ of the total length of the interval, and $\diff(x) < 0$ in the other $25\%$.

	With this, we can rewrite Equation~\eqref{equation:integral_cd_desglosado_por_grupos_con_el_mismo_cj} as

	$$\sum_{d=1}^{d_{max}} [C_{2n} = c_d] \cdot \psi(c_d) \cdot \frac{1}{2n} = \sum_{j=1}^{2n} \frac{\psi(c_j)}{2n}, $$
	
	where $\psi$ is the function introduced in Definition~\ref{def:estimation_of_CD}.
	
	Now, we only need to prove $ii)$.
	Specifically, we need to show that
	
	$$\int_{0}^{1}\mathcal{I}[\diff(x) \neq 0]dx = k_c.$$
	
	We have that 
	
	$$\int_{0}^{1}\mathcal{I}[\diff(x) \neq 0]dx = \sum_{d=1}^{d_{max}} \int_{\frac{j_{d\downarrow} - 1}{2n}}^{\frac{j_{d\uparrow}}{2n}}\mathcal{I}[\diff(x) \neq 0]dx,$$
	
	and
	
	$$k_c = \frac{\sum_{j=1}^{2n} \mathcal{I}[\psi(c_j) \neq 0]}{2n} = $$
	$$\sum_{d=1}^{d_{max}} [C_{2n} = c_d] \frac{ \mathcal{I}[\psi(c_d) \neq 0]}{2n}.$$

	Finally, it is easy to see that 
	
	$$ \int_{\frac{j_{d\downarrow} - 1}{2n}}^{\frac{j_{d\uparrow}}{2n}}\mathcal{I}[\diff(x) \neq 0]dx = [C_{2n} = c_d] \frac{ \mathcal{I}[\psi(c_d) \neq 0]}{2n}, $$
	
	because $\diff(x) = 0$ in the interval $(\frac{j_{d\downarrow} - 1}{2n},\frac{j_{d\uparrow}}{2n})$ if and only if $\psi(c_d) = 0$.

\end{proof}

\subsection{Convergence of the estimators}
\label{section:convergence_of_estimators}

\begin{myproposition}
	Let $X_A$ and $X_B$ be two \underline{continuous} random variables and $\{a_i\}_{i\in\mathbb{N}}$ and $\{b_i\}_{i\in\mathbb{N}}$ be two infinite sequences of their observations respectively.
	Let $A_n$ and $B_n$ be the two finite subsequences that contain the first $n$ elements of $\{a_i\}_{i\in\mathbb{N}}$ and $\{b_i\}_{i\in\mathbb{N}}$ respectively.
	Then, 
	$$ \mathcal{C}_\mathcal{P}(X_A,X_B)  = \lim\limits_{n \to \infty} \widetilde{\mathcal{C}_\mathcal{P}}(A_n, B_n)$$
\end{myproposition}
\begin{proof}
	Let $\{P_s\}_{s \in \mathbb{N}}$ be a sequence of estimators with every estimator is determined randomly with the following procedure:
	
	1) generate two random permutations $\sigma_s$ and $\tau_s$ of size $n$.
	
	2) define each estimation as
	
	$$P_s(A_n,B_n) = \sum_{i=1}^{n}\dfrac{\sign(b_{\sigma_s(i)} - a_{\tau_s(i)})}{2n} + \dfrac{1}{2}.$$
	
	It is easy to see that each $P_s$ is an estimator of $\mathcal{P}(X_A < X_B)$ (since $X_A$, $X_B$ are continuous, we know that ${\mathcal{P}(X_A = X_B) = 0}$).
	Now observe that the sequence $\left\lbrace \dfrac{\sum_{t = 1}^{s} P_t(A_n,B_n) }{s} \right\rbrace_{n \in \mathbb{N}}$ converges to ${\widetilde{\mathcal{C}_\mathcal{P}}(A_n, B_n) =  \sum_{i,k=1...n} \dfrac{\sign(b_k -a_i)}{2n^2} + \frac{1}{2}}$, which means that $\widetilde{\mathcal{C}_\mathcal{P}}(A_n, B_n)$ is also an estimator of ${\mathcal{P}(X_A < X_B)}$.

\end{proof}

Unfortunately, the estimator $\widetilde{\mathcal{C}_\mathcal{D}}$ will not always converge: $\mathcal{C}_\mathcal{D}$ fails to satisfy Property~\ref{prop:scale_of_portions}, and this means that a a few points can still have a big impact in the estimation of $\mathcal{C}_\mathcal{D}$.
Specifically, given the continuous random variables $X_A$ and $X_B$ defined in $N$, $\widetilde{\mathcal{C}_\mathcal{D}}$ will converge iff ${\int_{N} \mathcal{I}[G_A(x) = G_B(x)] \cdot (g_A + g_B) dx = 0}$.

Luckily, this lack of convergence is not a problem when the estimation of $\mathcal{C}_\mathcal{D}$ is carried out visually in the cumulative difference-plot.
Since the visual representation of the cumulative difference-plot involves rendering the plot with pixels, there exists an small $\delta > 0$ such that when $|\diff(x)| < \delta$, the difference is displayed as $0$.

In practice, we do not even need to account for the case that $\diff(x) = 0$.
The cumulative difference-plot models the uncertainty with a confidence band, and when $\diff(x) = 0$ is inside the confidence band, then so are $\diff(x) > 0$ and $\diff(x) < 0$.
If we assume that the difference is positive, negative or zero every time that $\diff(x) = 0$ is inside the confidence band, we obtain the estimations $\widetilde{\mathcal{C}_\mathcal{D}}^+$, $\widetilde{\mathcal{C}_\mathcal{D}}^-$ and $\widetilde{\mathcal{C}_\mathcal{D}}^0$ respectively.
Now since $\widetilde{\mathcal{C}_\mathcal{D}}^+ > \widetilde{\mathcal{C}_\mathcal{D}}^= > \widetilde{\mathcal{C}_\mathcal{D}}^- $, the estimation of $\mathcal{C}_\mathcal{D}$  with the highest part of the confidence band is an upper bound of $\mathcal{C}_\mathcal{D}$.
The same is true for the estimation with the lowest part of the confidence band: it is a lower bound of $\mathcal{C}_\mathcal{D}$.

Although $\widetilde{\mathcal{C}_\mathcal{D}}$ does not converge to $\mathcal{C}_\mathcal{D}$, for any $\epsilon > 0$ we can find a $\delta$ small enough such that the difference between $\widetilde{\mathcal{C}_\mathcal{D}}^\delta$ and $\mathcal{C}_\mathcal{D}$ is smaller than $\epsilon$.
We formalize this claim in Conjecture~\ref{conjecture:cd_estimation_converges}, and we leave the proof for future work.

\begin{mydef}($\delta$-estimation of $\mathcal{C}_\mathcal{D}$) \\
	\label{def:estimation_of_CD_with epsilon}
	Let $X_A$ and $X_B$ be two continuous random variables and $A_n$ and $B_n$ their $n$ observations respectively.
	Let $\{c_j\}_{j=1}^{2n}$ the sorted list of all the observations of $A_n$ and $B_n$ where $c_1$ is the smallest observation and $c_{2n}$ the largest.
	Let $\{c_d\}_{d=1}^{d_{max}}$ be the sorted list of unique values in $\{c_j\}_{j=1}^{2n}$.
	
	We define the $\delta$-estimation of $\mathcal{C}_\mathcal{D}$, denoted as $\widetilde{\mathcal{C}_\mathcal{D}}^\delta$, as the same estimation as $\widetilde{\mathcal{C}_\mathcal{D}}$, but 
	assuming that the empirical distributions computed from $A_n$ and $B_n$ are equal when $|\hat{G}_A(x) - \hat{G}_B(x)| < \delta$.
\end{mydef}

The previous definition can also be based in the  $\delta$-difference, defined as $\diff^{\delta}(x) = \diff(x)$ when $\diff(x) \geq \delta$, and $\diff^{\delta}(x) = 0$ otherwise.

\begin{myconjecture}
	\label{conjecture:cd_estimation_converges}
	Let $X_A$ and $X_B$ be two \underline{continuous} random variables and $\{a_i\}_{i\in\mathbb{N}}$ and $\{b_i\}_{i\in\mathbb{N}}$ be two infinite sequences of their observations respectively.
	Let $A_n$ and $B_n$ be the two finite subsequences that contain the first $n$ elements of $\{a_i\}_{i\in\mathbb{N}}$ and $\{b_i\}_{i\in\mathbb{N}}$ respectively.
	Then, for all $\epsilon > 0$, there exists a $\delta > 0$ such that 
	$$ \left| \mathcal{C}_\mathcal{D}(X_A,X_B)  - \lim\limits_{n \to \infty} \widetilde{\mathcal{C}_\mathcal{D}}^\delta(A_n, B_n)\right|  < \epsilon$$
\end{myconjecture}

\subsection{Experimental verification}
In the following, we experimentally verify that the cumulative difference-plot can be used to deduce $\mathcal{C}_\mathcal{D}$ and $\mathcal{C}_\mathcal{P}$. 
To do so, we define six pairs of example random variables and measure the $\mathcal{C}_\mathcal{P}$ and $\mathcal{C}_\mathcal{D}$ with three different methods: the definition of $\mathcal{C}_\mathcal{D}$ and $\mathcal{C}_\mathcal{P}$ (Equation~\eqref{equation:definition_of_C_P_from_density} and Definition~\ref{def:dominance_rate}), the estimators in Definitions~\ref{def:estimation_of_CP} and \ref{def:estimation_of_CD} and from the \textit{cumulative difference-plot}.
The \textit{cumulative difference-plot} has a confidence band in addition to the estimation, and this confidence band allows the lower and upper bounds of $\mathcal{C}_\mathcal{D}$ and $\mathcal{C}_\mathcal{P}$ to be computed.

The probability density functions of the six examples are shown in Figures~\ref{fig:example1density1} through \ref{fig:example1density6}.
The probability density of these random variables is a mix of normal distributions, the beta distribution, and the log-normal distribution.

The difference plot and the estimations were carried out with $5000$ samples from each random variable.
The $\mathcal{C}_\mathcal{P}$ and $\mathcal{C}_\mathcal{D}$ values computed are shown in Figures~\ref{fig:three_methods_estimationCP} and \ref{fig:three_methods_estimationCD} respectively.
In every case, the estimations with the three methods match, except for $\mathcal{C}_\mathcal{D}$ in Example 4 (Figure~\ref{fig:example1density4}).
This is a deceptive example because, in most of the probability mass of $X_A$ and $X_B$, the cumulative distribution functions are equal.
Consequently, in this example, the estimator of $\mathcal{C}_\mathcal{D}$ introduced in Definition~\ref{def:estimation_of_CD} is unstable: it is very likely that the estimated empirical distributions are different even though the cumulative distribution functions are identical.
Overcoming this limitation involves choosing a small $\delta >0$, such that when the difference between the empirical distributions is smaller than $\delta$, they are considered equal.

We conclude that, in most cases, the three estimation methods (from densities, using the estimators and with the cumulative difference-plot) yield a similar result, which validates the statements in the previous section.

\FloatBarrier

\begin{figure}[h]
	\centering
	\includegraphics[width=0.57\linewidth]{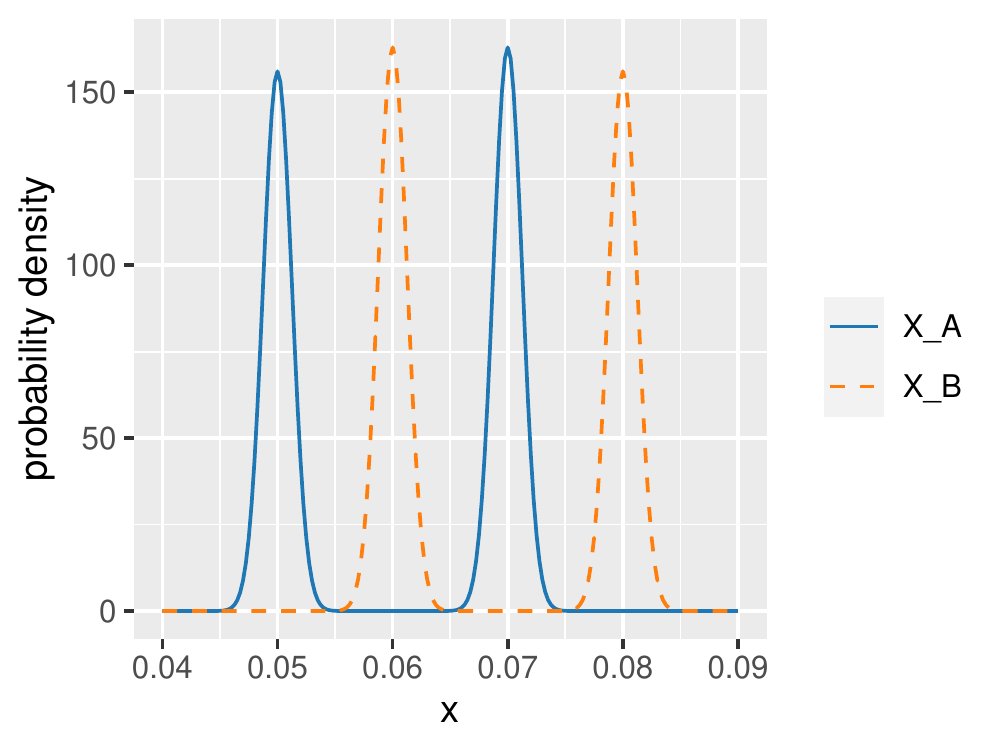}
	\caption{Probability density functions of Example 1}
	\label{fig:example1density1}
\end{figure}

\begin{figure}[h]
	\centering
	\includegraphics[width=0.57\linewidth]{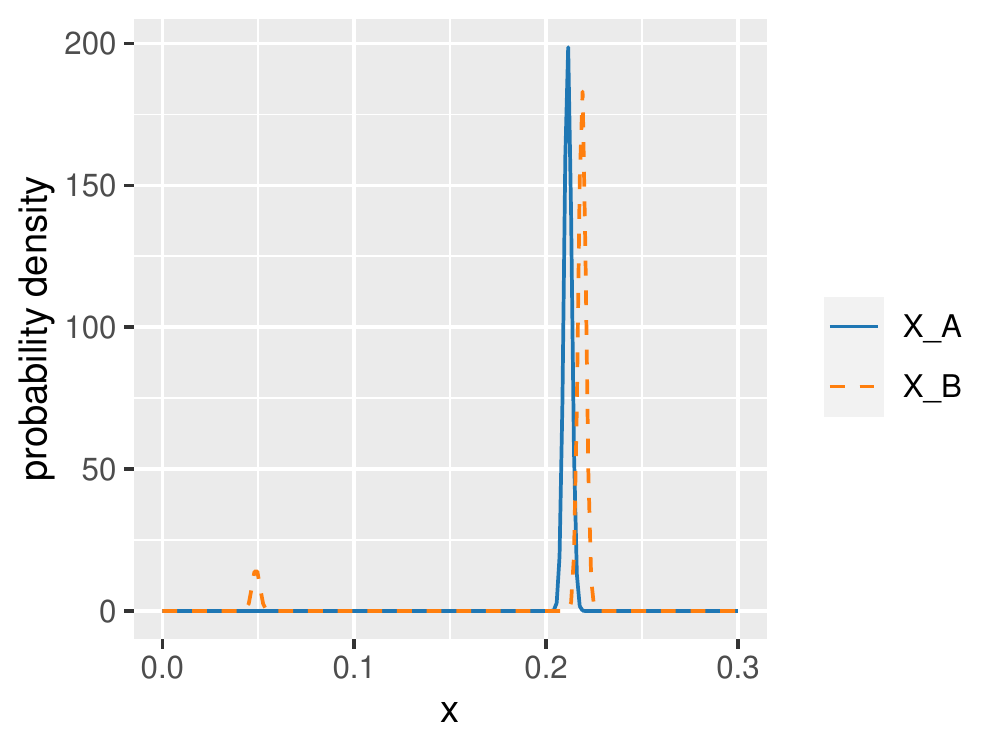}
	\caption{Probability density functions of Example 2}
	\label{fig:example1density2}
\end{figure}

\begin{figure}[h]
	\centering
	\includegraphics[width=0.57\linewidth]{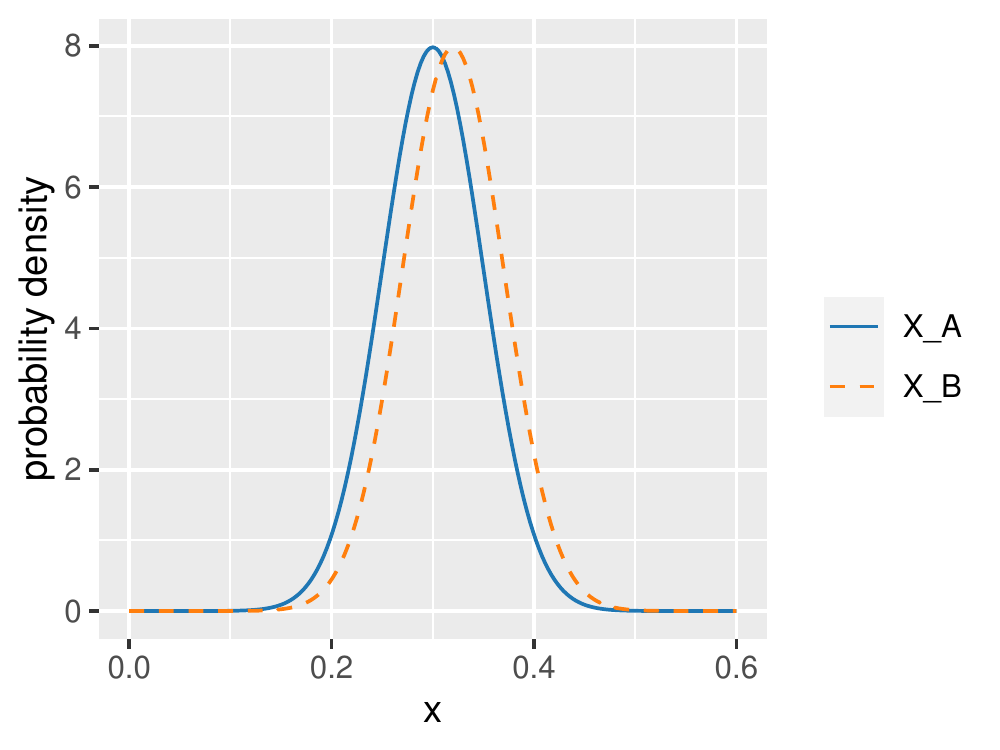}
	\caption{Probability density functions of Example 3}
	\label{fig:example1density3}
\end{figure}

\begin{figure}[h]
	\centering
	\includegraphics[width=0.57\linewidth]{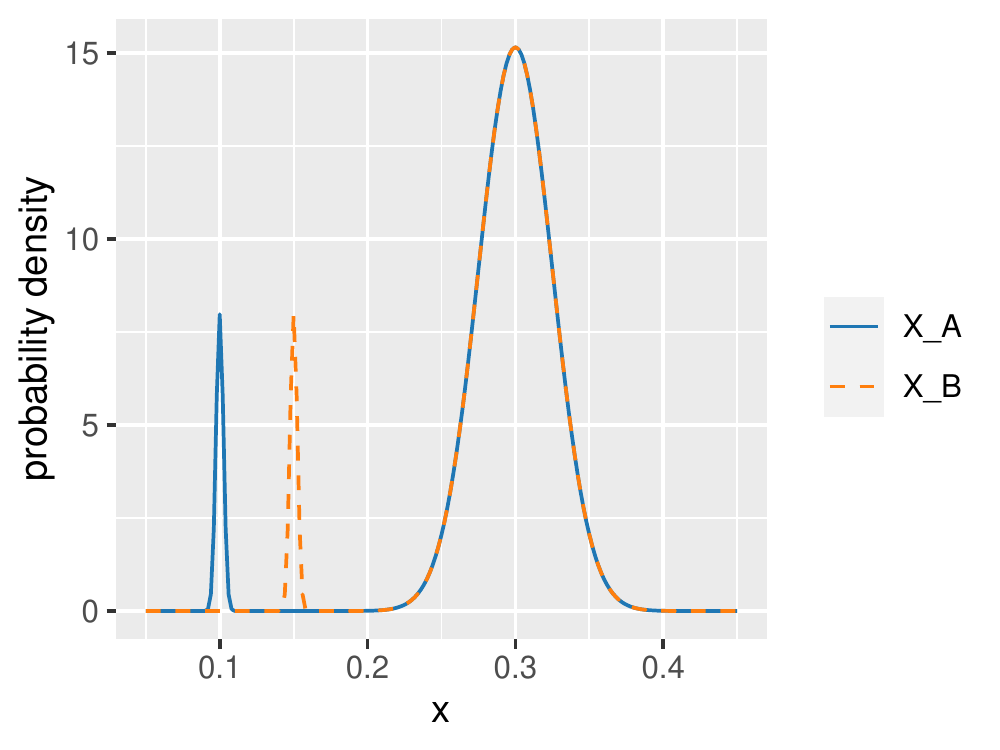}
	\caption{Probability density functions of Example 4}
	\label{fig:example1density4}
\end{figure}

\begin{figure}[h]
	\centering
	\includegraphics[width=0.57\linewidth]{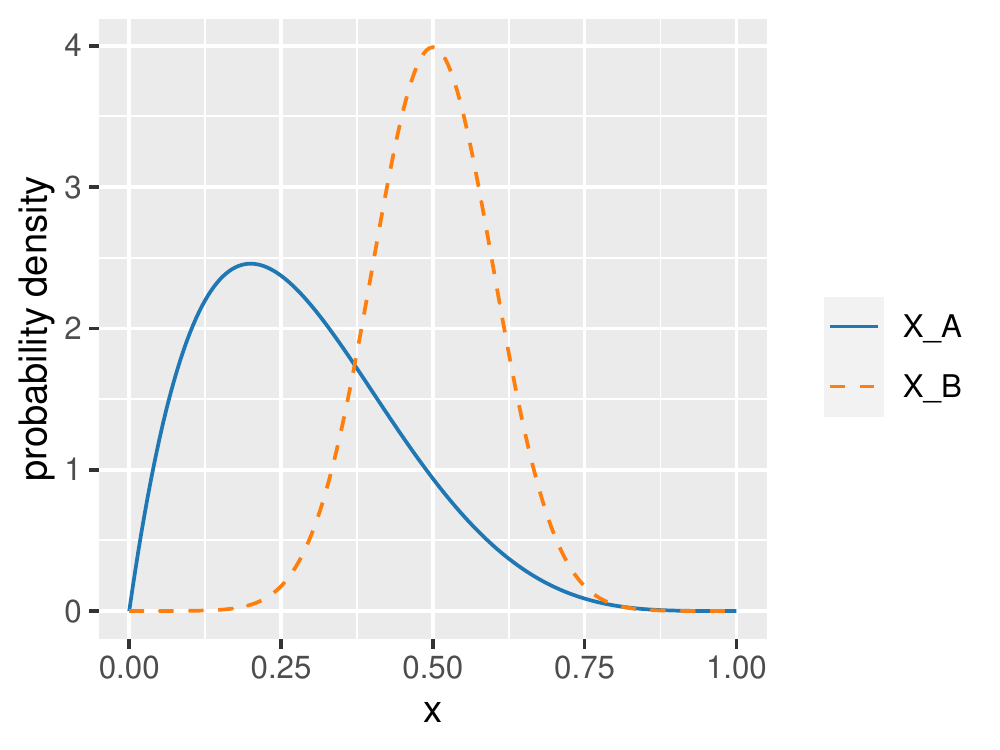}
	\caption{Probability density functions of Example 5}
	\label{fig:example1density5}
\end{figure}

\begin{figure}[h]
	\centering
	\includegraphics[width=0.57\linewidth]{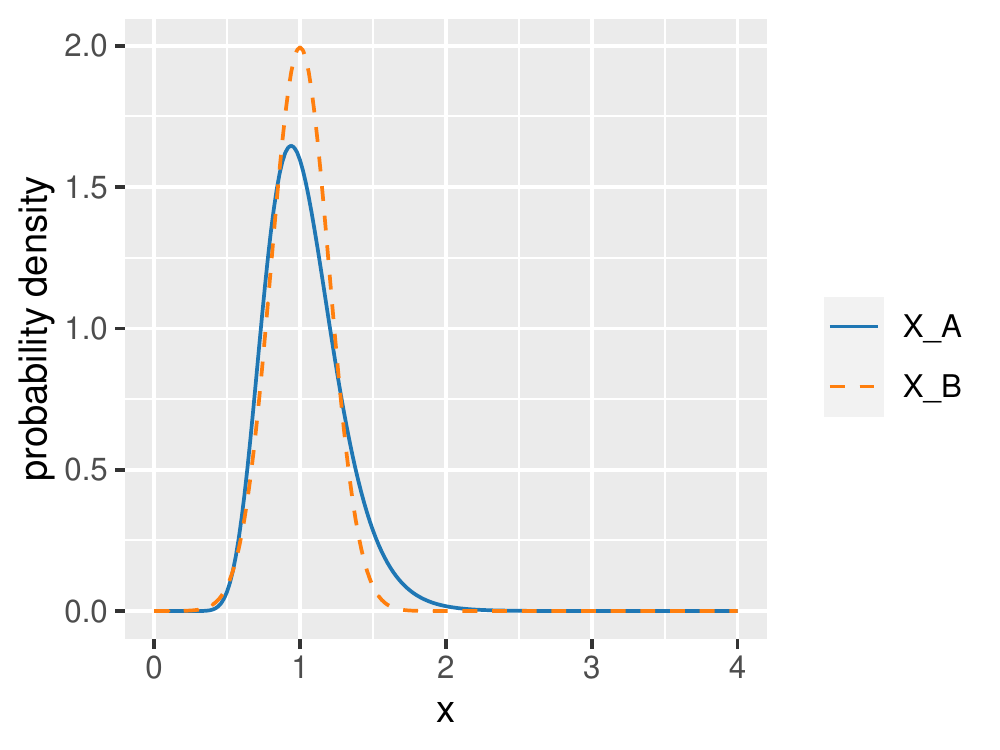}
	\caption{Probability density functions of Example 6}
	\label{fig:example1density6}
\end{figure}

\begin{figure}[h]
	\centering
	\includegraphics[width=0.57\linewidth]{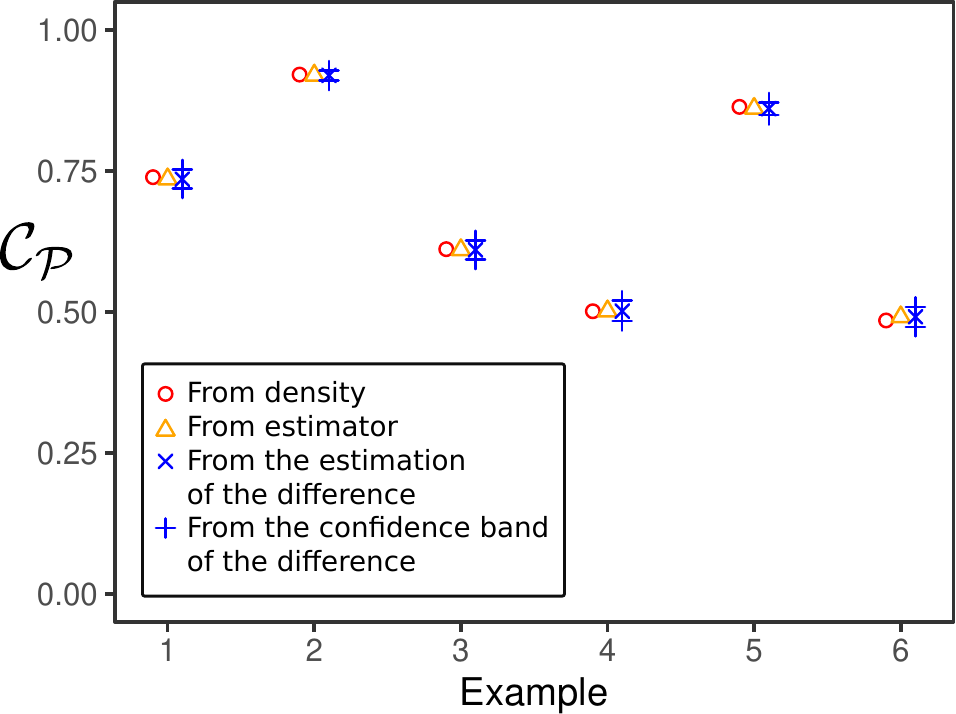}
	\caption{The $\mathcal{C}_\mathcal{P}$ values obtained in the six examples with the three methods.}
	\label{fig:three_methods_estimationCP}
\end{figure}

\begin{figure}[h]
	\centering
	\includegraphics[width=0.57\linewidth]{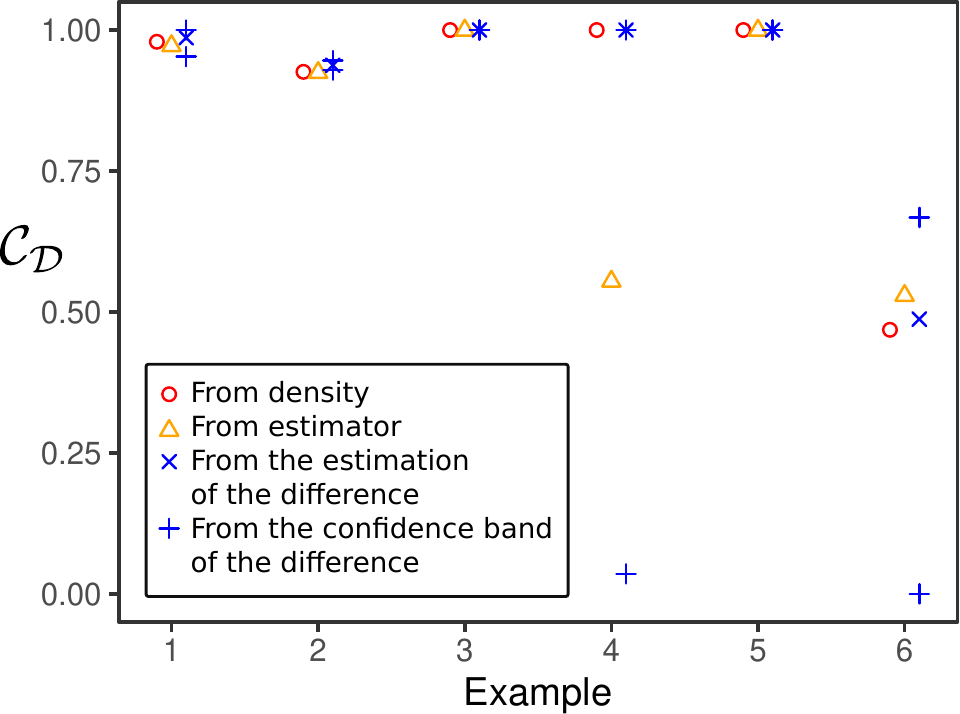}
	\caption{The $\mathcal{C}_\mathcal{D}$ values obtained in the six examples with the three methods.}
	\label{fig:three_methods_estimationCD}
\end{figure}

%	% Only the appendices header + download link
%	\pagebreak
%	\section*{Appendices}
%	To obtain a reasonable paper length, the appendices have been omitted and can be downloaded from\\ \href{https://github.com/EtorArza/RVCompare-paper/releases/download/v1.0/appendices.pdf}{github.com/EtorArza/RVCompare-paper/releases/download/v1.0/appendices.pdf}.
%	
%	\section{A literature review of measures}
%	\label{appendix:additional_information_on_bad_comparison_functions}
%	\subsection{$f$-divergences} % https://en.wikipedia.org/wiki/F-divergence
%	\label{sec:f_divergences_explained}
%	\subsection{Wasserstein distance} % https://en.wikipedia.org/wiki/F-divergence
%	\subsection{Heuristic derivation of the first-order stochastic dominance} 
%		
%	\section{Quantile random variables}
%	\subsection{Computing the probability density functions of $Y_A$ and $Y_B$}
%	\label{appendix:estimation_densityY_A_step_by_step}
%	\subsection{The quantile random variables have the same $\mathcal{C}_\mathcal{P}$ and $\mathcal{C}_\mathcal{D}$ as the kernel density estimates of $X_A$ and $X_B$.}
%	\label{appendix:Y_A_mantains_same_mathcalC}
%
%	\section{ $\mathcal{C}_\mathcal{P}$ and $\mathcal{C}_\mathcal{D}$ in the cumulative difference-plot}
%	\label{appendix:proof_difference_graph}
%	\subsection{Estimating $\mathcal{C}_\mathcal{P}$ and $\mathcal{C}_\mathcal{D}$ from the cumulative difference-plot}
%	\subsection{Convergence of the estimators}
%	\label{section:convergence_of_estimators}
%	\subsection{Experimental verification}

\end{document}